\pdfoutput=1

 \documentclass[10pt,twocolumn,twoside,journal]{IEEEtran}

\usepackage{cs}
\usepackage{mathsymb}
\usepackage{verbatim}
\usepackage{subfigure}
\usepackage{url}
\usepackage{graphicx}
\graphicspath{{./}{./Fig/}{./Figs/}{./Fig/eps/}{./Fig/pdf/}}

\usepackage{cite,eucal}

\usepackage{hyperref}
\usepackage{xcolor}
\definecolor{dark-red}{rgb}{0.4,0.15,0.15}
\definecolor{dark-blue}{rgb}{0.15,0.15,0.4}
\definecolor{medium-blue}{rgb}{0,0,0.5}
\hypersetup{
    colorlinks, linkcolor={dark-red},
    citecolor={dark-blue}, urlcolor={medium-blue}
}

\usepackage{algorithm2e}
\let\savedalgorithm\algorithm
\let\savedendalgorithm\endalgorithm
\newenvironment{algorithmic}{%
\savedalgorithm
}{%
\savedendalgorithm
}

\usepackage{algorithm}
\newtheorem{dfn}{Definition}   %
\newtheorem{lem}{Lemma}        %
\newtheorem{thm}{Theorem}      %

\usepackage{amsmath}
\usepackage{url}
\usepackage{multirow}
\usepackage{pifont}

\def\rboost{{RBoost}\xspace}
\def\randomRank{{RBoost$^{\,\text{rank} }$}\xspace}
\def\randomProj{{RBoost$^{\,\text{proj} }$}\xspace}

\def\Integer{\mathbb{Z}^+}

\def\brho{{\boldsymbol \rho}}

\def\yi{{y_i}}

\def\bxi{{\bx_i}}
\def\Pyi{{\bP^{(y_i)}}}
\def\Prr{{\bP^{(r)}}}
\def\RR{\Real}

\def\Delta{ \pmb{\delta} }

\def\cone{{\ding{172}}}
\def\ctwo{{\ding{173}}}
\def\cthree{{\ding{174}}}
\def\cfour{{\ding{175}}}
\def\cfive{{\ding{176}}}
\def\csix{{\ding{177}}}

\newcommand{\inner}[2]{\left\langle #1,#2 \right\rangle}

\DeclareMathOperator{\Ycal}{\mathcal{Y}}
\DeclareMathOperator{\Ncal}{\mathcal{N}}

\def\paragraph{\textbf}

\newcommand{\revised}[1]{{\color{blue}#1}}
\renewcommand{\revised}[1]{{\color{black}#1}}

\begin{document}

\title{RandomBoost: Simplified Multi-class Boosting through Randomization}

\author{
         Sakrapee Paisitkriangkrai,
         Chunhua Shen,
         Qinfeng Shi,
         Anton van den Hengel
\thanks
{
}
\thanks
{
The authors are with School of Computer Science,  
The University of Adelaide, SA 5005, Australia.
Correspondence should be addressed to C. Shen 
(e-mail:  chunhua.shen@adelaide.edu.au). 
}
\thanks
{
This work is in part supported by Australian Research Council grants
LP120200485 and FT120100969.
 }
}

\markboth{Working Paper}
{Paisitkriangkrai
\MakeLowercase{\textit{et al.}}: RandomBoost: Simplified Multi-class Boosting through Randomization}

\maketitle

\begin{abstract}

We propose a novel boosting approach to multi-class classification
problems, in which
multiple classes are distinguished by a set of random projection
matrices in essence.
The approach uses random projections to alleviate the proliferation of
binary classifiers typically required to perform multi-class
classification.  The result is a multi-class classifier with a single
vector-valued parameter, irrespective of the number of classes
involved.  Two variants of this approach are proposed.
The first method randomly projects the original data into new spaces,
while the second method randomly projects the
outputs of learned weak classifiers.  These methods are not only
conceptually simple but also effective and easy to implement.  A
series of experiments on synthetic, machine learning and visual
recognition data sets demonstrate that our proposed methods compare
favorably to existing multi-class boosting algorithms in terms of both
the convergence rate and classification accuracy.

\end{abstract}


  \tableofcontents
  \clearpage
 
\section{Introduction}

Multi-class classification has not only become an important tool in
statistical data analysis, but also a critical factor in the progress
that is being made towards solving some of the key problems in
computer vision, such as generic object recognition.
Applications of multi-class classification vary, but the objective
in each is to assign the correct class label to each input test
example, whether it be assigning the correct value to a handwritten
digit, or the correct identity to a face.

Boosting is a well-known machine learning algorithm which builds a
strong ensemble classifier
by combining weak learners which are in turn
generated by a base learning oracle.  The fact that a wide variety of
weak learners can be employed makes the algorithm extremely flexible,
yet it has been shown that boosting is robust and seems resistant to
over-fitting in many cases
\cite{Garcia2009Constructing, Sun2010Sparse, Wang2012Fast,MDBoost2010Shen}.
A boosting classifier is made up of a set of weak classification rules
and a corresponding set of coefficients controlling the manner in
which they are combined, and many multi-class variants have been
proposed.
Most of these algorithms reduce the multi-class classification problem
to multiple binary-class problems and learn a coding matrix or a
vector of coefficients for each class
(\eg, \cite{Allwein2001Reducing,Dietterich1995Solving,Guruswami1999Multiclass,Schapire1997Using,Shen2011Totally}).
The main justification for this reduction is the fact that binary
classification problems are well studied and many effective algorithms
have been carefully designed.
In contrast to existing approaches, we propose to learn a single
model with a single vector of coefficients
that is independent of the number of classes.
We achieve this by using random projections as the main tool.
Random projections have been widely used as a dimensionality reduction
technique in many areas, \eg, signal processing \cite{Candes2006Near},
machine learning \cite{Arriaga2006Algorithmic, Fern2003Random},
information retrieval \cite{Fradkin2003Experiments},
data mining \cite{Indyk1998Approximate}, face
recognition \cite{Shi2010Rapid}.  The algorithm is based on the idea
that any input feature spaces can be embedded into a new lower
dimensional space without significantly losing the structure of the
    data or pairwise distances between instances.
We choose random projections since we want to introduce diversity in
the data space (either the original input data space or the weak
classifiers' output space) for multi-class problems while preserving
pairwise relationships.
To our knowledge, this is the first time that random projections are
used to simplify and implement multi-class boosting classification.

Our main contributions are as follows.
\begin{itemize}
\item
    We propose a new form of multi-class boosting which trains a
    single-vector parameterized classifier irrespective of
    the number of classes.
    We illustrate this new approach by incorporating random
    projections and pairwise constraints into the boosting framework.
\item
    Two algorithms are proposed based on this high-level idea. The first
    algorithm randomly projects the original data into new spaces and the
    second algorithm randomly projects the outputs of selected weak
    classifiers.
    We then design multi-class boosting based on the column generation
    technique in convex optimization.

    The first algorithm is optimized in a stage-wise fashion,
    bearing resemblance to RankBoost \cite{Freund2003Efficient} (and
    AdaBoost because of the equivalence of RankBoost and AdaBoost
    \cite{Rudin2009Margin}).
    The optimization procedure of our second method
    is inspired by the totally corrective boosting framework
    \cite{LPBoost,Shen2011Totally}, although
    for the second approach,
    the mechanism for
    generating weak classifiers is entirely different from
    all conventional boosting methods.
    {\em
    Our new design is not only conceptually simple, due to the reduced
    parameter space, but also effective as we  empirically
    demonstrate on various data sets.
    }
\item
    We theoretically justify the use of random projections by
    proving the margin separability in the proposed boosting.
    This theoretical analysis provides some insights in terms of the margin preservation
    and the minimal number of projected dimensions to guarantee margin separability.
\item
    We empirically show that both proposed methods
    perform well. We
    demonstrate some of the
    benefits of the proposed algorithms in a series of experiments.
    In terms of test error rates, our proposed methods are at least as
    well as many existing multi-class algorithms.
    We have made the source code of the proposed boosting methods
    accessible at
    {\url{http://code.google.com/p/boosting/downloads/}}.
\end{itemize}
Next we review the literature related to random projections and
multi-class boosting.

\section{Related work}

\revised{
Random projections have attracted much research interest from various scientific
fields, \eg, signal processing \cite{Candes2006Near},
clustering \cite{Fern2003Random}, multimedia indexing and retrieval \cite{Gionis1999Similarity,Cai2007Scalable,ke2004efficient},
machine learning, \cite{Bingham2001Random, Fradkin2003Experiments} and
computer vision \cite{Shakhnarovich2003Fast, Shi2010Rapid}.
Random projections are a powerful method of dimensionality reduction.
The technique involves taking a high-dimensional data and maps it into
a lower-dimensional space, while providing some guarantee on the approximate
preservation of distance.
Random projections have been successfully applied in many research fields.
One of the most widely used applications of random projections is sparse
signal recovery.
Cand\'es and Tao show that the original signal can be reconstructed within
very high accuracy from a small number of random
measurements \cite{Candes2006Near}.
Random projections have also been applied in data mining as an
efficient approximate nearest neighbour search algorithm.
The search algorithm,  known as locality sensitive hashing (LSH),
approximates the cosine distance in the nearest
neighbour problem \cite{Indyk1998Approximate}.
The basic idea of LSH is to choose a random hyperplane and use
it to hash input vectors to a single bit.
Hash bits of two instances match with probability proportional
to the cosine distance between two instances. Unlike traditional
similarity search, 
LSH has been shown to work
effectively and efficiently for large-scale high-dimensional data.
}

\revised{
In machine learning, random projections have been applied to
both supervised learning and unsupervised clustering problems.
Fern and Broadley show that random projections can be used to
improve the clustering result for high dimensional
data \cite{Fern2003Random}.
Bingham and Manilla compare random projections with several
dimensionality reduction methods on text and image data and
conclude that the random lower dimension subspace yields
results comparable to other conventional dimensionality
reduction techniques  with significantly less computation
time \cite{Bingham2001Random}.
Fradkin and Madigan explore random projections in a supervised
learning context \cite{Fradkin2003Experiments}.
They conclude that random projections offer clear
computational advantage over principal component analysis
while providing a comparable
degree of accuracy. 
Thus far, we are not aware of any
existing works which apply random projections to  multi-class
boosting.
}

Boosting is a supervised learning algorithm which
has attracted significant research attention over the past
decade due to its effectiveness and efficiency.
The first practical  boosting algorithm,
AdaBoost, was introduced  for binary
classification problems \cite{Freund1997Decision}.
Since then, many subsequent works have been
focusing on binary classification problems.
Recently,
however, several multi-class boosting algorithms have been proposed.
Many of these algorithms convert multi-class problems into a set of
binary classification problems.
Here we loosely divide existing work on multi-class boosting into
four categories.

{\em One-versus-all}
The simplest conversion is to reduce the problem of
classifying $k$ classes into $k$ binary problems, where each problem
discriminates a given class from other $k-1$ classes.
Often $k$ binary classifiers are used.
For example, to classify digit `$0$' from all other digits,
one would train the binary classifier with positive samples
belonging to the digit `$0$' and negative samples belonging
to other digits, \ie, `$1$', $\cdots$, `$9$'.
During evaluation, the sample is assigned to the class of
the binary classifier with the highest
confidence.
Despite the  simplicity of  one-versus-all,
Rifkin and Klautau have shown that
one-versus-all can provide performance on par with that of more
sophisticated multi-class classifiers \cite{Rifkin2004Defense}.
An example of one-versus-all boosting is AdaBoost.MH
\cite{Schapire1999Improved}. 

{\em All-versus-all}
In all-versus-all classifiers, the algorithm compares each class to all other classes.
A binary classifier is built to discriminate between each pair of
classes while discarding the rest of the classes. The algorithm thus
builds $\frac{k(k-1)}{2}$ binary classifiers.
During evaluation the class with the maximum number of votes wins.
Allwein \etal conclude that all-versus-all often has 
better generalization performance than one-versus-all algorithm \cite{Allwein2001Reducing}.
The drawback of this algorithm is that the complexity grows
quadratically with the number of classes. Thus it is not scalable in
the number of classes. 

{\em Error correcting output coding (ECOC)}
The above two algorithms are special cases of ECOC.
The idea of ECOC is to associate each class with a codeword which is a row of a coding matrix $\bM \in \Real^{k\times T}$ and $M_{ij} \in \{-1, 0, 1\}$.
The algorithm trains $T$ binary classifiers to distinguish between $k$ different classes.
During evaluation, the output of $T$ binary classifiers (a $T$-bit string) is compared to each codeword and the sample is assigned to the class whose codeword has the minimal hamming distance.
Diettrich and Bakiri report improved generalization ability of this method over the above two techniques \cite{Dietterich1995Solving}.
In boosting, the binary classifier is viewed as weak learner and each is learned one at a time in sequence.
Some well-known ECOC based boostings are AdaBoost.MO, AdaBoost.OC and AdaBoost.ECC \cite{Schapire1997Using, Guruswami1999Multiclass}.
Although this technique provides a simple solution to multi-class
classification, it does not fully exploit the pairwise correlations between
classes.

{\em Learning a matrix of coefficients in a single optimization
problem}
One learns a linear ensemble  for each class.
Given a test example, the label is predicted by
$
 \; {\argmax}_r$ $\sum_{t} W_{rt} h_t(\bx).
$
Each row of the matrix $\bW$ corresponds to one of the classes.
The sample is assigned to the class whose row has the largest value of the weighted combination.
To learn the matrix $\bW$, one can formulate the problem in the
framework of multi-class maximum-margin learning.
Shen and Hao show that the large-margin multi-class boosting can be
implemented
using column generation \cite{Shen2011Direct}.

\revised{
In contrast to previous works on multi-class boosting,
we propose two novel boosting approaches
that learn a single-vector parameterized ensemble classifier to distinguish
between  classes.
We achieve this through the use of random projections and pairwise constraints.
To be  specific, the first algorithm randomly projects
each training datum to a new space.
We then show that the multi-class learning problem can be reduced
to a ranking problem.
For the second algorithm, we train the  multi-class
boosting by randomly projecting the outputs of learned weak classifiers.
}

{\bf Notation}
    We use a bold lowercase letter, \eg, $\bx$, to denote a column
    vector and a bold uppercase letter, \eg, $\bP$, to denote a
    matrix.  The $ij$-th entry of a matrix $\bP$ is written as
    $P_{ij}$.  $P_{i:}$ and $P_{:j}$ are the $i$-th row and $j$-th
    column of $\bP$, respectively.  Let $(\bxi, \yi) \in {\Real}^{d}
    \times \{1,2,\cdots, k\}, i = 1,\cdots, m$ be a set of $m$
    multi-class training samples where $k$ is the number of classes.
    Let $T$ be the maximum number of boosting iterations and the
    matrix, $\bH \in {\Real}^{m \times T}$, denote the weak
    classifiers' response on the training data.  Each column $H_{:t}$
    contains the output of the $t$-th weak classifier $h_t(\cdot)$.
    Each row $H_{i:}$ contains the outputs of all weak classifiers
    from the training instance $\bxi$.

\section{Our approach}
\label{sec:approach}

\begin{figure}[t]
    \centering
        \includegraphics[width=0.48\textwidth]{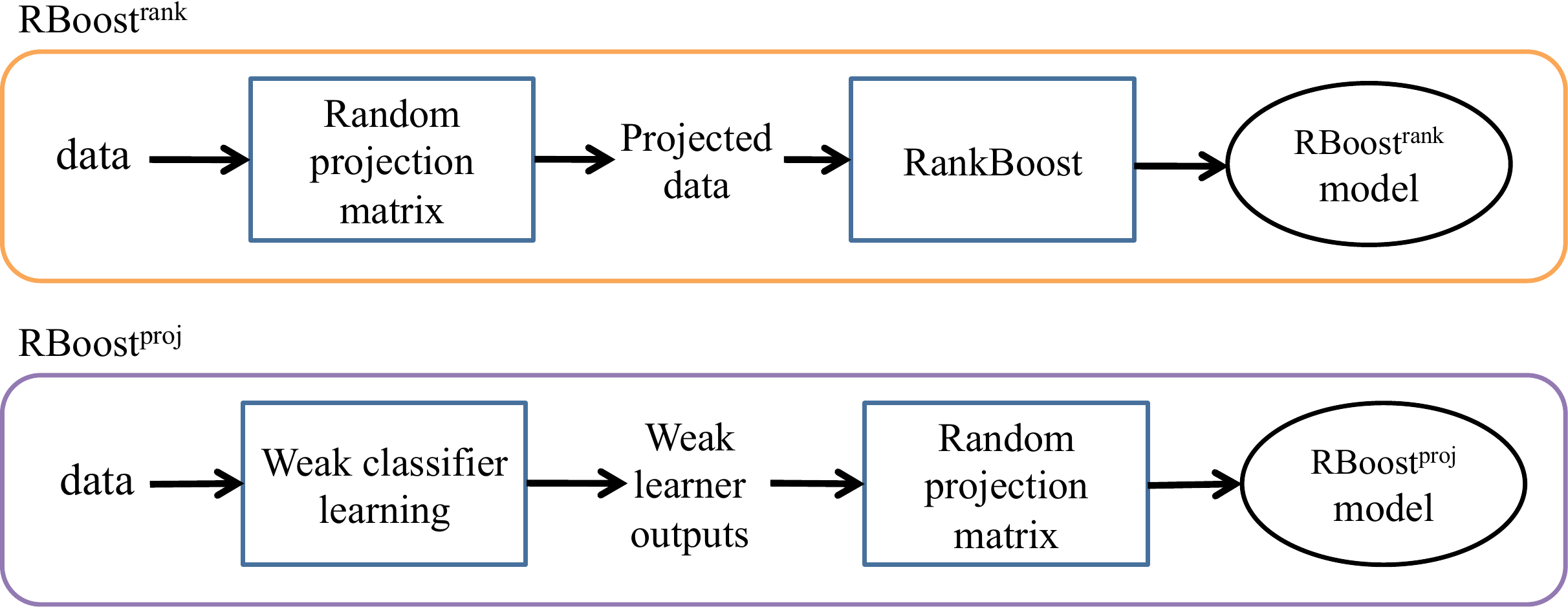}
    \caption{
    Flowchart illustration of \randomRank and \randomProj.
    }
    \label{fig:flowchart}
\end{figure}

\revised{
        Many existing multi-class boosting algorithms learn a strong
        classifier, and a corresponding set of weights, $\bw_r \in
        {\Real}^{T}$, for each class $r$.
        The two novel methods that we propose here, however,
        learn a single vector of weights, $\bw \in
        {\Real}^{T}$, for all classes.
        our
        approaches are  conceptually simple and easy to implement. 
        We illustrate our new approaches by incorporating random projections
        into the boosting framework.
        Since random projections can be applied to either the original raw data
        or other intermediate results (for example, weak classifiers' outputs),
        we can formulate the multi-class problem as:
        $1$) a pairwise ranking
        problem that is based on random projections of the original
        data; and
        $2$) a maximum margin problem that is based on random projections of
        the weak classifiers' outputs.
}

        In our first approach, we generate a random Gaussian
        matrix
        $\bP \in {\Real}^{n \times d}$, whose entry $\bP(i,j)$ is
        $\frac{1}{\sqrt{n}} {a_{ij}}$ where
        $a_{ij}$ is i.i.d.\ random variables from $\Ncal(0,1)$.
        We 
         multiply it with each
        training instance, $\bx \in {\Real}^{d \times 1}$, to obtain a
        projected data vector, $\bP \bx \in
        {\Real}^{n \times 1}$.  The projected vector,
        $ \bP \bx$, approximately preserves all pairwise distances
        of input vector $\bx$,  provided that $\bP$ consists of i.i.d.\
        entries with zero mean and constant variance
        \cite{Arriaga2006Algorithmic}.  In our second approach, we
        generate a random Gaussian matrix, $\bP \in {\Real}^{n
        \times T}$. We then  multiply it with the output of weak learners,
        $\bP [h_1(\cdot), h_2(\cdot), \cdots, h_T(\cdot)]^\T$, to
        obtain a new weak learners' output space, $\bP H_{i:}^{\T} \in
        {\Real}^{n \times 1}$.
\revised{
        In short, both algorithms rely on the use of random
        projections to obtain the single-vector parameterized classifier.
        However, in the first algorithm, the original raw data is randomly
        projected while in the second algorithm,
        the weak learners' outputs are randomly projected.
        A high-level flowchart that illustrates both approaches
        is shown in Fig.~\ref{fig:flowchart}.
}

\subsection{Multi-class boosting by randomly projecting the original
data}

We first formulate the multi-class learning problem as a pairwise
ranking problem.
The basic idea of our approach is to learn a multi-class classifier
from the same instance being projected using $ k $ different random
projection matrices.
We create $k$ pre-defined random projection matrices, $\bP^{(1)},
\bP^{(2)}, \cdots, \bP^{(k)}$, for each class, where the superscript
indicates the class label associated to the random projection matrix.
Given a training instance, $(\bxi, \yi)$, the following condition,
$F(\Pyi \bxi) > F(\Prr \bxi), \forany r \neq \yi$, has to be
satisfied.\footnote{For
simplicity, we omit the model parameter $ \bw$.}
That is to say, the {\em correct} model's response must be larger than
all the {\rm incorrect} models' responses.
We can strengthen this by requiring that
the difference $F(\Pyi \bxi) - F(\Prr \bxi) $ is as large as possible.
Motivated by the large margin principle, we formulate our multi-class problem in
the framework of maximum-margin learning.

\paragraph{Learning using the exponential loss}
Given that we have $m$ training samples with $k$ classes,
the total number of such pairwise relations is $m(k-1)$.
Putting it into the large-margin learning framework,
we can define the {\em margin} associated with the above condition as, $F(\Pyi
\bxi) - F(\Prr \bxi)$,
which can be explicitly  rewritten as,
\begin{align}
    \label{EQ:rank1}
        \rho_{ir} &= F(\Pyi \bxi) - F(\Prr \bxi)  \\ \notag
             &= {\textstyle \sum_{t=1}^T } h_t(\Pyi \bxi) w_t -
                {\textstyle\sum_{t=1}^T } h_t(\Prr \bxi) w_t \\ \notag
             &= {\textstyle\sum_{t=1}^T }  \delta h_t(\Pyi, \Prr, \bxi)  w_t
              \\ \notag
             &= \dH(\Pyi, \Prr, \bxi) ^\T \bw, %
\end{align}
where $\delta h_t(\Pyi, \Prr, \bxi) = h_t(\Pyi \bxi) - h_t(\Prr \bxi)$,
and
\[
\dH (\cdot, \cdot, \cdot) = \bigl[
\delta h_1(\cdot, \cdot, \cdot),
\delta h_2(\cdot, \cdot, \cdot),
\cdots,
\delta h_T(\cdot, \cdot, \cdot)
\bigr] ^\T  \in \Real^ { T \times 1}
\] is a column vector.
    The purpose is to learn a regularized model
    that satisfies as many constraints,
    $ \rho_{ir} > 0  $, as possible. That is to say, we minimize the
    training error of the model with a controlled capacity.
In theory, both of the  two proposed boosting
methods  can employ any convex loss function. We first show how to
derive the boosting iteration using the exponential loss. Later we
generalize it to any convex loss. 

With the exponential loss, the primal problem can be written as,
\begin{align}
    \label{EQ:rankPrimal}
        \min_{ \bw, \boldsymbol \rho}   \quad
        &
        \log \bigl( \sum\nolimits_{ir} \exp \left( -\rho_{ir} \right) \bigr) +
        \nu  {\bf 1}^{\T} \bw   \\ \notag
        \st \quad &
        \rho_{ir} =  \dH(\Pyi, \Prr, \bxi) ^\T  \bw, \forany \, {\rm pair
        } \, (
        ir ); \, \bw \geq 0,
\end{align}
    where $(ir)$ represents the joint index through all of the data
    and all of the classes.  Taking the logarithm of the original cost
    function does not change the nature of the problem
    as $\log( \cdot )$ is
    strictly monotonically increasing. This formulation is similar to
    the binary totally corrective boosting discussed in
    \cite{Shen2011Totally}.  Also we have applied the $ \ell_1 $ norm
    regularization as in
    \cite{Shen2011Totally,LPBoost} to control the model complexity.

    If we can solve the optimization problem \eqref{EQ:rankPrimal},
    the learned model can be easily obtained. Unfortunately,
    the number of weak learners is usually extremely large or even
    infinite, which corresponds to an extremely or infinite
    large number  of variables $ \bw $,
    it is usually intractable to solve \eqref{EQ:rankPrimal}
    {\em directly}.
    Column generation can be used to approximately solve this problem
    \cite{Shen2011Totally, LPBoost}.
    We need to derive a meaningful Lagrange dual problem such that
    column generation can be applied.
    The Lagrangian is
    $ L = \log (\sum_{ir} \exp( - \rho_{ir}  )  )   + \nu {\bf 1} ^ \T \bw
    - \bu ^\T {\boldsymbol \rho }  +  \sum_{ir} u_{ir} \dH (\Pyi, \Prr, \bxi)
    - \bq ^\T \bw.
    $ with $ \bq \geq 0 $.
    The dual problem can be obtained as
    $  \sup_{ \bu }  \inf_{ \bw, \boldsymbol \rho } L   $.

    The Lagrange dual problem can be derived as
    \begin{align}
    \label{EQ:rankDual}
        \min_{ \bu }   \quad
        &
        {\textstyle \sum_{ir}} u_{ir} \log( u_{ir} ) \\ \notag
        \st \quad &
        {\textstyle \sum_{ir}} u_{ir} \dH(\Pyi, \Prr, \bxi) ^\T
               \leq \nu
        {\boldsymbol 1} ^\T,
        \;
        \bu \geq 0, {\bf 1}^{\T} \bu = 1.
\end{align}
As is the case of AdaBoost \cite{Shen2011Totally}, the dual is a Shannon entropy maximization problem.
The objective function of the dual encourages the dual variables, $\bu$, to be uniform.
The Karush-Kunh-Tucker (KKT) optimality condition gives the
relationship between the optimal primal and dual variables:
\begin{align}
    \label{EQ:rankKKT}
        u_{ir} = \frac{ \exp(-\rho_{ir}) }{ \sum_{ir} \exp( -\rho_{ir} )}. %
\end{align}
The primal problem can be solved using an efficient Quasi-Newton
method like L-BFGS-B, and the
dual variables can be obtained using the KKT condition.
From the dual, the subproblem for generating weak classifiers is,
\begin{align}
    \label{EQ:rankWeak}
    h^{\ast}(\cdot)
     = \argmax_{h(\cdot) \in \mathcal{H}}
     {\sum_{ir}} u_{ir} \delta h(\Pyi,\Prr, \bxi).
\end{align}
This corresponds to find the most violated constraint of the dual
problem \eqref{EQ:rankDual}.
The details of our ranking based multi-class boosting algorithm are
given in Algorithm~\ref{ALG:alg1}.
\SetKwInput{KwInit}{Initilaize}
\SetVline
\begin{algorithm}[t]
\caption{\footnotesize Column generation based \randomRank.
}
\begin{algorithmic}
\scriptsize{
   \KwIn{
     \\$-$   A set of examples $( \bx_i,y_i ) $, $i=1 \cdots m$;
     \\$-$   The maximum number of weak classifiers, $T$;
     \\$-$    Random projection matrices, $\Prr \in {\Real}^{n \times d}$, $r = 1 \cdots k$;
   }

   \KwOut{
      The learned multi-class classifier
      $F(\bx) = \argmax_{r=1\cdots k} \sum_{t=1}^{T} w_t h_t( \Prr
      \bx)$. %
}

\KwInit {
   \\$-$      $t \leftarrow 0$;

   \\$-$      Initialize sample weights, $u_{ir} = \frac{1}{(m-1)k}$;
}

\While{ $t < T$ }
{
  \cone\ Train a weak learner, $h_t(\cdot)$, using \eqref{EQ:rankWeak};
  \\ \ctwo\ If the stopping criterion,
  ${\textstyle \sum_{ir}} u_{ir} \delta h_t(\Pyi, \Prr, \bx_i) \leq \nu + \epsilon$,
   has been met, stop the training;
  \\ \cthree\ Add the best weak learner, $h_t(\cdot)$,  into the current
  set, by solving \eqref{EQ:rankWeak};
  \\ \cfour\ Solve the primal problem, \eqref{EQ:rankPrimal}, or dual problem \eqref{EQ:rankDual};
  \\ \cfive\  If the primal problem is solved, update sample weights (dual
  variables) using \eqref{EQ:rankKKT};
  \\ \csix\ $t \leftarrow t + 1$;
}
} %
\end{algorithmic}
\vspace{-.2cm}
\label{ALG:alg1}
\end{algorithm}

\paragraph{Stage-wise boosting}
The advantage of the algorithm outlined above is that it is {\em totally
corrective} in that the primal variables, $\bw$, are updated at each
boosting iteration.  However, the training of this approach can be
expensive when the number of training data and classes are large.  In
this section, we design a more efficient approach by minimizing the
loss function in a stage-wise manner, similar to those derived in
AdaBoost.  Looking at the primal problem \eqref{EQ:rankPrimal}, the
optimal $\bw$ can be calculated analytically as follows.
At iteration $t$, we fix the value of $w_1$, $w_2$, $\cdots$, $w_{t-1}$.
So $w_t$ is the only variable to optimize.
The primal cost function can then be written as\footnote{We can simply
set  $ \nu $ to be zero in stage-wise boosting.
Following the framework of gradient-descent boosting of
\cite{Friedman2000Additive,masonboosting}, we can obtain
the same formulation as described here.
}
\begin{align}
    \label{EQ:stagewise1}
    F_p  = \sum_{ir} Q_{ir} \exp \bigl( - \delta h_t(\Pyi,\Prr,\bxi) w_t \bigr),
\end{align}
where $Q_{ir} = \exp \bigl( -\sum_{j=1}^{t-1}  \delta h_j(\Pyi,\Prr,\bxi) w_j \bigr)$.
If we use discrete weak learners, $h(\cdot) \in \{-1,+1\}$,
then $\delta h_t(\cdot) \in \{-2,0,2\}$, and
$F_p$ can be simplified into:
\begin{align}
    \label{EQ:stagewise2}
    F_p = { \sum_{\delta h_t=0}} Q_{ir} & + { \sum_{\delta h_t=2}} Q_{ir} \exp( -2 w_t )
      + { \sum_{\delta h_t=-2}} Q_{ir} \exp( 2 w_t ).
\end{align}
Let $Q_{+} = \sum_{\delta h_t=2} Q_{ir}$ and $Q_{-} = \sum_{\delta h_t=-2} Q_{ir}$,
then $F_p$ is minimized when
\begin{align}
    \label{EQ:stagewise3}
    w_t = \frac{1}{4} \log \left( \frac{Q_{+}}{Q_{-}} \right).
\end{align}

When real-valued weak learners are used, with the output in $\left[-1, 1\right]$, we
can calculate $w_t$ by minimizing the upper bound of $F_p$ as follows,
\begin{align*}
    F_p \leq \sum_{ir} Q_{ir} \Bigl[ & 0.5 \exp(w_t) \bigl(1 - \delta h_t(\Pyi,\Prr,\bxi) \bigr) \\ \notag
         &+ 0.5 \exp(-w_t) \bigl(1 + \delta h_t(\Pyi,\Prr,\bxi) \bigr) \Bigr].
\end{align*}
Here we have used the fact that $\exp(-w h) \leq 0.5 \exp(w) (1 - h) +
0.5 \exp(-w) (1 + h)$.
Similarly $ F_p $ is minimized when
\begin{align}
    \label{EQ:stagewise5}
    w_t = \frac{1}{2} \log \left( \frac{1+b}{1-b} \right),
\end{align}
where $b =  {\textstyle \sum_{ir}} Q_{ir} \delta h_t(\Pyi,\Prr,\bxi)$.
For the stage-wise boosting algorithm, we simply replace Step \cfour\
 in Algorithm~\ref{ALG:alg1} with \eqref{EQ:stagewise3} or
\eqref{EQ:stagewise5}.  Note that the formulation of our stage-wise
boosting is similar to that of RankBoost proposed by Freund \etal
\cite{Freund2003Efficient}.
Besides the efficiency of optimization at each iteration, this
stage-wise optimization does not have any parameter to tune. One only
needs to determine when to stop. The disadvantage, compared with
totally corrective boosting
\cite{Shen2011Totally}, is that it may need more iterations to
converge.

\paragraph{General convex loss}
The following derivations are based on the important concept of convex conjugate or
Fenchel duality from convex optimization.

\begin{dfn}[Convex Conjugate]
    Let $ f: \Real ^ n \rightarrow \Real $. The function $ F^* : \Real ^ n \rightarrow
    \Real $, defined as
    \begin{equation}
        \label{EQ:App1}
        f^*  ( \bu ) = \sup_{ \bx \in {\rm dom } f }  [  \bu^\T \bx - f( \bx)  ],
    \end{equation}
    is the convex conjugate or Fenchel duality of the function $ f(\cdot) $. The domain of
    the conjugate function consists of $ \bu \in \Real^n$ for which the supremum is
    finite.
\end{dfn}

It is easy to verify that $ f^* (\cdot) $ is always convex since it is the point-wise
supremum of a family of affine functions of $ \bu $. This holds even if $ f(\cdot) $ is
not a convex function.

    If $ f(\cdot) $ is convex and closed, then $ f^{**}  = f$.
    For a point-wise  loss function,
    $ \lambda ( \boldsymbol \rho ) =  \sum_i \lambda ( \rho_i )  $,
    the convex conjugate of the sum is the sum of the convex conjugates:
    \begin{align}
        \lambda^* ( \bu )  = \sum_{\boldsymbol \rho}
        \left\{
                \bu ^\T { \boldsymbol \rho }
                - \sum_i \lambda ( \rho_i  )
        \right\}
        &= \sum_i \sup_{ \rho_i } \{
                        u_i \rho_i - \lambda(\rho_i )
        \}
        \notag
        \\
        &= \sum_i \lambda ^*  ( u_i ).
        \label{EQ:App1B}
    \end{align}
    We consider functions of Legendre type here, which
    means, the gradient $ f'(\cdot) $ is defined on the domain of $ f(\cdot) $
    and is an isomorphism between the domains of $ f(\cdot) $ and $ f^* ( \cdot ) $.

    The general $ \ell_1 $-norm  regularized  optimization problem we want to learn a
    classifier is
    \begin{align}
        \label{EQ:App2}
        \min_{ \bw, \boldsymbol \rho }
        \;
        &
        \;
        \sum_{ ir }  \lambda (   \rho_{ir}   )  + \nu {\bf 1} ^\T \bw
        \notag
        \\
        \st \; & \rho_{ir} = \dH (  \P^{ ( y_i )  } , \P^{  (  r  )  } , \bx_i  ) ^\T \bw ,
        \bw \geq 0.
     \end{align}
        Here $  \lambda( \cdot )   $ is a  convex surrogate of the zero-one loss, \eg,
        the exponential loss, logistic regression loss.
        We assume that $  \lambda( \cdot )   $ is smooth.  

        Although the variable of interest is $ \bw$,
        we need to keep the auxiliary variable $ \boldsymbol \rho $
        in order to derive a meaningful dual problem. The  Lagrangian is
        \begin{align*}
            L & =  \sum_{ir} \lambda( \rho_{ir}  )   + \nu {\bf 1}^\T \bw \notag \\
               & \quad
               - \sum_{ ir } u_{ir}
                ( \rho_{ir}  -
                \dH (  \P^{ ( y_i )  } , \P^{  (  r  )  } , \bx_i  ) ^\T \bw )
                - \bq ^\T
                 \bw
                \\
                & =
                 \Bigl[
                 \nu {\bf 1} ^\T  +
                        \sum_{ ir } u_{ir}
                \dH (  \P^{ ( y_i )  } , \P^{  (  r  )  } , \bx_i  ) ^\T
                - \bq \Bigr] \bw
                 \notag
                \\
               &  \quad
                 - \left[
                 \sum_{ir} u_{ir} \rho_{ir} - \sum_{ir} \lambda( \rho_{ir} )
                 \right].
        \end{align*}
        In order for $ L $ to  have finite infimum over the primal variables,
        the first term of $ L $ must be zero, which leads to
        \begin{equation}
            \label{EQ:App10}
            \nu {\bf 1} ^\T  +
                        \sum_{ ir } u_{ir}
                \dH (  \P^{ ( y_i )  } , \P^{  (  r  )  } , \bx_i  ) ^\T \geq 0.
        \end{equation}
        The infimum of the second term of $ L $ is  $  - \sum_{ir} \lambda ^* ( u_{ir} )  $
        by using \eqref{EQ:App1} and \eqref{EQ:App1B}.
        Therefore the Lagrange dual problem of \eqref{EQ:App2} is
        \begin{align}
            \label{EQ:AppDual1}
                \max_{ \bu }
        \;
        &
        \;
        - \sum_{ir} \lambda^*( u_{ir} ), 
        \;\, \st  \,  \eqref{EQ:App10}.
        \end{align}
        We can reverse the sign of the dual variable $ \bu $
        and rewrite \eqref{EQ:AppDual1} into its equivalent form
            \begin{align}
            \label{EQ:AppDual2}
            \min_{ \bu }
        \;
        &
        \;
        \sum_{ir} \lambda^*( - u_{ir} )
        \notag
        \\
        \st  \; &
        \sum_{ir} u_{ir} \dH (  \P^{ ( y_i )  } , \P^{  (  r  )  } , \bx_i  ) ^\T \leq \nu
        {\bf 1 }^\T.
        \end{align}
        The KKT condition between the primal \eqref{EQ:App2}
        and the dual \eqref{EQ:AppDual2} shows the relation of the primal and dual
        variables
        \begin{equation}
            \label{EQ:KKT-APP}
            u_{ir} = - \lambda'( \rho_{ ir } ),
        \end{equation}
        which holds at optimality.
        The dual variable $\bu$  is the negative gradient of the loss at
        $ \rho_{ir} $. This can be  obtained by setting the first derivative of $ L $ to
        be zeros. Under the assumption that both the primal and dual problems are feasible
        and the Slater's condition satisfies, strong duality holds between the primal and
        dual problems.

        We need to use column generation to {\em approximately} solve the original problem
        because the dimension of the primal variable $ \bw $ can be extremely large or
        infinite.  The high-level idea of column generation is to
         only consider a small subset of the variables in the primal; i.e., only a
        subset of $ \bw $ is considered. The problem solved using this
        subset is called the
        restricted master problem (RMP).
        It is well known that each primal variable corresponds to a
        constraints in the dual problem.
        Solving RMP is equivalent to solving a relaxed version of the
        dual problem. 
        With a finite $ \bw $, the set of constraints in the dual problem are finite,
        and we can solve the dual problem such that it satisfies all
        the existing constraints. If we can prove
        that among all the constraints that we have not added to the dual problem, no
        single constraint is violated, then we can draw the conclusion
        that solving the restricted
        problem is equivalent to solving the original problem. Otherwise, there exists at
        least one constraint that is violated. The violated constraints correspond to
        variables in primal that are not in RMP. Adding these variables to RMP leads to a
        new RMP that needs to be re-optimized.
        To speed up convergence, one typically finds the most violated
        constraint in the dual by solving the following problem,
        according to the constraint in
        \eqref{EQ:AppDual2}:
        \begin{equation}
            \label{EQ:AppWL}
            \max_{ h(\cdot)  }  \sum_{ir} u_{ir} \delta h(\Pyi,\Prr, \bxi).
        \end{equation}

        We only need to change the primal and dual problems involved in
        Algorithm \ref{ALG:alg1} to obtain the column generation based multi-class random boosting with
        a {\em general} convex loss function. Specifically, only two lines need a change
        in Algorithm \ref{ALG:alg1} and the rest remains identical:

        Step \cfour:  Solve the primal problem \eqref{EQ:App2}, or the dual problem
        \eqref{EQ:AppDual2};

        Step \cfive: If the primal problem is solved, update the dual variable $ \bu $
        using \eqref{EQ:KKT-APP}.

        Note that the derivation of the dual problem \eqref{EQ:rankDual} also
        follows the above analysis (using the fact that the convex conjugate of the log-sum-exp function is the Shannon entropy).
        Mathematically the convex conjugate of $ f(\bx) = \log(
        \sum_i x_i ) $ is  $ f^*( \bu ) = \sum_i u_i \log u_i, $ if $ \bu \geq 0 $ and $
        \sum_i u_i = 1$; otherwise  $ f^*( \bu ) = \infty $.

\subsection{Multi-class boosting by randomly projecting  weak
classifiers' outputs}

\label{Sec:app2}

In contrast to the approach proposed in the previous section, where we
randomly project the original data to new spaces, we can also randomly
project the output of weak classifiers, $\bH$, to new spaces.
Our intuition is that if $\bH$ is linearly separable then the randomly projected data, $\bP \bH^{\T}$,
is likely to be linearly separable as well, as long as the random
projection matrices satisfy some mild assumptions
\cite{Fradkin2003Experiments}.
As in the previous approach, we learn a multi-class classifier based
on pairwise comparisons.
We create $k$ pre-defined random projection matrices, $\bP^{(1)}, \bP^{(2)}, \cdots, \bP^{(k)}$, one for each class.
Given a training instance $(\bxi, \yi)$ and the weak classifiers'
responses, $H_{i:}$, the  condition $\Pyi H_{i:}^{\T} \bw  >  \Prr
H_{i:}^{\T} \bw, \forany r \neq \yi$ has to be satisfied.
The intuition is the same as in the previous case:
the correct model's response should be
larger than all the incorrect models' responses.
Note that in this approach, $\bw \in {\Real}^{n}$, \ie, it has a fixed
size and is independent of the number of boosting iterations (as
compared to the previous approach where the size of $\bw$ is equal to
the number of boosting iterations, $\bw \in {\Real}^{T}$).  We define
a margin associated with the above condition as $\rho_{ir} = \Pyi
H_{i:}^{\T} \bw  - \Prr H_{i:}^{\T} \bw$.  Now  the margin has
been defined, and the learning  can be
solved within the large-margin framework, as described in the previous section.
We now apply the
logistic loss due to its robustness in handling noisy data
\cite{Friedman2000Additive}. Again, any other convex surrogate loss
can be used.
Since the projected space, $\bP \bH^{\T}$, can also be much larger
than the original space, $\bH$, we expect that some projected features
might turn out to be irrelevant.
We also apply the $\ell_1$-norm regularization as in
\cite{Shen2011Totally}, resulting in the following learning problem:
\begin{align}
    \label{EQ:logPrimal}
        \min_{ \bw, \brho }   \quad
        &
        \frac{1}{mk} \sum_{i=1}^{m} \sum_{r=1}^{k} \log \bigl( 1 + \exp \left( -\rho_{ir} \right) \bigr) +
        \nu  {\boldsymbol 1}^\T \bw     \\ \notag
        \st \quad &
        \rho_{ir} =  \Pyi H_{i:}^{\T} \bw  - \Prr H_{i:}^{\T} \bw, \forall i, \forall r;
        \quad
        \bw \geq 0.
\end{align}
Note that $\bw \geq 0$ enforces the non-negative constraint on $\bw$.
The Lagrangian of \eqref{EQ:logPrimal} can be written as
\begin{align}
    \label{EQ:log2}
    L =& \frac{1}{mk} \sum_{i,r} \log \bigl( 1 + \exp\left( -\rho_{ir} \right) \bigr) +
            \nu {\boldsymbol 1}^\T  \bw   \\ \notag
        &
        -\sum_{i,r} u_{ir}
        ( \rho_{ir} - \Pyi H_{i:}^{\T} \bw + \Prr H_{i:}^{\T} \bw -  \bp^{\T} \bw ),
\end{align}
with $ \bu \geq 0$ and $\bp \geq 0$.
At optimum, the first derivative of the Lagrangian w.r.t.\ the primal variables, $\bw$, must be zeros,
\begin{align}
    \label{EQ:log3}
    \frac{\partial L}{\partial \bw } =  {\bf 0}
    &\rightarrow
        \sum_{i,r}
        u_{ir}  \left( \Pyi - \Prr \right) H_{i:}^{\T}
        = \bp^{\T} - \nu {\bf 1}^{\T} \\ \notag
    &\rightarrow
        \sum_{i,r}
            u_{ir} \Delta \P (y_{i}, r) H_{i:}^{\T}
            = \bp^{\T} - \nu {\bf 1}^{\T},
\end{align}
where $\Delta \P (y_{i}, r) = \Pyi - \Prr$.
By taking the infimum over the primal variables, $\rho_{ir}$,
\begin{align}
    \label{EQ:log4}
    \frac{\partial L}{\partial \rho_{ir}} = 0 \rightarrow
        \rho_{ir} = -\log\left( \frac{-m k u_{ir}}{m k u_{ir}-1} \right), \forall i, \forall r,
\end{align}
and
\begin{align}
    \label{EQ:log5}
    \inf_{\rho_{ir}} L =
       & \frac{1}{mk}
        \sum_{i,r}  \Bigl[  - (1 + mk u_{ir})  \log\left(1+mk u_{ir}\right)
                 \\ \notag
                 &
                 -mk u_{ir} \log \left( -mk u_{ir} \right)  \Bigr].
\end{align}
Reversing the sign of $\bu$, the Lagrange dual can be written as
\begin{align}
    \label{EQ:logDual}
        \max_{ \bu }   \quad
        &
        -\frac{1}{mk} \sum_{i=1}^{m} \sum_{r=1}^{k}
            \Bigl[mk u_{ir} \log \left( mk  u_{ir} \right)
            +
            \\ \notag & \qquad
            \left(1 - mk u_{ir}\right) \log \left( 1 - mk
            u_{ir} \right) \Bigr]
         \\ \notag
        \st \quad &
        \sum_{i,r}
            u_{ir}  \Delta \P (y_{i}, r) H_{i:}^{\T}
            < \nu {\bf 1} ^{\T}. %
\end{align}
    Note that here the number of constraints is equal to the size of
    the new space ($n$).  At each iteration, we choose the weak
    learner, $h_t(\cdot)$, that most violates the dual constraint in
    \eqref{EQ:logDual}.  The subproblem of generating the weak
    classifier at iteration $t$ can then be expressed as:
\begin{align}
    \label{EQ:logWeak}
    h^{\ast}(\cdot) = \argmax_{h(\cdot), v }
    \;
     \sum_{i,r}  u_{ir}  \left[ H_{i,1:t-1}, h(\bxi) \right] \Delta
     { \bp} (y_{i}, r; v ),
\end{align}
$v = 1, \ldots, n, \,\, \forall h(\cdot) \in \mathcal{H}$ and
\[
\Delta { \bp} (y_{i}, r ; v ) =
\left[ P^{(y_i)}_{v1} - P^{(r)}_{v1},
       \cdots,
       P^{(y_i)}_{vt} - P^{(r)}_{vt}
       \right]^\T \in {\Real}^{T \times 1}.
    \]
Here a significant difference compared with conventional boosting is that
the selection of the current best weak classifier depends on all
previously
selected weak classifiers.

    The idea behind our approach is that  performance improves as more
    weak classifiers, $h(\cdot)$, are added to the constraint.
    This process can continue
    as long as there exists at least one constraint that is violated, \ie,
    \[
    \max \left( {\textstyle \sum}_{i,r} U_{ir} \Delta \P (y_{i}, r) H_{i,1:t}^{\T}
    \right) < \nu + \epsilon,
    \]
    or when adding an additional weak classifiers ceases to have a significant impact on the objective value of \eqref{EQ:logPrimal}, \ie,
    $  \left| \frac{ \Opt_{t-1} - \Opt_{t} }{ \Opt_{t-1} } \right| < \epsilon$.
    In our experiments we use the latter as our stopping criterion.
    Through the KKT optimality condition, the gradient of Lagrangian \eqref{EQ:log2} over primal variables,
    $\brho$, and dual variables, $\bu$, must vanish at the optimum.
    The relationship between the optimal value of $\brho$ and $\bu$ can be expressed as
\begin{align}
    \label{EQ:logKKT}
    u_{ir} = \frac{ \exp(-\rho_{ir}) }{ mk \bigl( 1+\exp(-\rho_{ir}) \bigr) }.
\end{align}
The details of our random projection based multi-class boosting algorithm are given in Algorithm~\ref{ALG:alg2}.

\SetKwInput{KwInit}{Initilaize}
\SetVline

\begin{algorithm}[t]
\caption{\footnotesize Column generation based \randomProj.
}
\begin{algorithmic}
\scriptsize{
   \KwIn{
     \\$ - $    A set of examples $ ( \bx_i,y_i ) $, $i=1 \cdots m$;
     \\$ - $    The maximum number of weak classifiers, $T$;
     \\$ - $    Random projection matrices, $\Prr \in {\Real}^{n \times T}$, $r = 1 \cdots k$;
   }

   \KwOut{
      A multi-class classifier \\
      $F(\bx) = \argmax_r \Prr [h_1(\bx), \cdots, h_T(\bx)]^{\T} \bw$. %
}

\KwInit {
   \\$ - $      $t \leftarrow 0$;

   \\$ - $      $H = \emptyset	$;

   \\$ - $      Initialize sample weights, $u_{ir} = \frac{1}{mk}$;
}
\While{ $t < T$ }
{
  \cone\     Train a weak learner, $h_t(\cdot)$, using \eqref{EQ:logWeak};
  \\ \ctwo\  If the stopping criterion,
   $  \left| \frac{ \Opt_{t-1} - \Opt_{t} }{ \Opt_{t-1} } \right| < \epsilon$,
   has been met, stop the training;
  \\ \cthree\ Add the best weak learner, $h_t(\cdot)$, into the current set $H$;
  \\ \cfour\  Solve the primal problem, \eqref{EQ:logPrimal}, \eg, using
  Quasi-Newton methods such as L-BFGS-B;
  \\ \cfive\  Update sample weights (dual variables) using \eqref{EQ:logKKT};
  \\ \csix\  $t \leftarrow t + 1$;
}
} %
\end{algorithmic}
\label{ALG:alg2}
\vspace{-.2cm}
\end{algorithm}

The use of general convex loss here follows the similar generalization
procedure as shown in the previous section. 

\begin{table*}
  \begin{center}
  \begin{tabular}{l|cc}
  \hline
   &  \randomRank   &  \randomProj  \\
  \hline
  \hline
    Pre-processing step for decision stumps  & $O\bigl(n mk \log(mk)\bigr)$ & $O(d m \log m)$ \\
    At each iteration    &   &      \\
    \hspace{3mm}Train the weak learner (decision stump)  & $O(nmk)$ &  $O(nmk + nmd)$ \\
    \hspace{3mm}Solve the optimization problem (\rboost) &  $O(mk)$ &  $O(n^{3})$ \\
  \hline
    Total computational complexity & $O\bigl(n mk \log(mk) + nmkT \bigr)$ &
        $O\bigl(d m \log m + nm(k+d)T + n^{3}T \bigr)$ \\
  \hline
  \end{tabular}
  \end{center}
  \caption{Computational complexity of \rboost.
  $m$ is the number of training samples.
  $n$ is the number of projected dimensions.
  $k$ is the number of classes.
  $d$ is the dimension size of the original data.
  $T$ is the number of iterations.
  Note that weak classifier training (learning decision stumps) 
take up most of the computation time for both methods.
  }
  \label{tab:time_comp}
\end{table*}

\subsection{Computational complexity}
\revised{
    We analyze the complexity of our new approaches in this section.
    For the sake of completeness, we also analyze the computational
    complexity of training weak classifiers.  For simplicity, we use a
    decision stump as our weak classifier.  Note that any weak
    classifier algorithms can be applied here.  For fast training of
    decision stumps, we first sort feature values and cache sorted
    results in memory.
At each boosting iteration, all decision stumps' thresholds will be searched
and the optimal decision stump $h^{\ast}(\cdot)$, which satisfies
\eqref{EQ:rankWeak} or \eqref{EQ:logWeak}, will be saved as the weak learner for the $t$-iteration.
For \randomRank (Algorithm~\ref{ALG:alg1}), the total number of pairwise relationships is $m(k-1)$.
We first sort features in each projected dimension.
This pre-processing step  requires $O\bigl(n mk \log(mk)\bigr)$ for sorting $n$ dimensions.
In Step \cone\, we train decision stumps for each projected dimension.
Step \cone\ takes $O(nmk)$.
Step \cfour\ can simply be ignored since it can be solved efficiently
using \eqref{EQ:stagewise3} or \eqref{EQ:stagewise5}.
Let the maximum number of iterations be $T$, the time complexity is $O(nmkT)$.
The total time complexity for \randomRank is  $ O\bigl(n mk \log(mk) +
nmkT \bigr)$.

For \randomProj (Algorithm~\ref{ALG:alg2}), the time required to sort $d$ features is
$O(d m \log m)$.
Step \cone\ finds the optimal weak learner that satisfies \eqref{EQ:logWeak}.
The multiplication, $u_{ir} \Delta {\bp}(y_{i}, r; v )$,  in \eqref{EQ:logWeak} takes $O(nmk)$
for all $n$ dimensions.
Training decision stumps requires $O(nmd)$.
Hence, Step \cone\ requires $O(nmk + nmd)$.
In Step \cfour\, we solve $n$ variables at each iteration.
Let us assume the computational complexity of L-BFGS-B is roughly cubic.
Hence, the time complexity for $T$ boosting iterations is $O\bigl(nm(k+d)T + n^{3}T \bigr)$ and
the total time complexity for \randomProj is $ O\bigl(d m \log m + nm(k+d)T + n^{3}T \bigr)$.
The computational complexity of both approaches is summarized in Table~\ref{tab:time_comp}.
Note that weak classifier training (learning decision stumps) 
take up most of the computation time for both methods.
}

\begin{figure*}[t]
    \begin{center}
        \includegraphics[width=0.23\textwidth,clip]{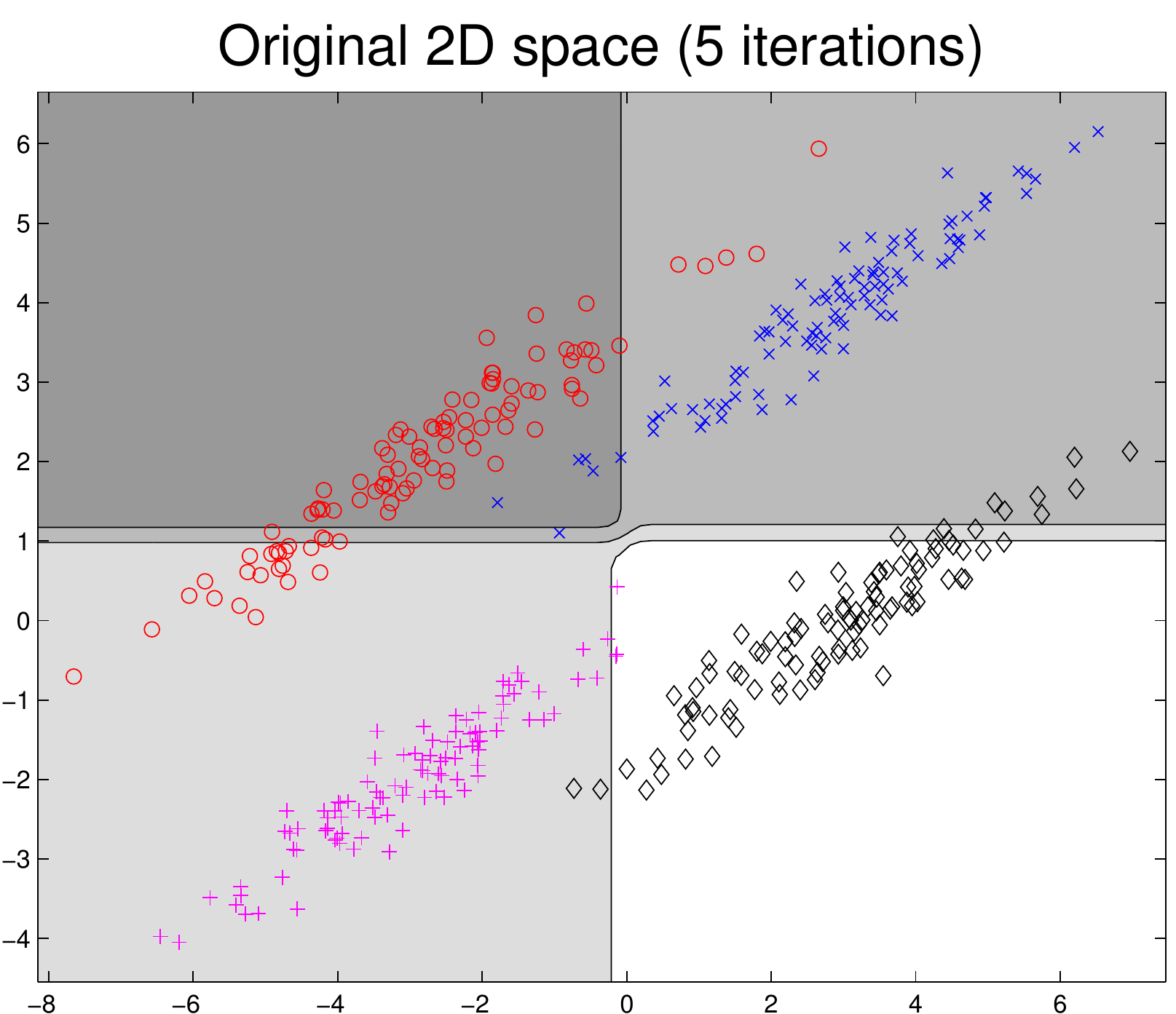}
        \includegraphics[width=0.23\textwidth,clip]{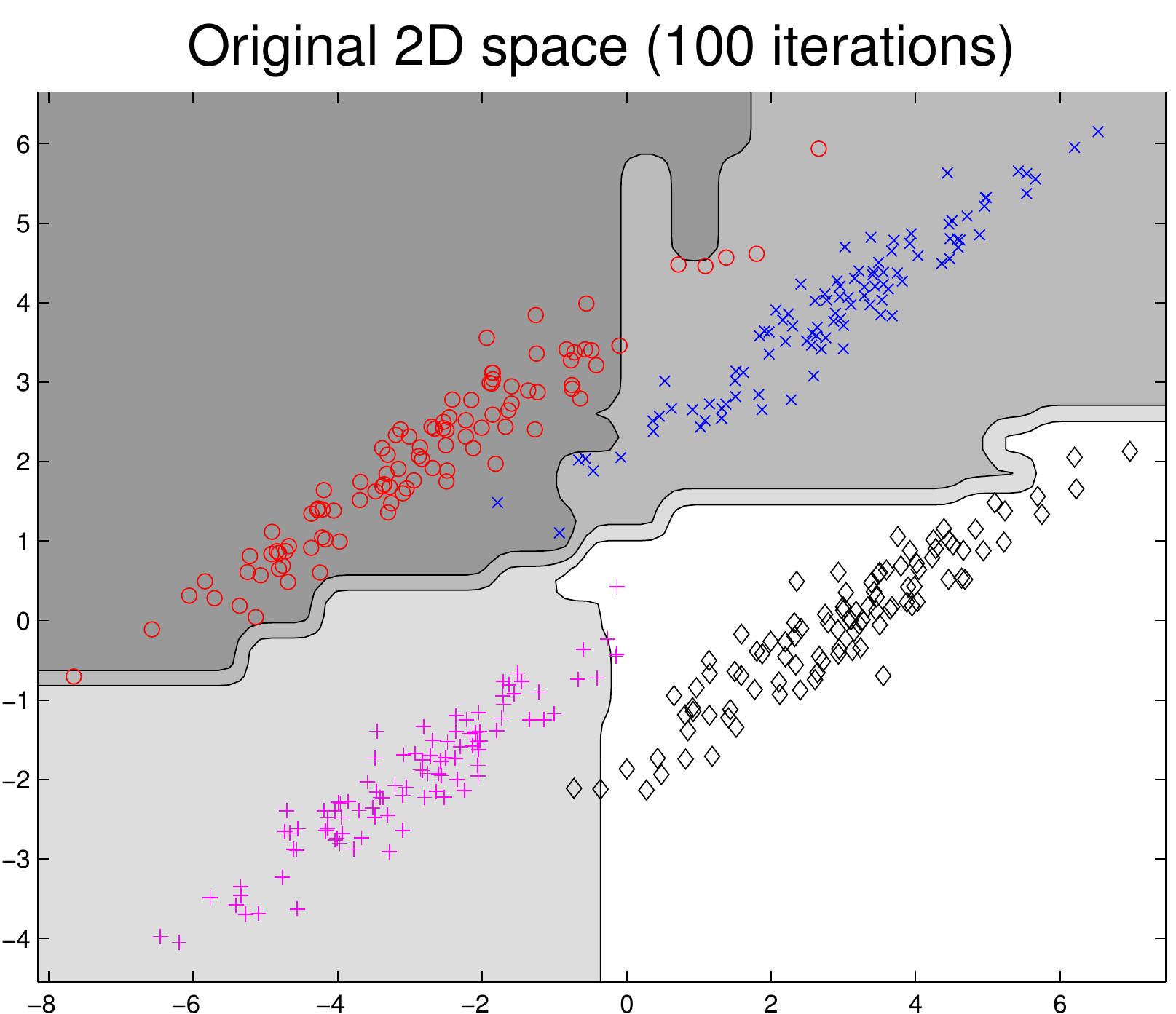}
        \includegraphics[width=0.23\textwidth,clip]{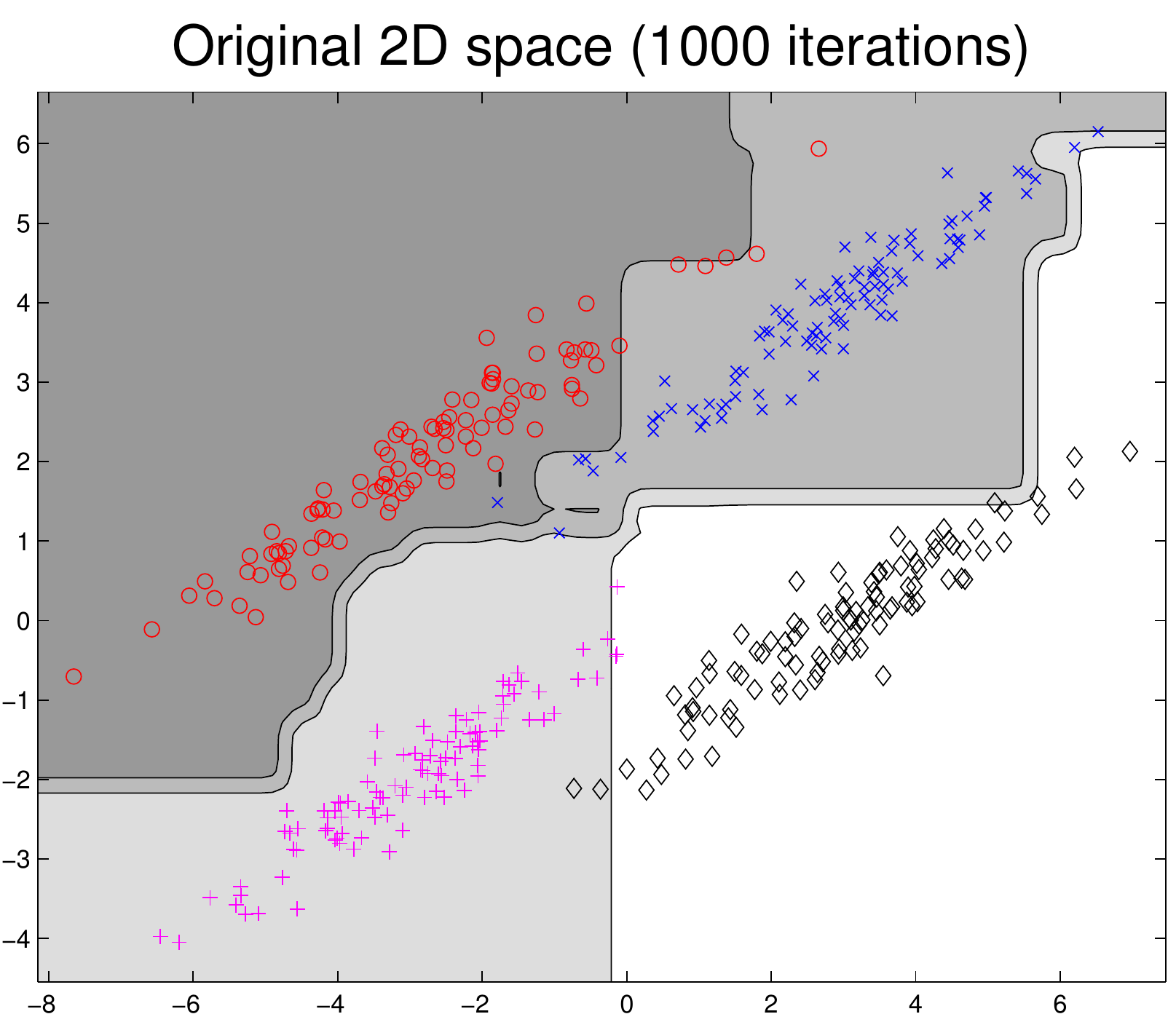}
        \includegraphics[width=0.23\textwidth,clip]{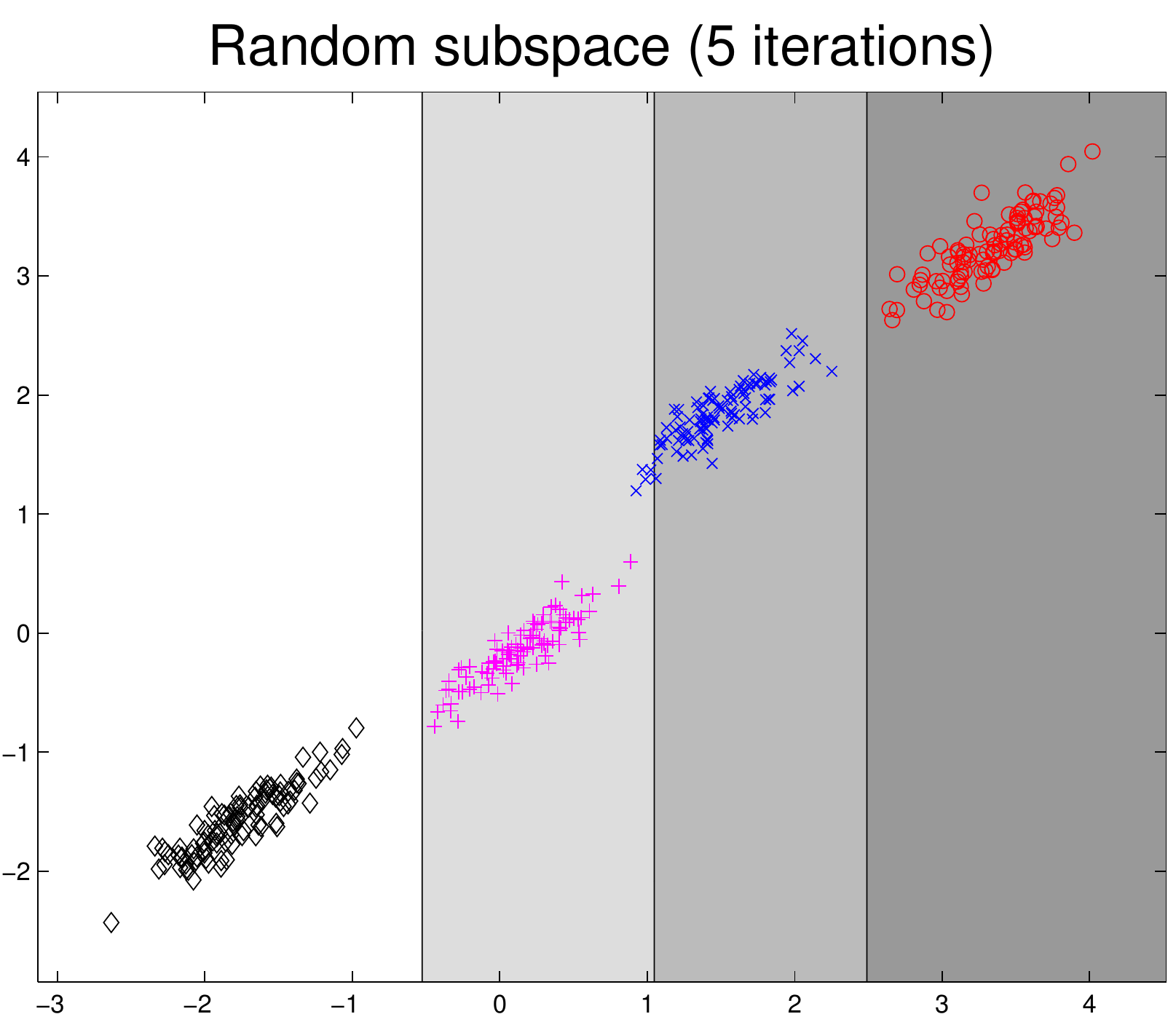}
    \end{center}
    \caption{
Decision boundaries on the artificial data set.
{\em First $3$ columns}: Classification of four diagonal distributions on 
the original two dimensional space
at $5$, $100$ and $1000$ boosting iterations.
{\em Last column}: Classification on randomly projected subspace 
(selectively chosen to
illustrate a better separation between classes).
    }
    \label{fig:toyRank}
\end{figure*}

\subsection{Discussion}
\paragraph{Advantage of applying random projections}
\revised{
One possible advantage of applying random projections is that
random projections may further increase class separation on some data sets.
We illustrate this in the following toy example.
We generate an artificial data set with four 
diagonal distributions.
Each diagonal distribution is randomly generated from
the multivariate normal distribution with covariance,
$\left[2.5, 1.5;1.5, 1 \right]$ and mean $\left[-3, 2 \right]$,
$\left[-3, -2 \right]$, $\left[3, 4 \right]$, $\left[3, 0 \right]$.
We train a one-versus-all boosting (with the decision stump as the weak learner)
and plot the decision boundary at $5$, $100$ and $1000$ boosting iterations.
We also randomly project the artificial data to the new $2$D space
and train the one-versus-all boosting classifier.
Decision boundaries of different examples are shown in Fig.~\ref{fig:toyRank}.
From the figure,
classification on the randomly projected subspace
(Fig.~\ref{fig:toyRank}: last column) clearly indicates its advantage
compared to classification on the original space.
}

\paragraph{Theoretical justification based on margin analysis}
In this section we justify the use of random projections on the proposed single-model multi-class classifier.
We begin by defining the margin on MultiBoost \cite{Shen2011Direct}
and its bound when the weak classifiers' response, $\bH$, is randomly
projected to the new space with a random projection matrix, $\bP$.
\begin{dfn}[Multi-class Margin for Boosting] Given a data set, $S = \{(\bxi \in \RR^d, y_i\in \Ycal = \{1, \cdots, k\})\}_{i=1}^m$,
the weak learners' response on the training data, $\bH$,
and weak learners' coefficients, $\bW = \left[ \bw_1^\T, \cdots \bw_k^\T \right]$ where $\bw_r \in \RR^T$ is weak learners' coefficients for class $r$.
The margin for boosting can be defined as,
\begin{align*}
\gamma = \min_{(\bx,y)\in S} \Big(\frac{\inner{\bw_y}{\bH(\bx)}}{\|\bw_y\|\|\bH(\bx)\|} -
\max_{y' \neq y}\frac{\inner{\bw_{y'}}{\bH(\bx)}}{\|\bw_{y'}\|\|\bH(\bx)\|}\Big).
\end{align*}
\end{dfn}
\begin{thm} [Margin Preservation] If the boosting has margin $\gamma$,
 then for any $\delta,\epsilon \in(0,1) $ and any \[n >
 \frac{12}{3\epsilon^2-2\epsilon^3} \ln\frac{6km}{\delta},\] with
 probability at least $1-\delta$, the boosting associated with
 projected weak learners' coefficients, $\bP\bw_r, \forany r$, and the
 projected weak learners' response, $\bP \bH$, has margin no less than
 \[ -\frac{1+3\epsilon}{1-\epsilon^2}+ \frac{\sqrt{1 - \epsilon^2} }{1 + \epsilon}+\frac{1 + \epsilon }{1 - \epsilon } \gamma.\]
\label{thm:multiclass}
\end{thm}
The above theorem shows that the multi-class margin can be well preserved
after both weak learner's coefficients, $\bW$, and weak learners' responses, $\bH$,
are randomly projected.
This theorem justifies the use of random projection on MultiBoost \cite{Shen2011Direct}.
The next theorem defines margin separability for the proposed single-vector parameterized multi-class boosting.
\begin{thm}[Single-vector Multi-class Boosting]
Given any random Gaussian matrix $\bR \in \RR^{n \times kT}$, whose
entry $\bR(i,j) = $ $ \frac{1}{\sqrt{n}}{a_{ij}}$ where $a_{ij}$ is i.i.d.
random variables from $\Ncal(0,1)$.  Denote $\bP_y \in \RR^{n,T}$ as
the $y$-th submatrix of $\bR$, that is $\bR =
        [\bP_1,$ $ \cdots,\bP_r,$ $\cdots, $ $ \bP_k]$.
If the boosting has margin
$\gamma$, then for any $\delta,\epsilon \in (0,1] $ and any \[n
>\frac{12}{ 3\epsilon^2-2\epsilon^3  } \ln{\frac{6m(k-1)}{\delta}} ,\]
there exists a single-vector $\bv \in \RR^n $, such that
\begin{align}
\Pr\Big( %
&\frac{\inner {\bv}{\bP_y  \bH(\bx) } -\inner {\bv}{\bP_{y'}\bH(\bx)}}{\|\bv\|\sqrt{\|\bP_y\bH(\bx) \|^2+\|\bP_{y'}\bH(\bx) \|^2}} \ge
\nonumber\\
&
\frac{-2\epsilon}{1-\epsilon}+\frac{1+\epsilon}{{\sqrt{2k}}(1-\epsilon)}\gamma \Big)\ge 1-\delta,
\quad \quad
            \forall y'\neq y.
\end{align}
\label{thm:single-param-trick}
\end{thm}
The above theorem reveals that there exists a single-vector
$\bv \in \RR^n $ under
which the margin is preserved up to an order of $O(\gamma/\sqrt{2k} )$.
In other words, the multi-class margin can be well preserved after random
projection as long as the newly projected dimension, $ n $,
satisfies some mild condition.
Not only the theorem justifies the use of random projection to learn
the single-model classifier, it also shows that
the projected dimensions, $n$, only grows logarithmically with
the number of classes, $k$.
This finding is important for problems where the number of classes is large.
Note that Theorem \ref{thm:single-param-trick} only applies to the
second approach presented in this work.

\section{Experiments}
\label{sec:exp}

We evaluate our approaches on artificial, machine learning and
visual recognition data sets and compare our approaches against
existing multi-class boosting algorithms.
For AdaBoost.ECC, we perform binary partitioning at each iteration using the random-half method \cite{Li2006Multiclass}.
Decision stumps are used as the weak classifier for all boosting algorithms.

  \begin{figure*}[t]
        \centering
        \begin{minipage}{1\textwidth}
               \scriptsize
                ~~~~~~~~~~~~~~~~ Data
                ~~~~~~~~~~~~~~~~~~ Ada.ECC
                ~~~~~~~~~~~~~~~~~ Ada.MH
                ~~~~~~~~~~~~~~~~~  Ada.MO
                ~~~~~~~~~~~~~ MultiBoost$^{\ell_1} $
                ~~~~~~~~~~~~~ \randomRank
                ~~~~~~~~~~~~~ \randomProj
        \end{minipage}
        \includegraphics[width=0.13\textwidth,clip]{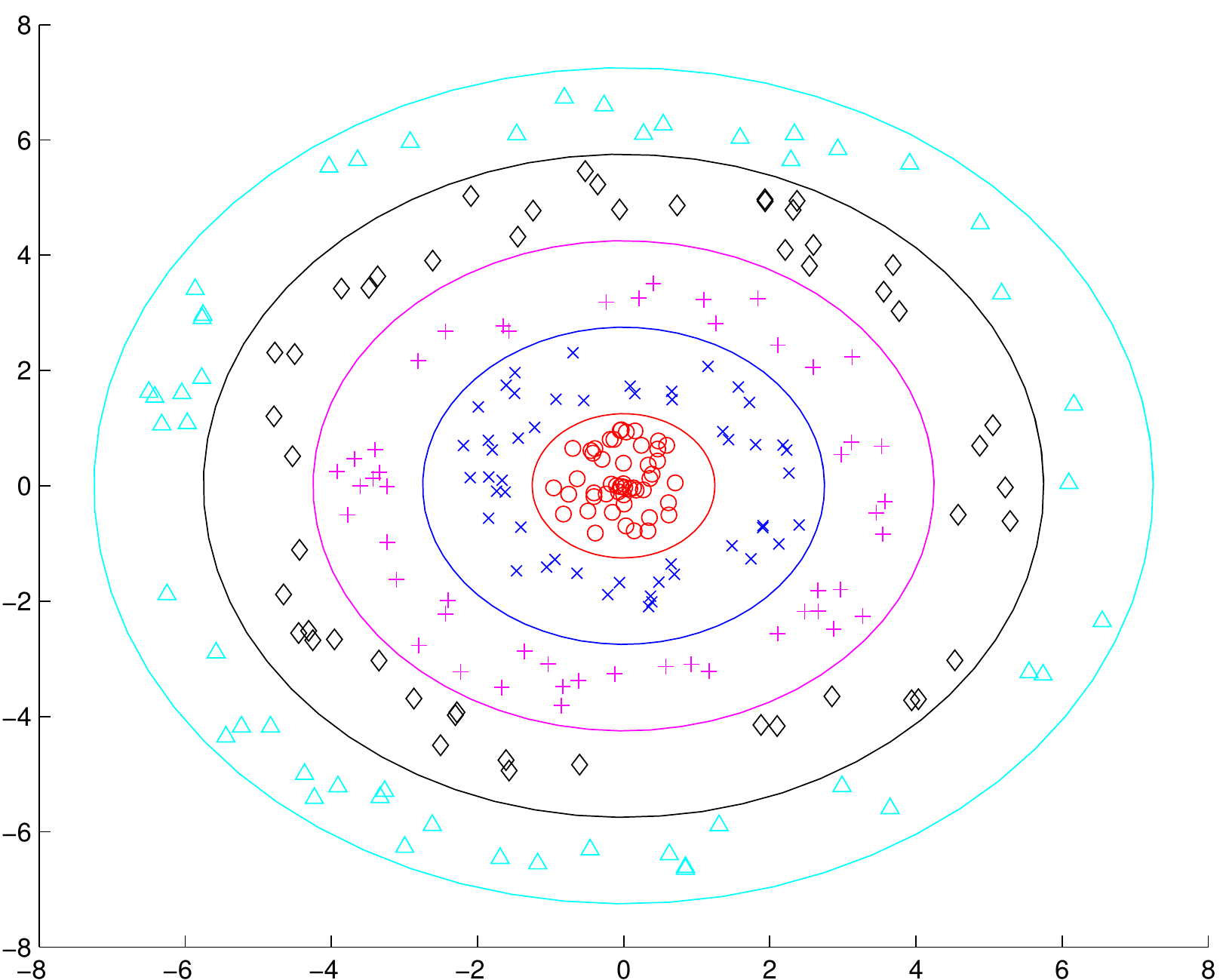}
        \includegraphics[width=0.13\textwidth,clip]{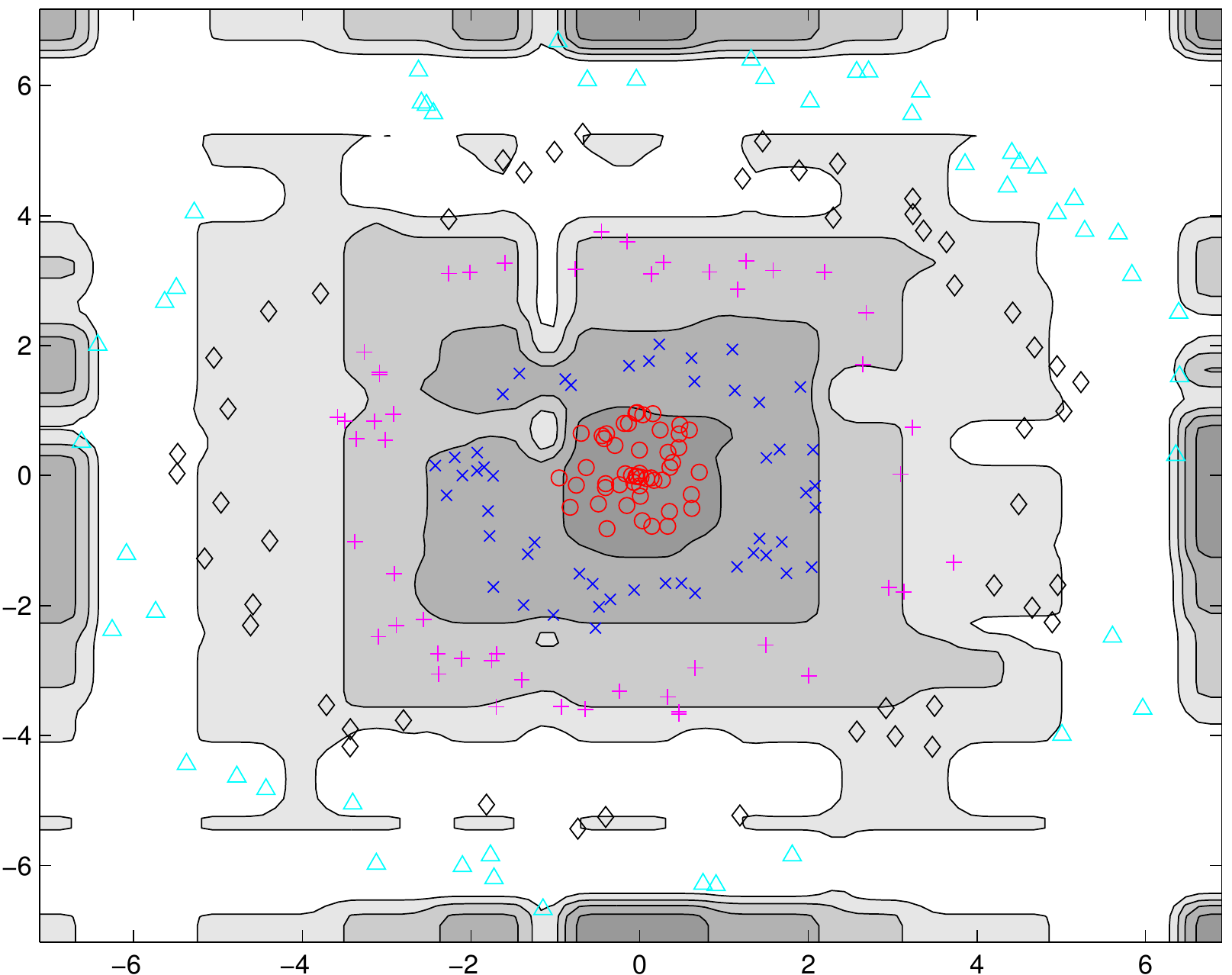}
        \includegraphics[width=0.13\textwidth,clip]{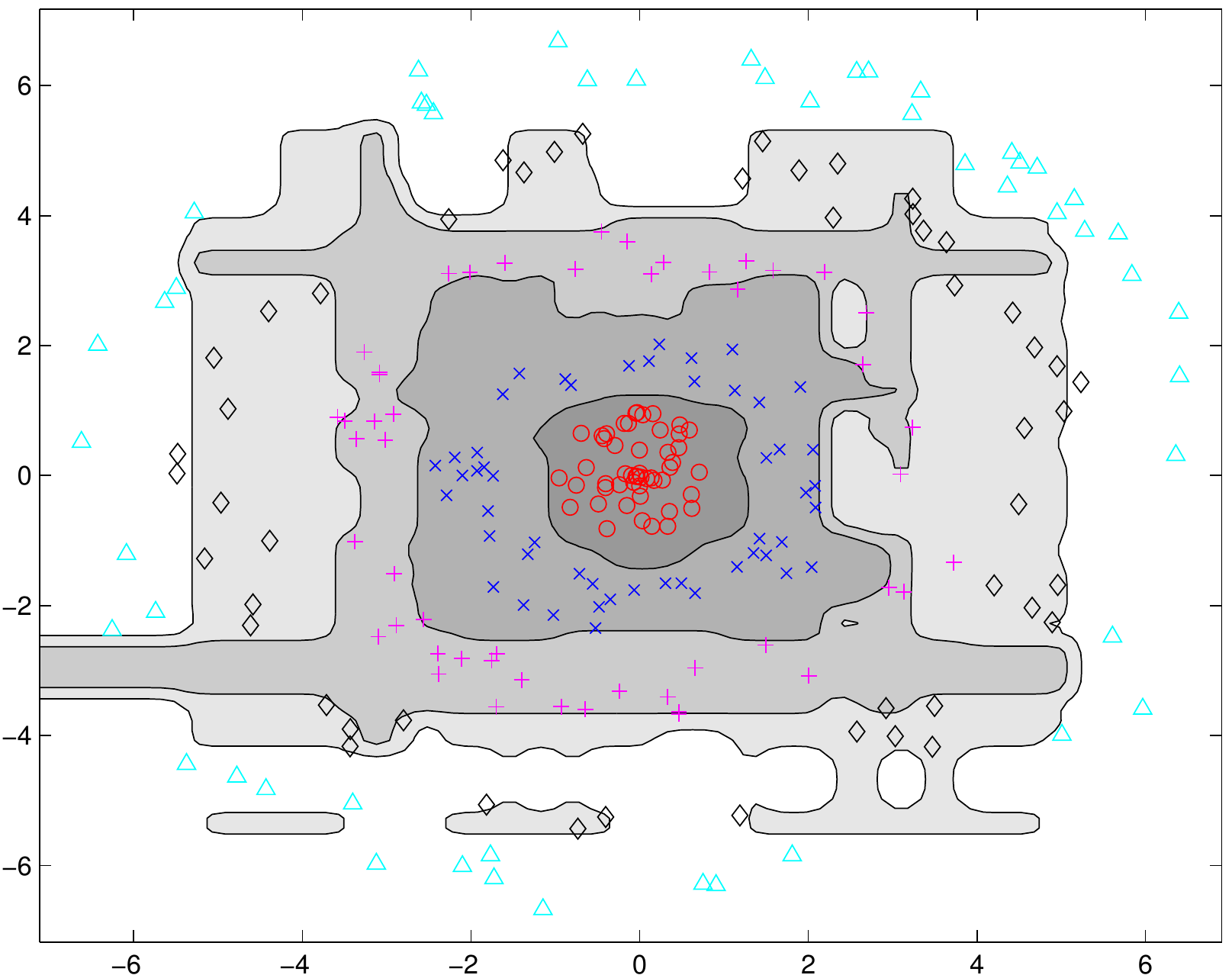}
        \includegraphics[width=0.13\textwidth,clip]{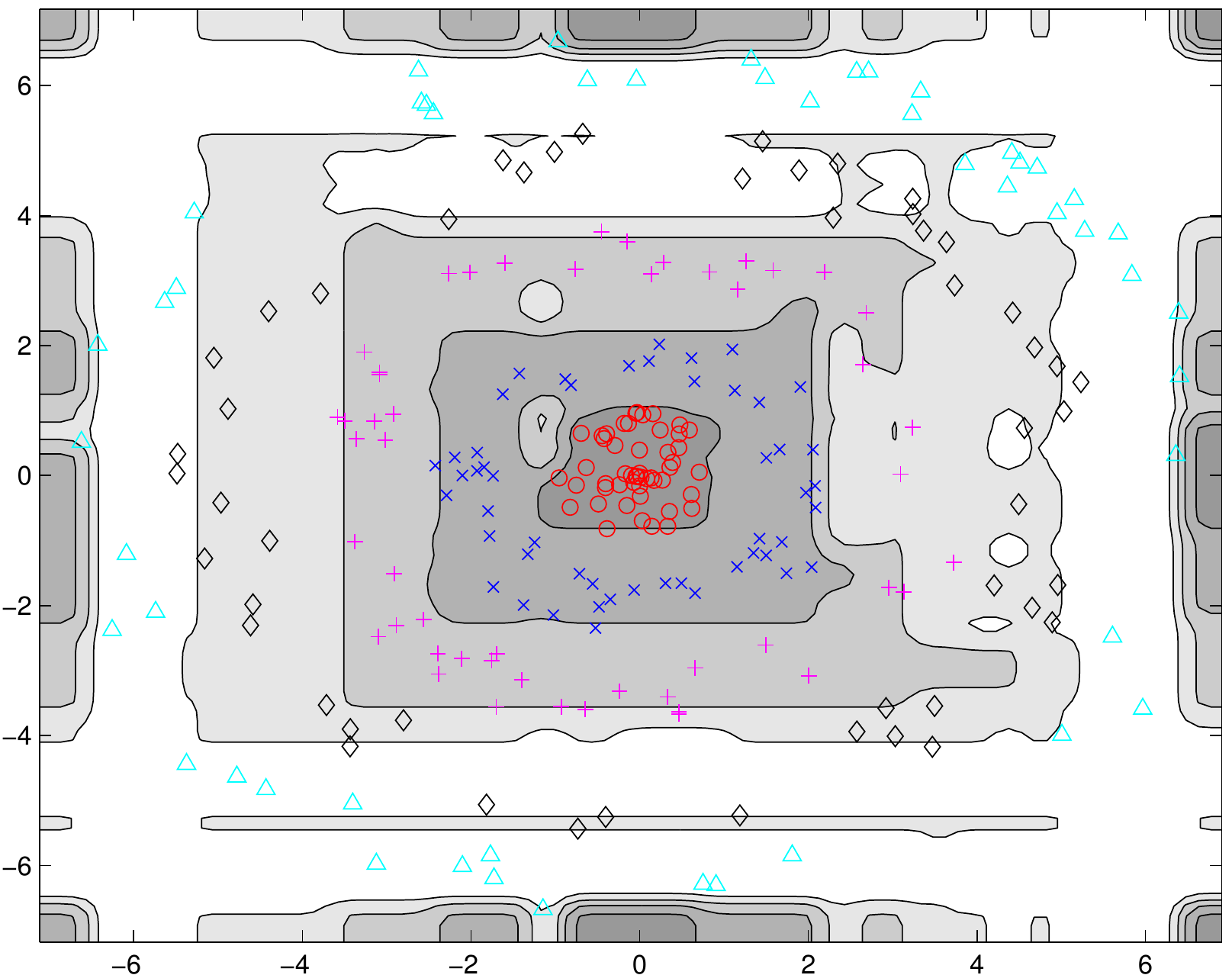}
        \includegraphics[width=0.13\textwidth,clip]{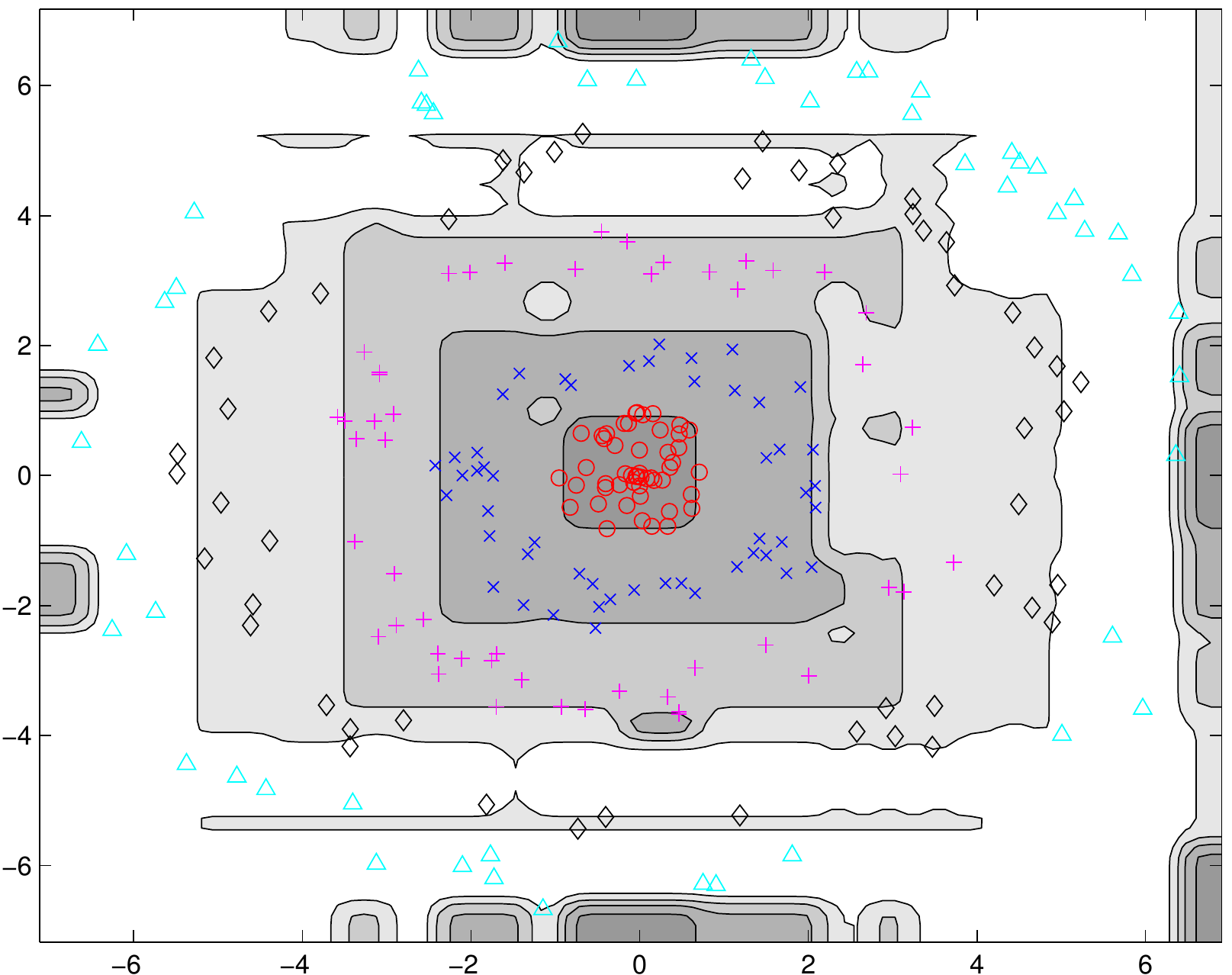}
        \includegraphics[width=0.13\textwidth,clip]{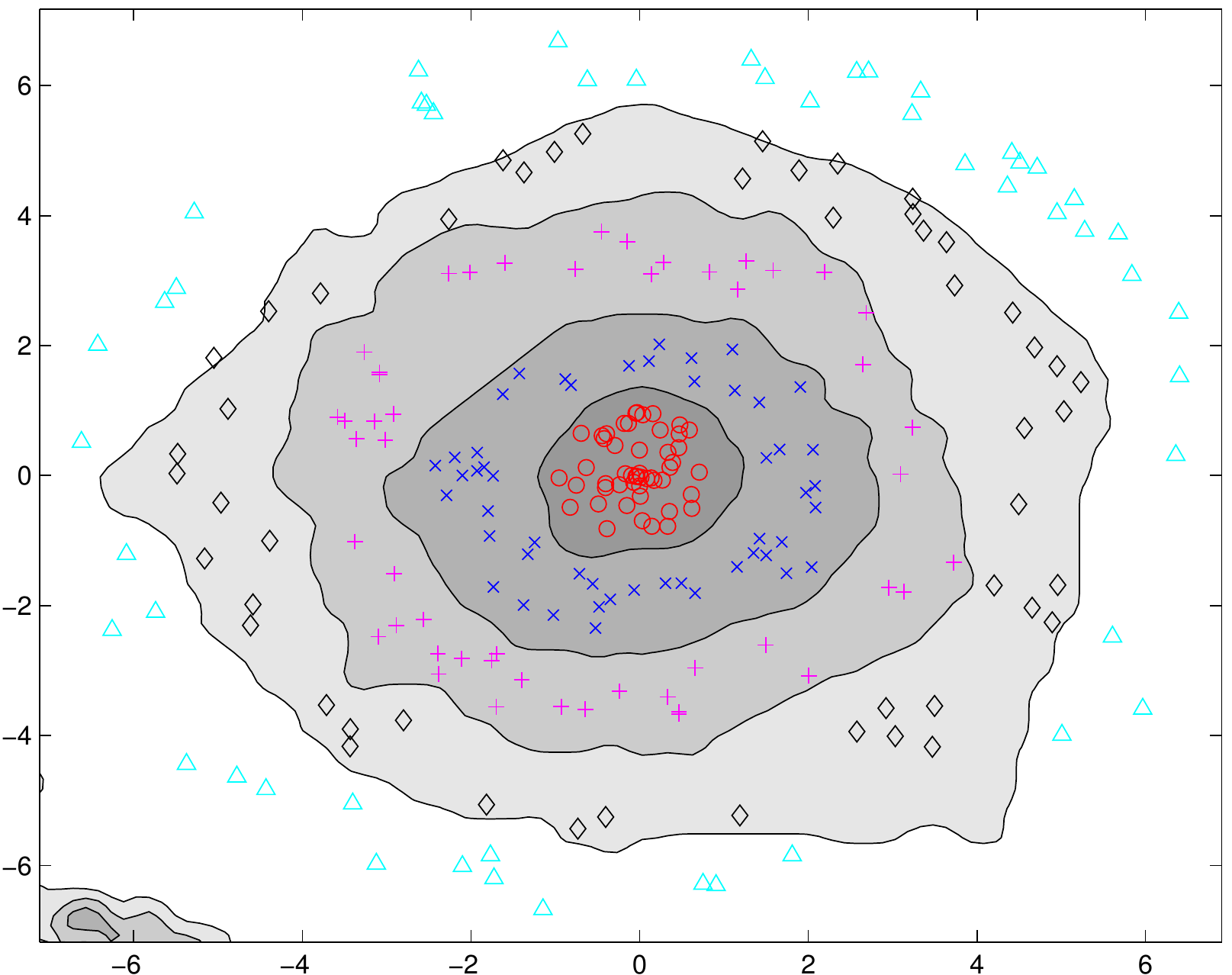}
        \includegraphics[width=0.13\textwidth,clip]{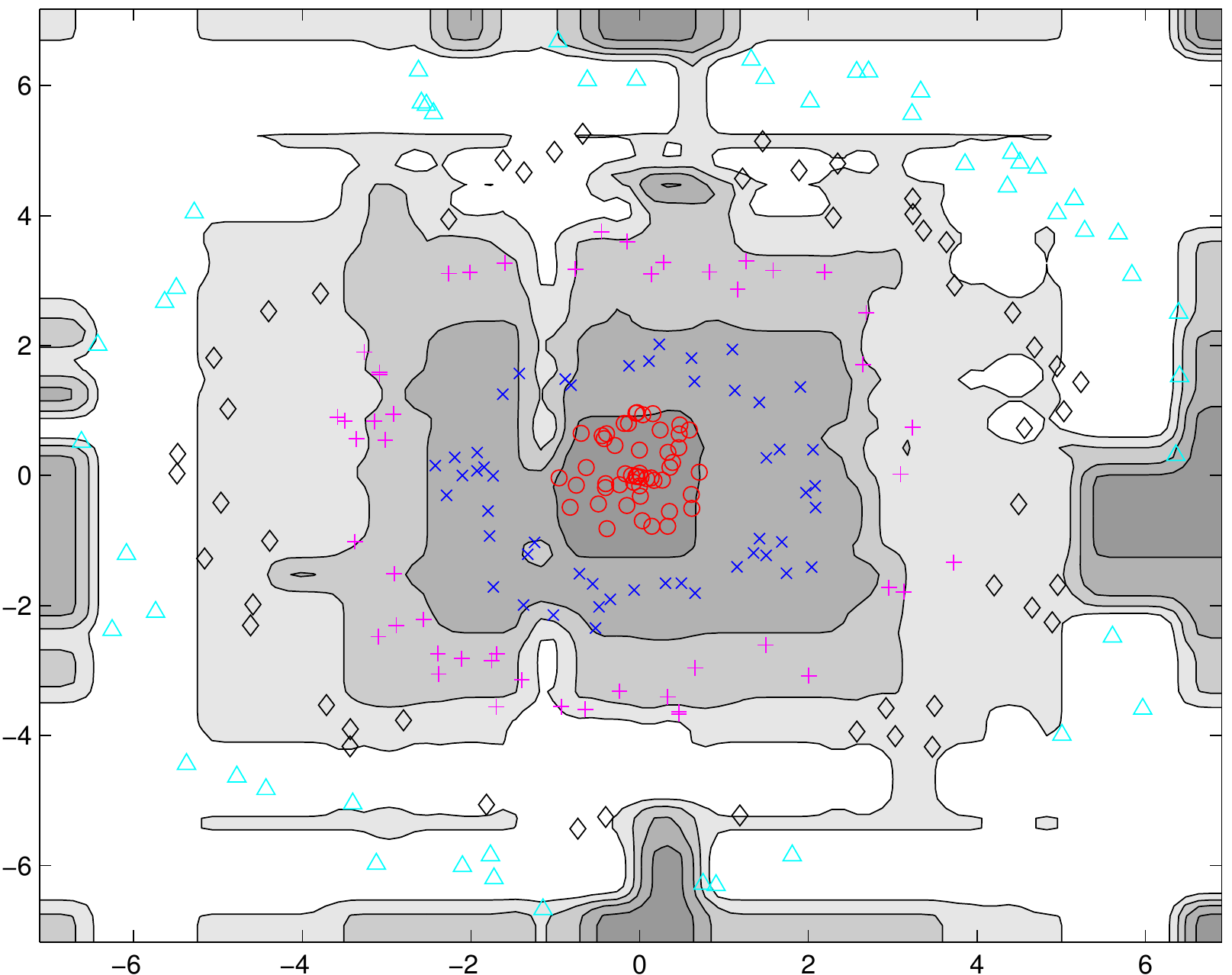}
        \includegraphics[width=0.13\textwidth,clip]{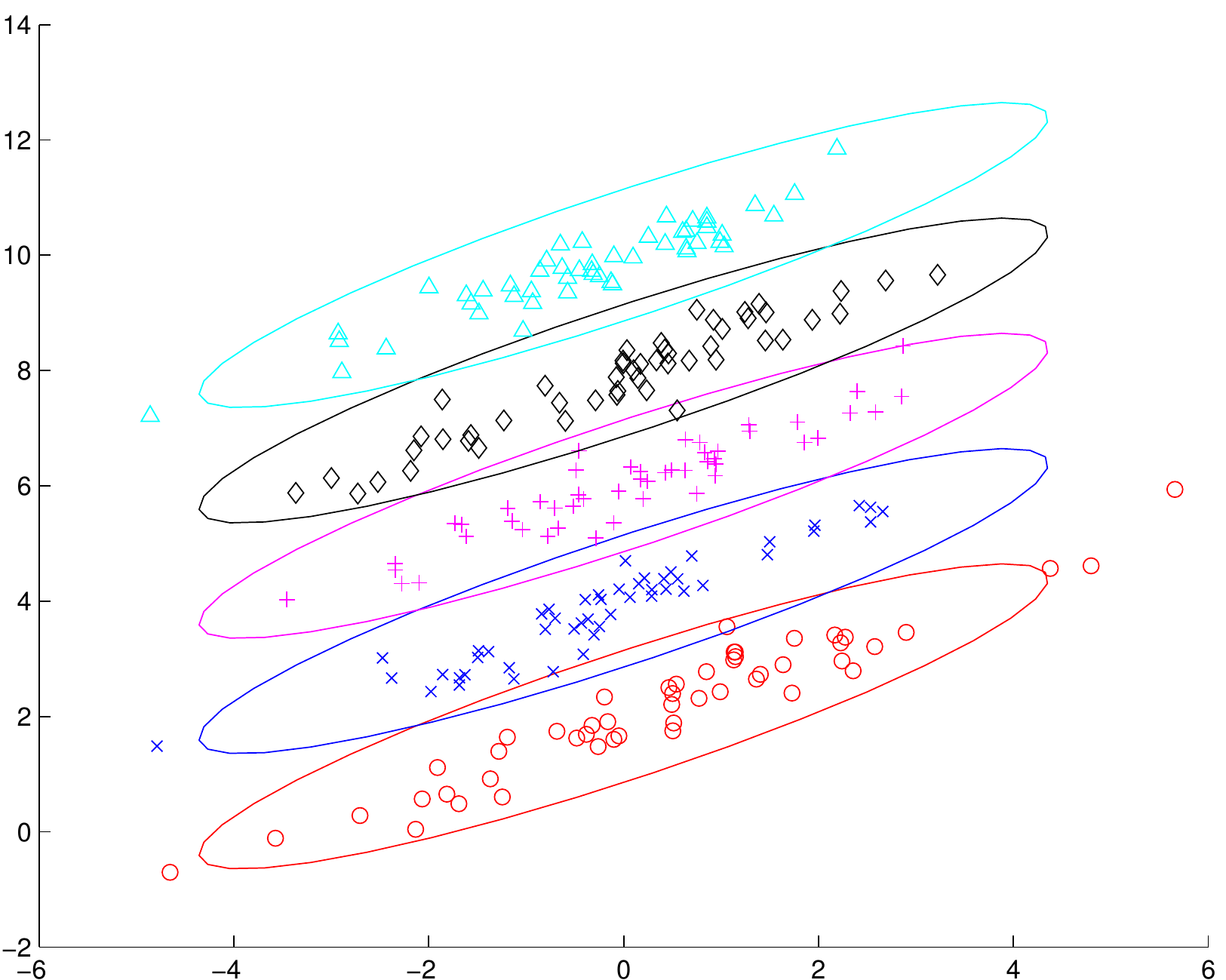}
        \includegraphics[width=0.13\textwidth,clip]{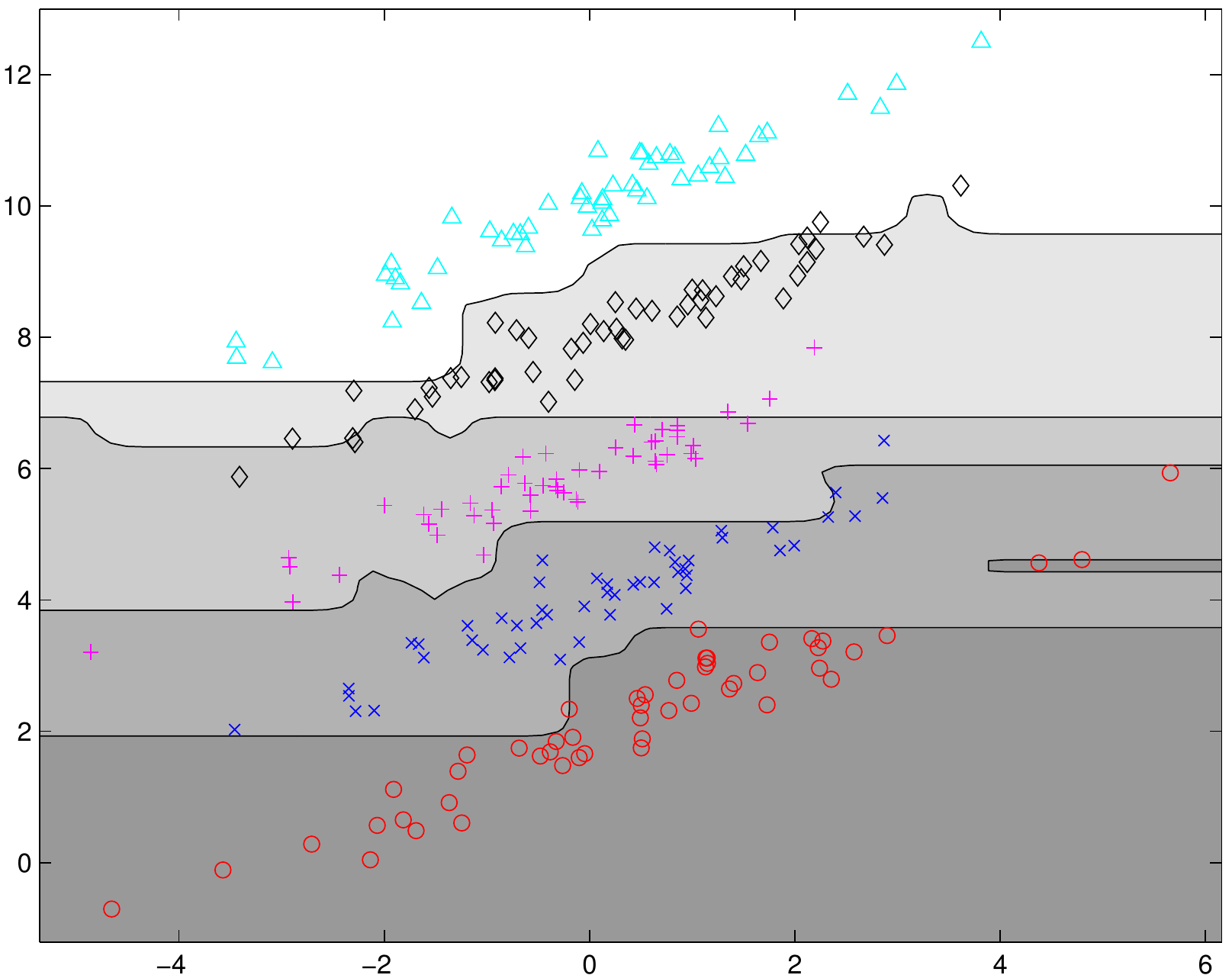}
        \includegraphics[width=0.13\textwidth,clip]{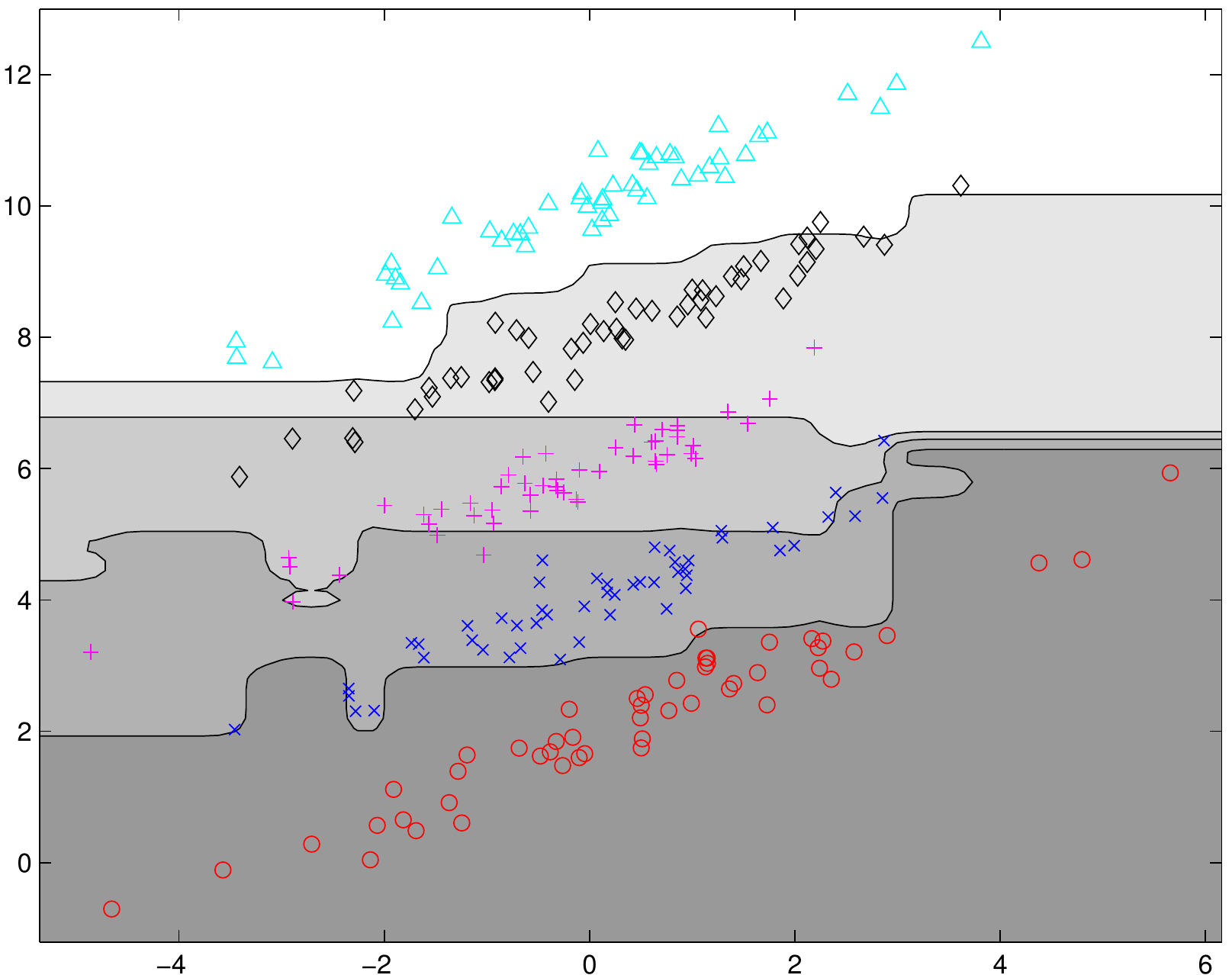}
        \includegraphics[width=0.13\textwidth,clip]{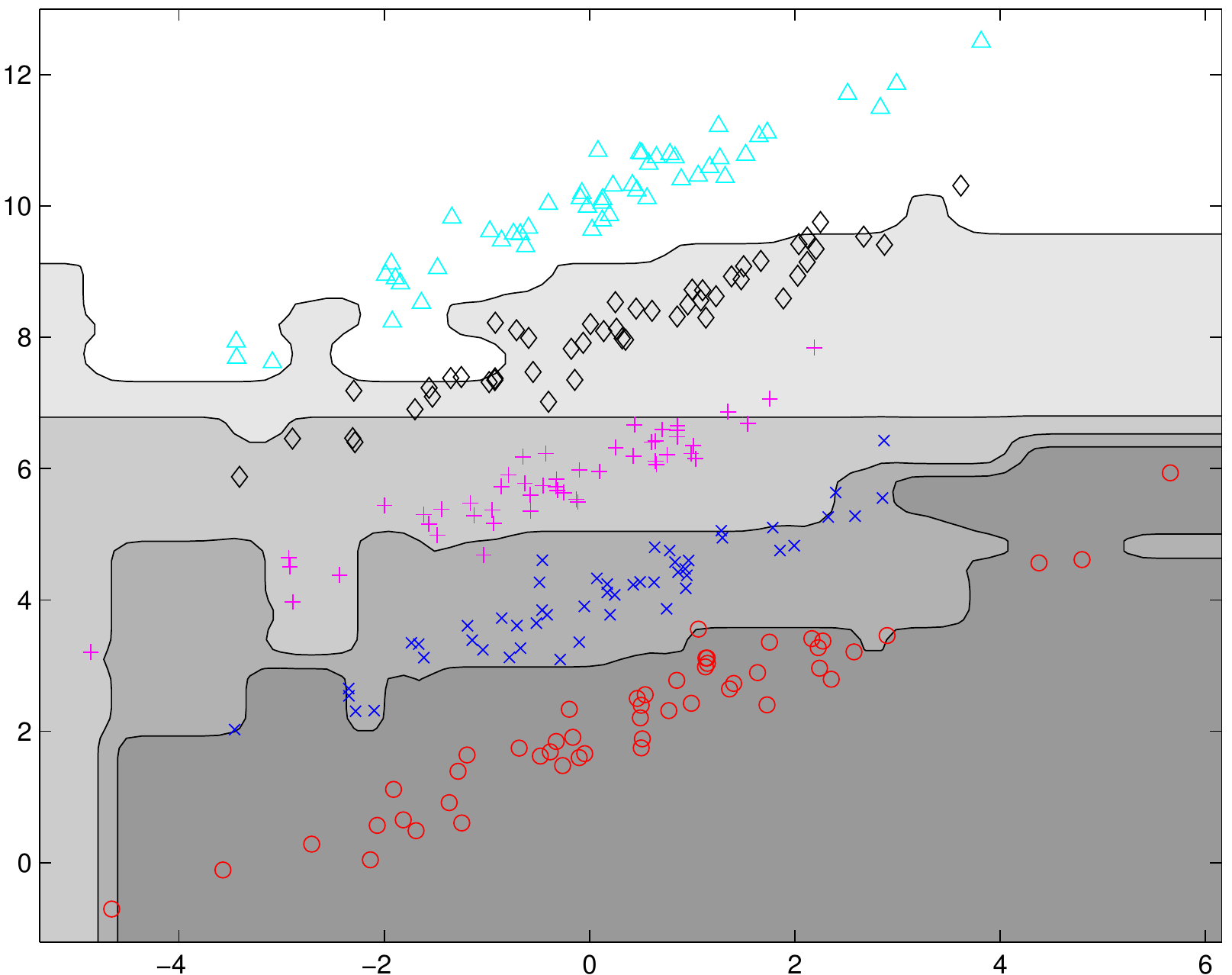}
        \includegraphics[width=0.13\textwidth,clip]{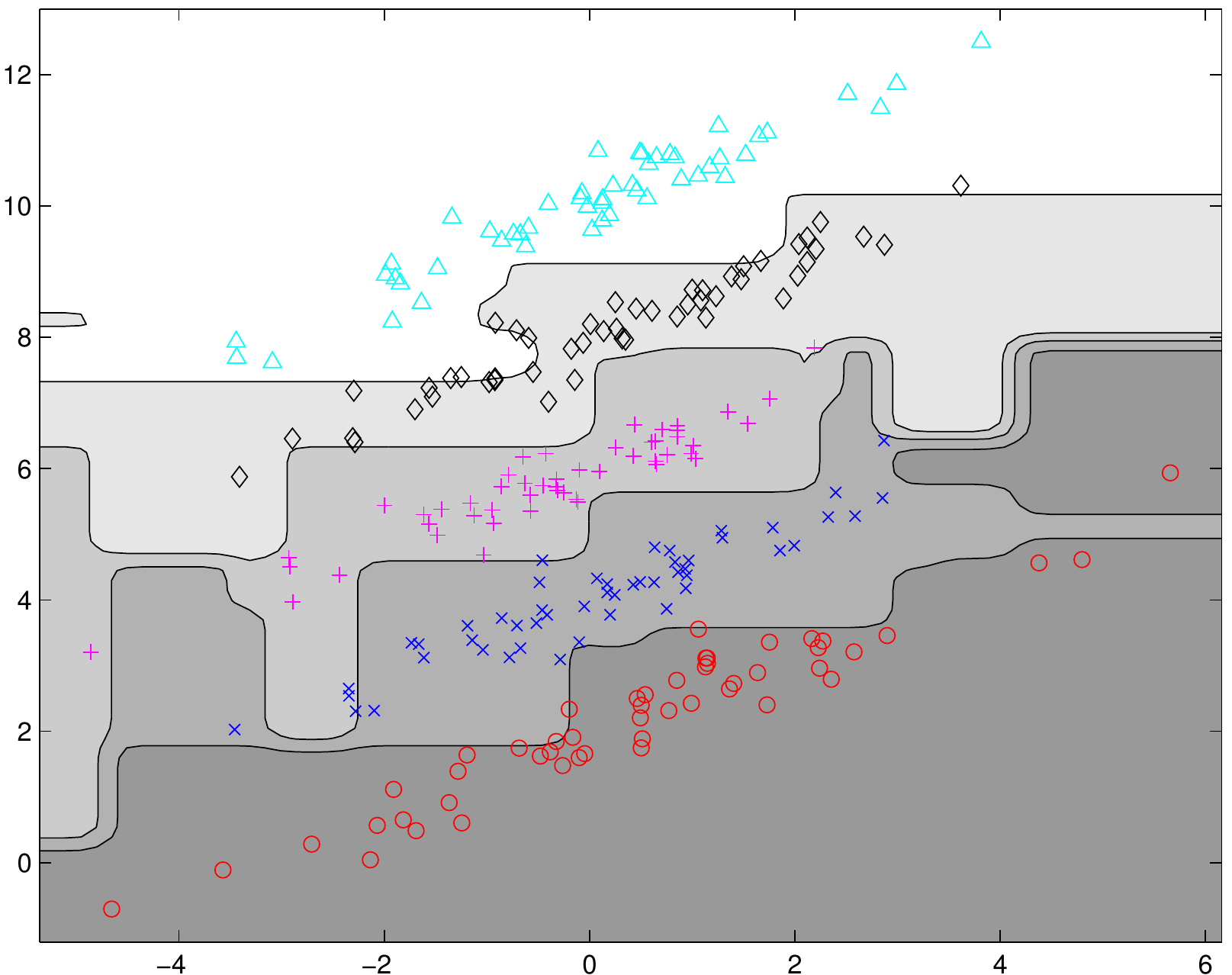}
        \includegraphics[width=0.13\textwidth,clip]{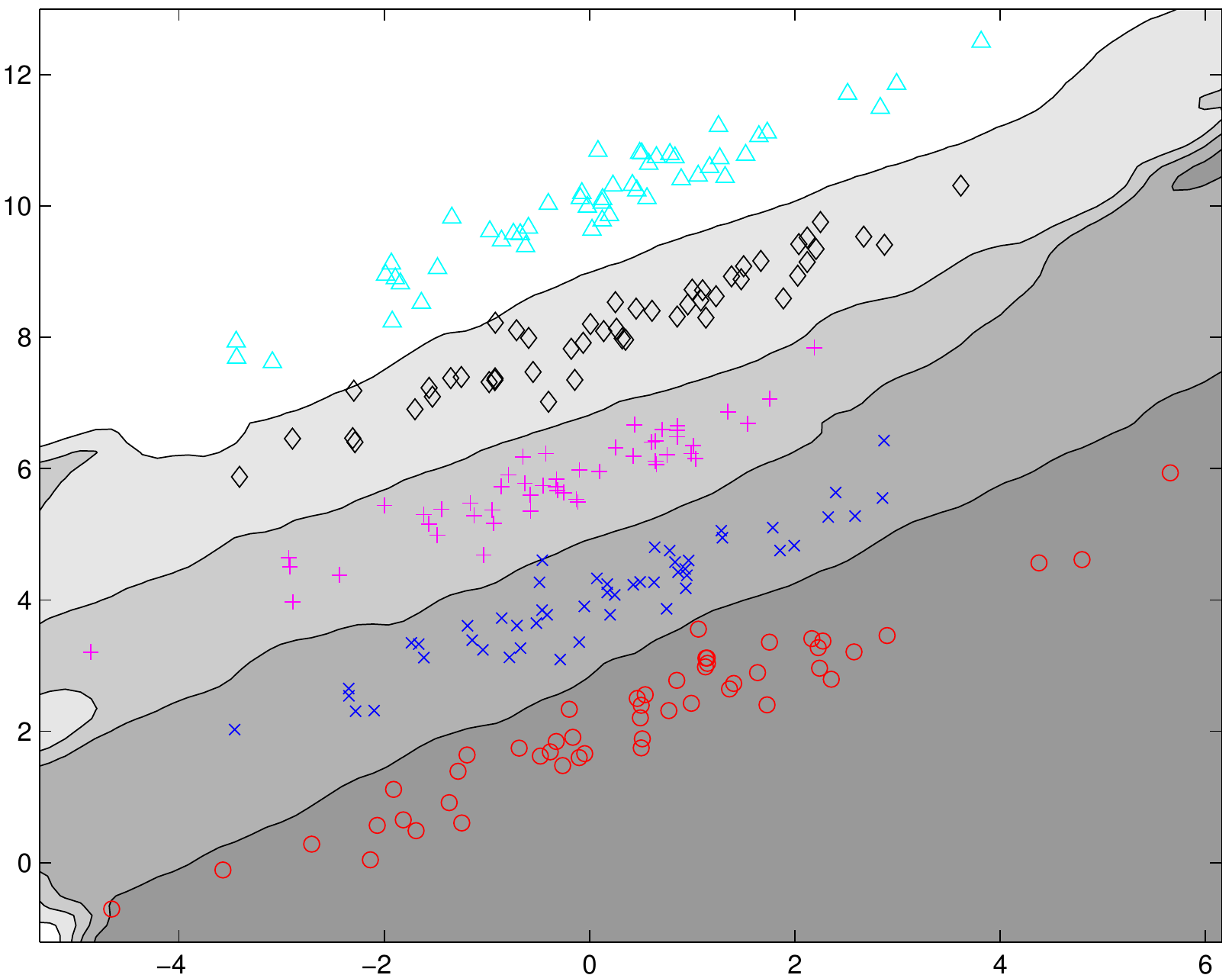}
        \includegraphics[width=0.13\textwidth,clip]{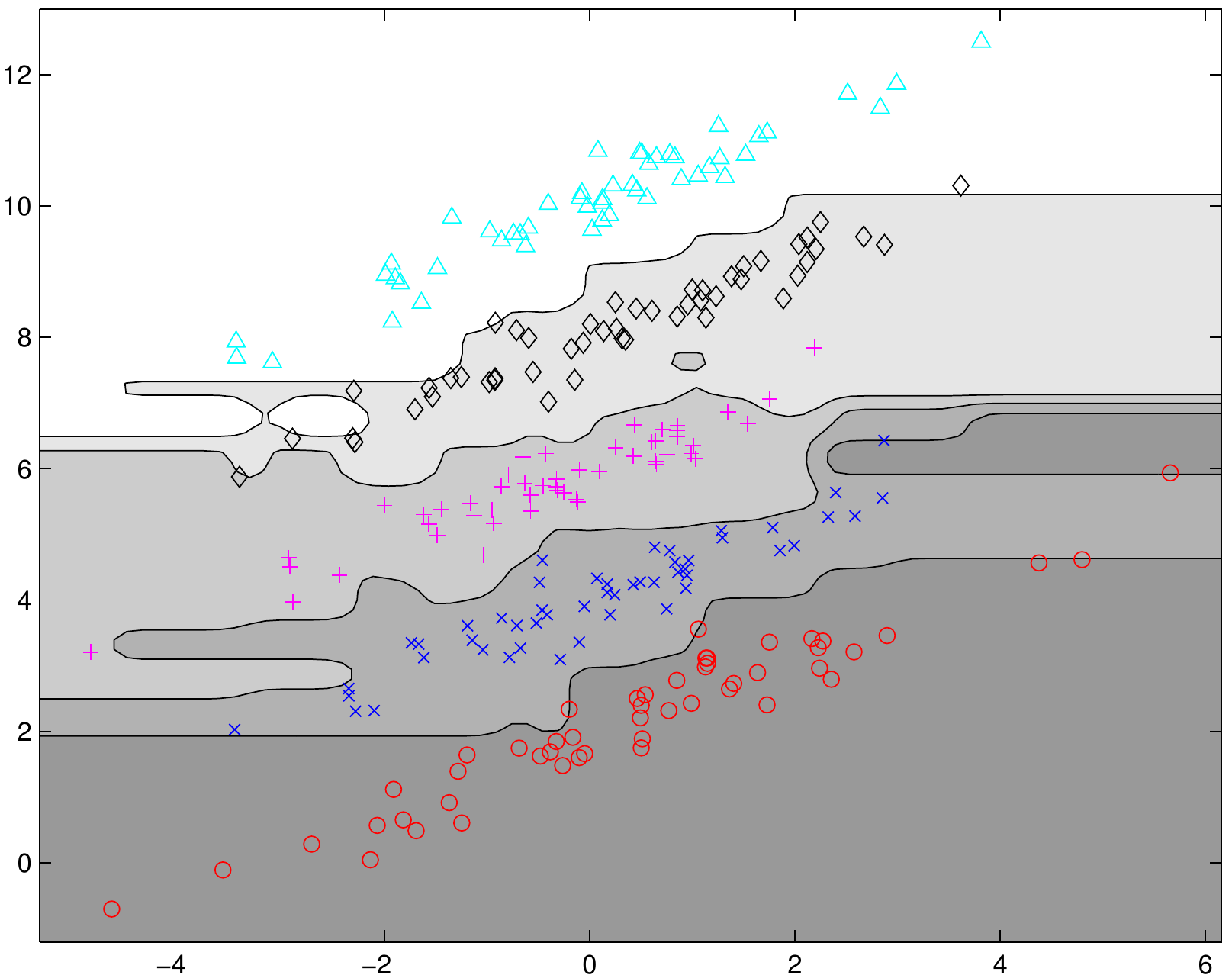}
        \includegraphics[width=0.13\textwidth,clip]{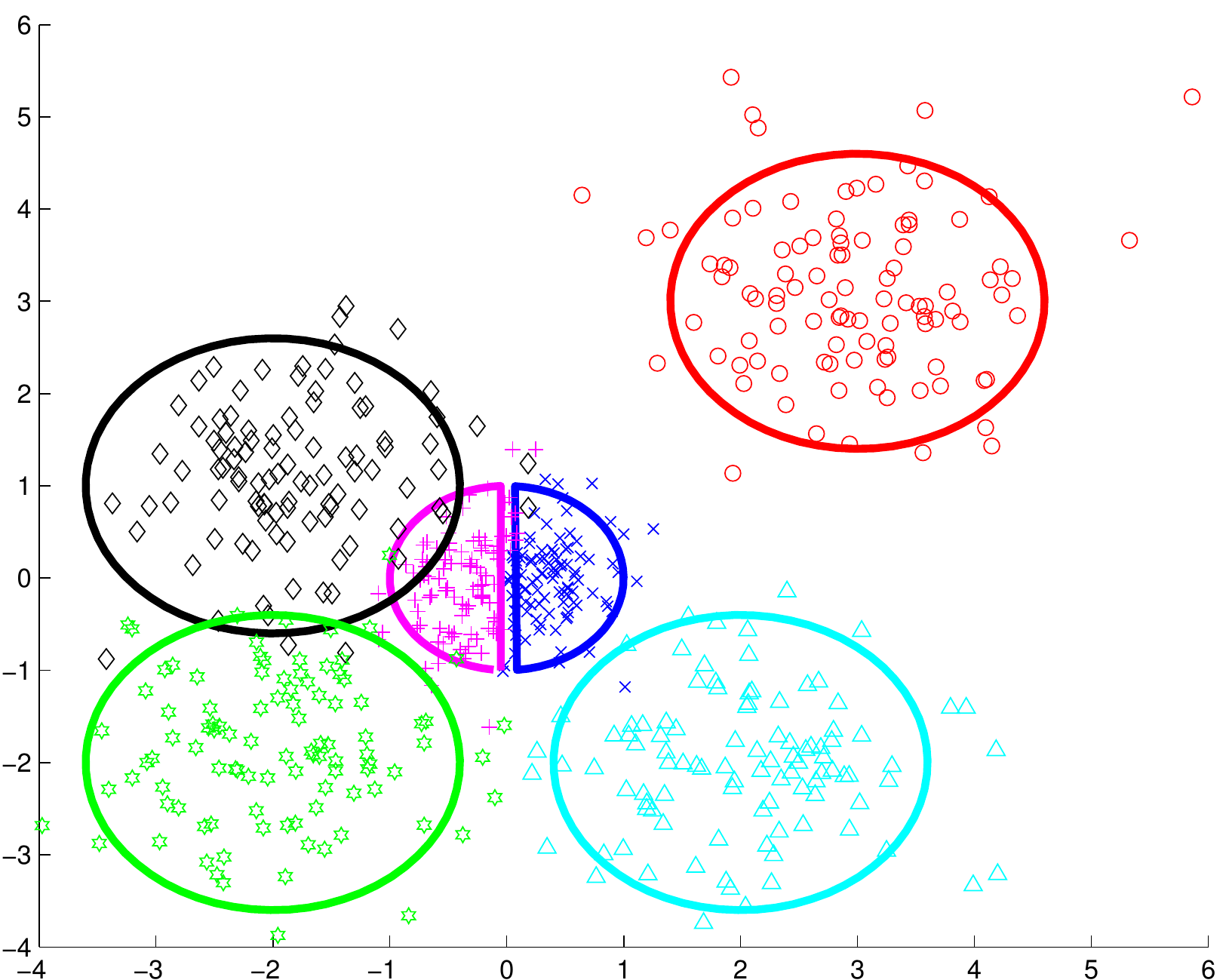}
        \includegraphics[width=0.13\textwidth,clip]{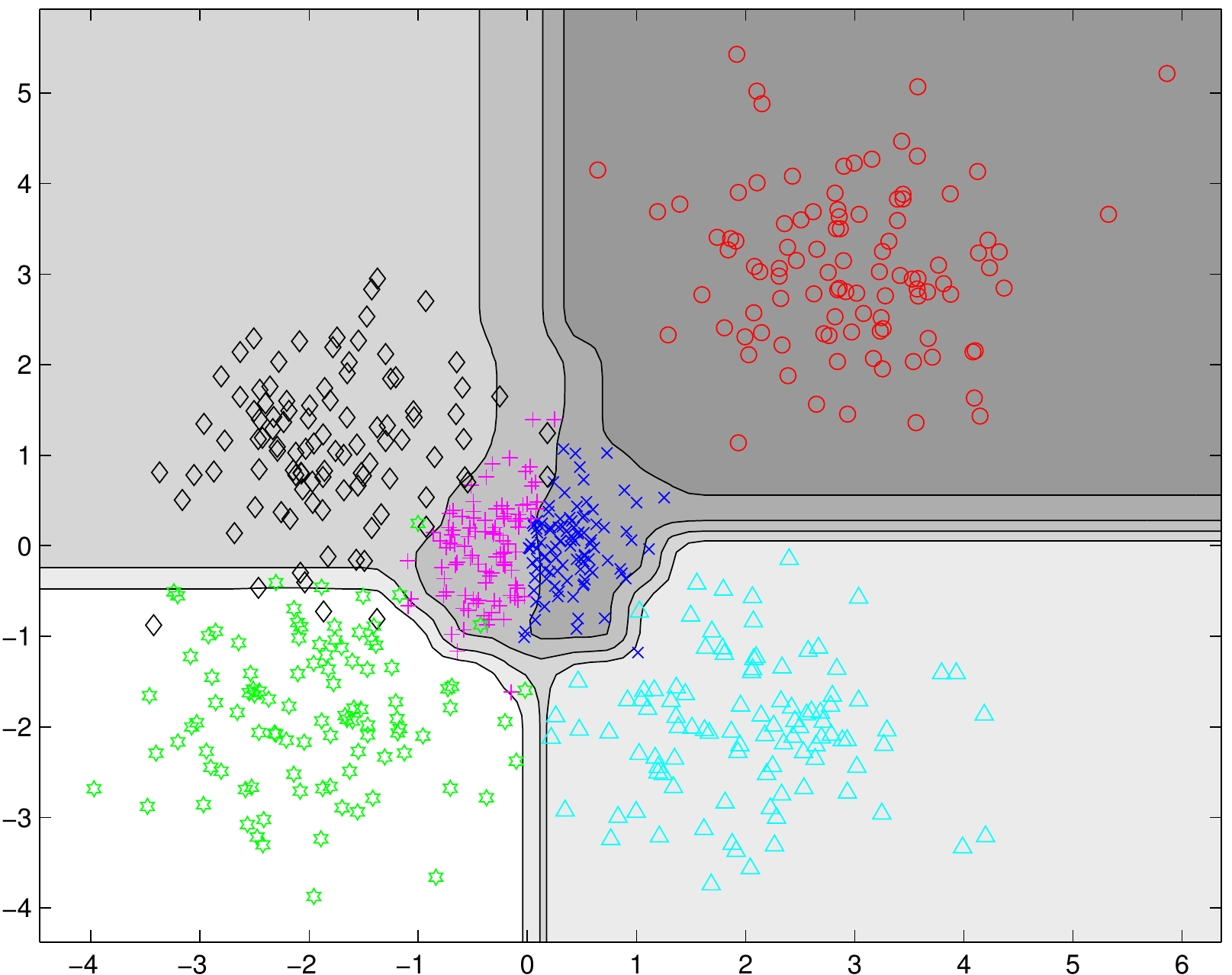}
        \includegraphics[width=0.13\textwidth,clip]{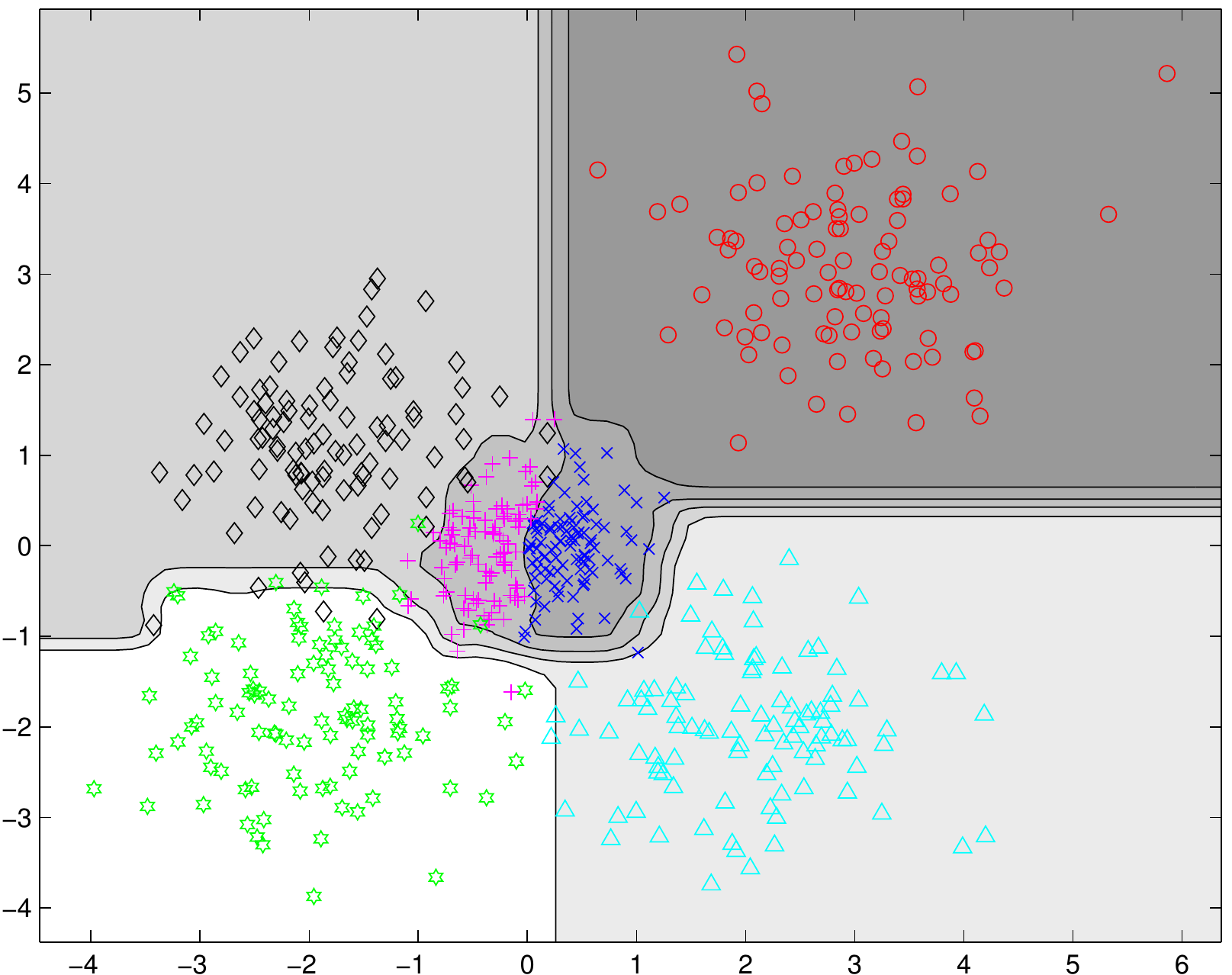}
        \includegraphics[width=0.13\textwidth,clip]{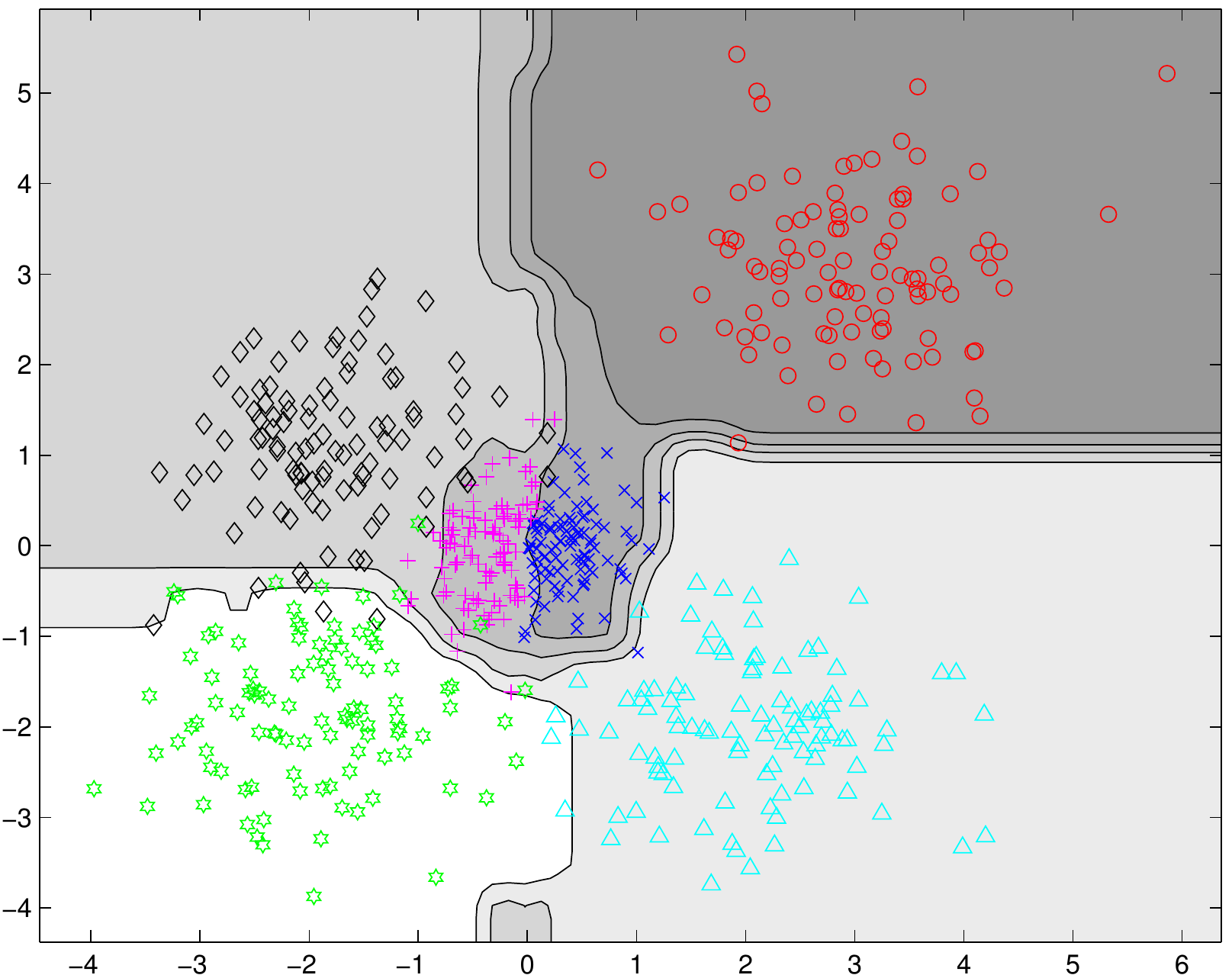}
        \includegraphics[width=0.13\textwidth,clip]{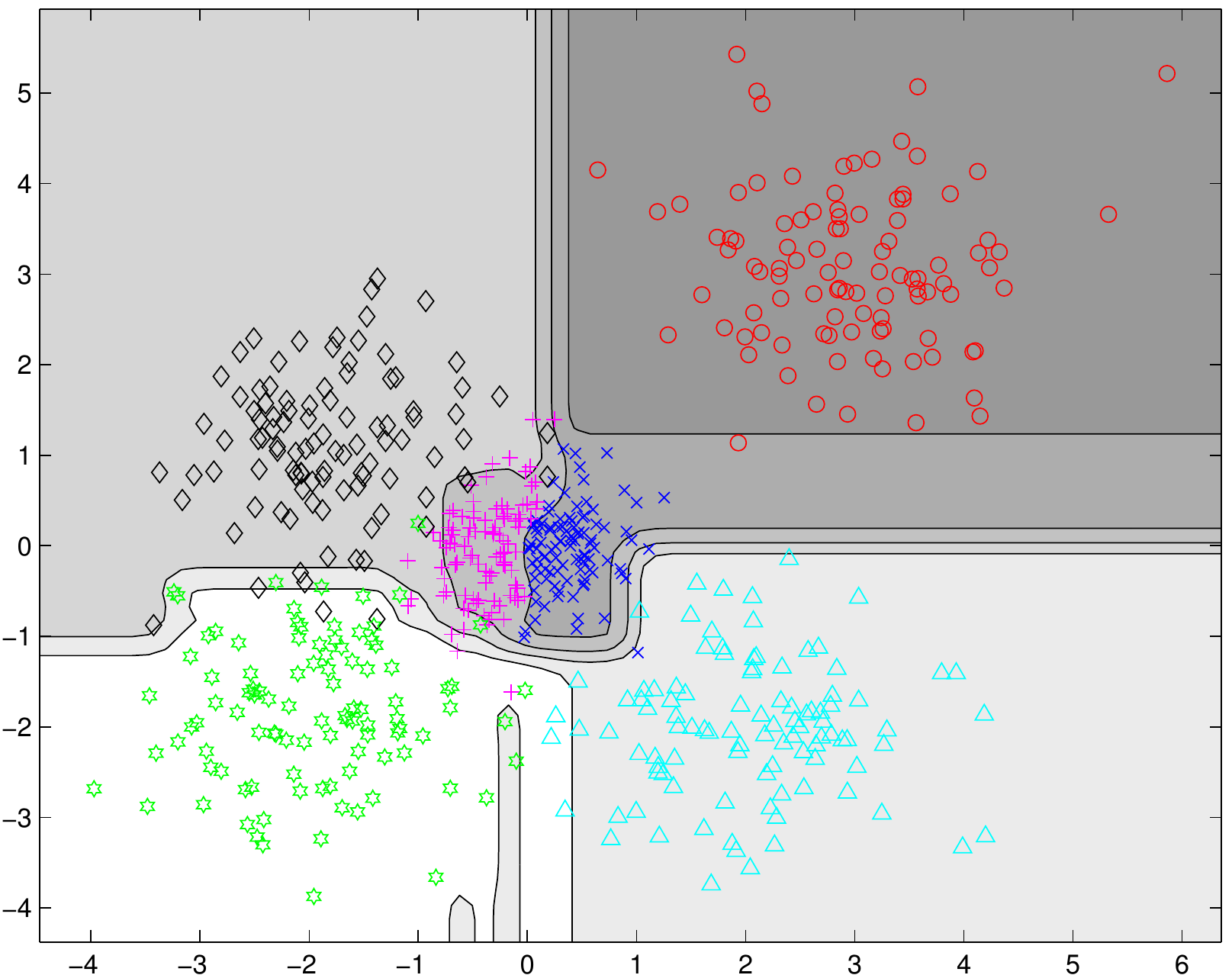}
        \includegraphics[width=0.13\textwidth,clip]{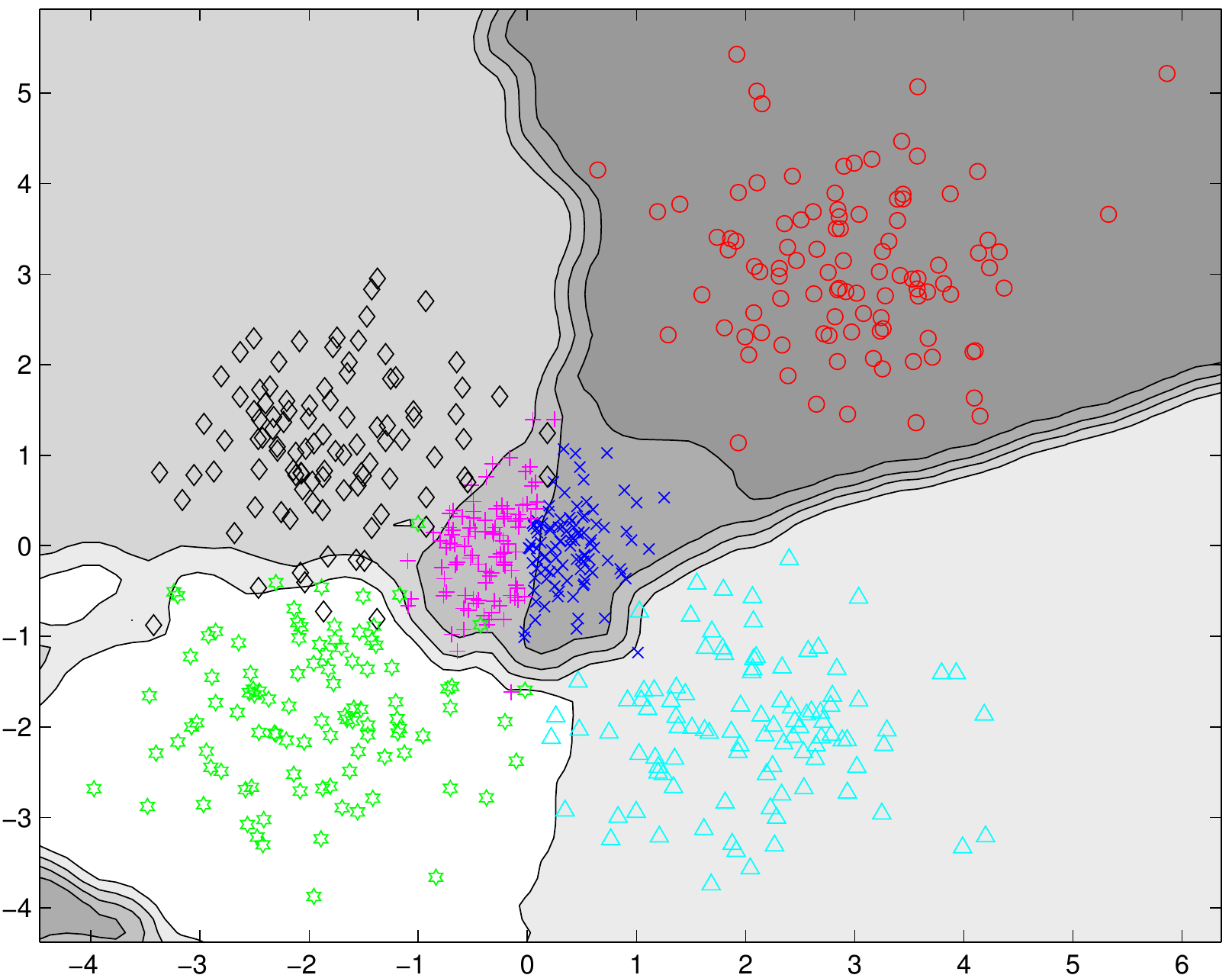}
        \includegraphics[width=0.13\textwidth,clip]{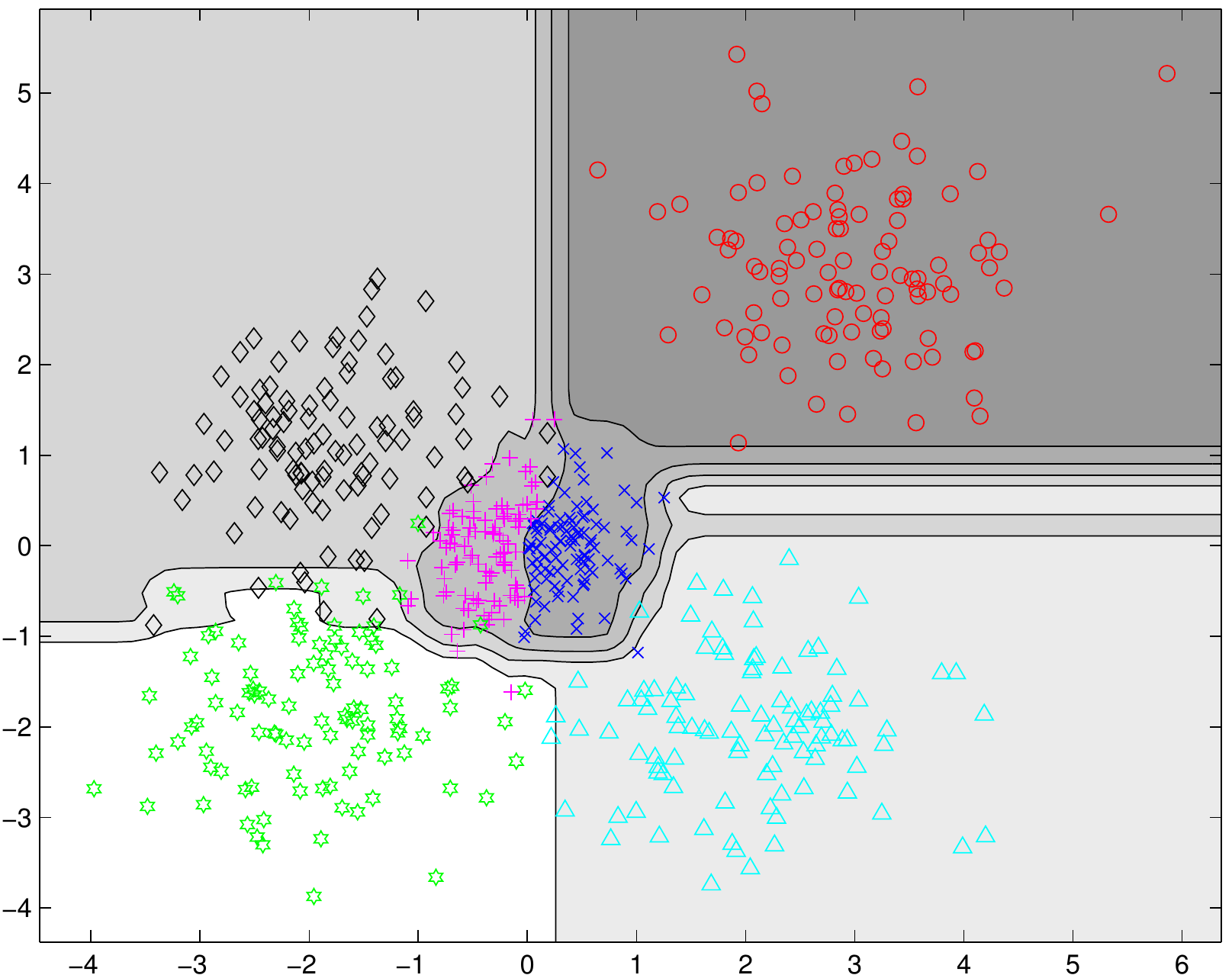}
    \caption{
    Decision boundaries of various
    multi-class boosting algorithms on artificial data sets.
    The data distribution is shown in the first column.
    The number of boosting iterations is set to $500$.
    }
    \label{fig:toy}
  \end{figure*}

\subsection{Toy data}

        We first illustrate the behavior of our algorithms on artificial multi-class
data sets.  We consider the problem of discriminating various object
classes on a $2$D plane.  For this experiment the feature vectors are
the $xy$-coordinates of the $2$D plane.  We train $6$ different
classifiers using AdaBoost.ECC \cite{Guruswami1999Multiclass},
AdaBoost.MH \cite{Schapire1999Improved}, AdaBoost.MO
\cite{Schapire1999Improved}, MultiBoost \cite{Shen2011Direct},
and our proposed
\randomRank and \randomProj.
For MultiBoost, we use hinge loss and
choose the regularization parameter from $\{ 10^{-4},$ $10^{-3},$ $10^{-2} \}$.
For both
\randomRank and \randomProj, we set $n$ to be equal to $500$.
For
\randomProj, we choose the regularization parameter, $\nu$,
from $\{ 10^{-7},$ $10^{-6},$
$10^{-5},$ $10^{-4},$ $10^{-3} \}$.
In this experiment, we set the number of boosting
iterations to $500$.
Fig.~\ref{fig:toy} plots decision boundaries of various methods.
On two dimensional toy data sets, we observe that
decision boundaries of \randomRank\ are very similar to the
true decision boundary.  This is not surprising since all toy
data sets are generated from the multivariate normal distribution.
Hence, \randomRank\ produces very accurate decision boundaries.

\begin{table*}[bt]
  \centering
  \begin{tabular}{l|cccc|cccc}
  \hline
    &\multicolumn{4}{c|}{\randomRank} & \multicolumn{4}{c}{\randomProj}\\
  \cline{2-9}
    Data set & $n=1000$ & $2500$ & $5000$ & $10000$ &  $n=250$ & $500$ & $1000$ & $2000$ \\
  \hline
  \hline
  Synthetic 1 &  $3.4$ ($1.5$) & $2.6$ ($0.9$) & $2.1$ ($0.9$) & $1.8$ ($1.5$) &
                $11.2$ ($1.9$) & $13.0$ ($1.6$) & $11.2$ ($1.3$) & $10.9$ ($2.1$) \\
  Synthetic 2 &  $0.6$ ($0.6$) & $0.6$ ($0.4$) & $0.2$ ($0.4$) & $0.5$ ($1.1$) &
                $7.0$ ($2.3$) & $8.2$ ($1.4$) & $6.6$ ($2.2$) & $7.0$ ($1.3$) \\
  Synthetic 3 & $3.7$ ($1.2$) & $3.3$ ($1.4$) & $4.0$ ($1.3$) & $3.5$ ($1.3$) &
                $6.5$ ($2.0$) & $6.9$ ($2.0$) & $6.4$ ($1.8$) & $7.1$ ($2.9$) \\
  \hline
  \end{tabular}
  \caption{
  Average test errors and standard deviations
  (shown in $\%$) for different values of $n$.
  All experiments are repeated $5$ times
  }
  \label{tab:D}
\end{table*}

\paragraph{Size of the projected space}
We use three previous artificial data sets and vary the size of the projected space, $n$.
Each data set is randomly split into two groups: $75\%$ for training and the rest for evaluation.
We set the maximum number of boosting iterations to $500$.
We vary $n$ from $1000$ to $10,000$ for \randomRank and $250$ to $2000$ for \randomProj.
For \randomRank, the larger the parameter $n$, the more features that the algorithm can choose during training.
From Theorem~\ref{thm:multiclass}, as long as 
$D$ is approximately larger than $\log(mk)$, 
the margin is preserved with high probability for \randomProj.
Table~\ref{tab:D} reports final classification errors of various $n$.
For \randomRank, we observe a slight increase in generalization performance when $n$ increases.
For \randomProj, as long as $n$ is sufficiently large ($> 250$ in this experiment), the final performance is almost not affected by the value of $n$.

\begin{table*}[t]
  \centering
  {
  \begin{tabular}{l|c|c|c|c|c|c}
  \hline
   &  \multicolumn{2}{|c|}{Exponential loss (TC)}  &
      \multicolumn{2}{|c|}{Logistic loss (TC)} &
      \multicolumn{2}{|c}{Exponential loss (SW)} \\
  \cline{2-7}
   Data set & Test error & CPU time & Test error & CPU time & Test error & CPU time  \\
  \hline
  \hline
australian & $\mathbf{14.9}$ ($\mathbf{2.5}$) & $11.7$ & $17.4$ ($2.6$) & $6.3$ &
    $15.8$ ($2.1$) & $\mathbf{0.03}$ \\
heart & $\mathbf{19.9}$ ($\mathbf{4.5}$) & $6.5$ & $22.7$ ($3.8$) & $1.9$ &
    $21.3$ ($4.2$) & $\mathbf{0.02}$ \\
wine & $3.0$ ($2.4$) & $13.7$ & $\mathbf{2.5}$ ($\mathbf{2.3}$) & $1.4$ &
    $2.5$ ($2.7$) & $\mathbf{0.03}$ \\
glass & $34.9$ ($6.1$) & $10.6$ & $\mathbf{30.4}$ ($\mathbf{5.2}$) & $4.8$ &
    $31.3$ ($6.2$) & $\mathbf{0.03}$ \\
segment & $\mathbf{3.0}$ ($\mathbf{0.6}$) & $145$ & $3.1$ ($0.7$) &
    $45.2$ & $3.3$ ($0.6$) & $\mathbf{0.09}$ \\
  \hline
  \end{tabular}
  }
  \caption{Average test errors (in \%) and CPU time (seconds) (time
  taken to solve the optimization problem in step \cfour\, Algorithm~\ref{ALG:alg1}).
  TC: totally-corrective \randomRank and
  SW: stage-wise \randomRank
  }
  \label{tab:randomRankTC}
\end{table*}

\subsection{Totally-corrective \randomRank and stage-wise \randomRank}
\revised{
In this experiment, we compare the performance of
totally-corrective \randomRank with stage-wise \randomRank.
We use UCI machine learning repository data sets and randomly split the data sets
into two groups: $75\%$ of samples for training and the rest for evaluation.
We set the maximum number of boosting iterations to $500$.
We conduct an experiment on two convex
losses: the exponential loss and the logistic loss.
For totally-corrective boosting, we solve the optimization problem,
step \cfour\ in Algorithm~\ref{ALG:alg1}, using L-BFGS-B.
For L-BFGS-B parameters, we set the maximum number of iterations to $100$,
the accuracy of the line search to $10^{-5}$,
the convergence parameter to terminate the program to $10^7 \cdot \epsilon$
(where $\epsilon$ is a machine precision) and
the number of corrections to approximate the inverse hessian matrix to $5$.
We use the same L-BFGS-B parameters for all experiments.
The regularization parameter in \eqref{EQ:App2}, $\nu$, 
is determined by $5$-fold cross validation.
We choose the best $\nu$ from $\{ 10^{-7}$, $10^{-6}$, 
$10^{-5}$, $10^{-4}$, $10^{-3}$, $10^{-2} \}$.
For stage-wise \randomRank, we set $n$ to be $100$ times
the dimension size of the original data.
All experiments are repeated $10$ times and the average
and standard deviation of test errors are reported in
Table~\ref{tab:randomRankTC}.
We observe that the performance of both convex losses are comparable
and stage-wise \randomRank produces
comparable test accuracy to totally-corrective \randomRank.
However, stage-wise \randomRank has a much lower CPU time.
Since both totally-corrective and stage-wise \randomRank are comparable,
we use stage-wise \randomRank in the rest of our experiments.
}

\subsection{UCI data sets}
\begin{table*}[t]
  \centering
  {
  \begin{tabular}{l|ccc|ccc|ccc}
  \hline
   &  \multicolumn{3}{|c|}{AdaBoost}  &  \multicolumn{3}{|c|}{\randomRank} & \multicolumn{3}{|c}{\randomProj} \\
   Data set & Test $50$ & Test $100$ & Test $1000$ & Test $50$ & Test $100$ & Test $1000$
   & Test $50$ & Test $100$ & Test $1000$ \\
  \hline
  \hline
australian & $14.8$ ($2.9$) & $14.8$ ($2.1$) & $16.6$ ($2.1$) & $15.3$ ($2.8$) & $15.7$ ($2.2$) & $16.9$ ($2.6$) & $\mathbf{14.2}$ ($\mathbf{2.4}$) & $\mathbf{14.2}$ ($\mathbf{2.4}$) & $\mathbf{14.2}$ ($\mathbf{2.4}$) \\
b-cancer & $4.3$ ($1.2$) & $4.4$ ($1.1$) & $4.6$ ($1.3$) & $4.6$ ($1.4$) & $4.2$ ($1.3$) & $4.3$ ($1.4$) & $\mathbf{3.9}$ ($\mathbf{1.0}$) & $\mathbf{4.0}$ ($\mathbf{1.0}$) & $\mathbf{4.1}$ ($\mathbf{1.0}$) \\
c-cancer & $20.0$ ($7.7$) & $18.7$ ($8.8$) & $16.0$ ($9.0$) & $\mathbf{16.7}$ ($\mathbf{7.9}$) & $\mathbf{15.3}$ ($\mathbf{8.3}$) & $\mathbf{16.0}$ ($\mathbf{7.8}$) & $23.3$ ($11.9$) & $23.3$ ($11.9$) & $23.3$ ($11.9$) \\
diabetes & $26.7$ ($2.1$) & $26.3$ ($3.0$) & $25.7$ ($2.9$) & $25.7$ ($1.5$) & $\mathbf{25.5}$ ($\mathbf{1.3}$) & $26.4$ ($2.6$) & $\mathbf{25.5}$ ($\mathbf{2.2}$) & $25.7$ ($2.1$) & $\mathbf{25.7}$ ($\mathbf{2.1}$) \\
german & $\mathbf{24.2}$ ($\mathbf{2.3}$) & $24.4$ ($2.3$) & $25.8$ ($3.0$) & $24.6$ ($2.5$) & $\mathbf{24.2}$ ($\mathbf{2.9}$) & $24.9$ ($3.1$) & $24.5$ ($1.6$) & $24.5$ ($1.6$) & $\mathbf{24.5}$ ($\mathbf{1.6}$) \\
heart & $\mathbf{16.7}$ ($\mathbf{3.1}$) & $17.6$ ($3.4$) & $20.9$ ($3.1$) & $16.9$ ($3.9$) & $\mathbf{16.7}$ ($\mathbf{4.0}$) & $\mathbf{17.6}$ ($\mathbf{3.5}$) & $19.6$ ($2.2$) & $19.9$ ($1.9$) & $19.9$ ($1.9$) \\
ionosphere & $11.7$ ($2.2$) & $11.6$ ($2.4$) & $10.0$ ($2.8$) & $\mathbf{9.4}$ ($\mathbf{3.1}$) & $\mathbf{7.5}$ ($\mathbf{2.9}$) & $\mathbf{7.4}$ ($\mathbf{3.6}$) & $12.2$ ($3.2$) & $11.9$ ($3.3$) & $12.0$ ($3.3$) \\
liver & $\mathbf{27.9}$ ($\mathbf{6.3}$) & $\mathbf{28.0}$ ($\mathbf{4.6}$) & $\mathbf{28.4}$ ($\mathbf{3.7}$) & $30.6$ ($4.7$) & $30.5$ ($4.6$) & $30.6$ ($3.6$) & $29.9$ ($5.6$) & $30.0$ ($5.4$) & $30.0$ ($5.4$) \\
mushrooms & $0.2$ ($0.1$) & $\mathbf{0.0}$ ($\mathbf{0.0}$) & $\mathbf{0.0}$ ($\mathbf{0.0}$) & $0.1$ ($0.1$) & $\mathbf{0.0}$ ($\mathbf{0.0}$) & $\mathbf{0.0}$ ($\mathbf{0.0}$) & $\mathbf{0.0}$ ($\mathbf{0.0}$) & $\mathbf{0.0}$ ($\mathbf{0.0}$) & $\mathbf{0.0} $ ($\mathbf{0.0}$) \\
sonar & $\mathbf{17.9}$ ($\mathbf{4.3}$) & $\mathbf{16.9}$ ($\mathbf{3.8}$) & $\mathbf{16.5}$ ($\mathbf{2.9}$) & $22.5$ ($5.4$) & $19.8$ ($6.1$) & $18.8$ ($4.7$) & $21.7$ ($4.9$) & $18.5$ ($4.0$) & $19.0$ ($4.5$) \\
splice & $8.7$ ($0.8$) & $8.4$ ($1.1$) & $8.7$ ($1.3$) & $16.5$ ($1.6$) & $15.1$ ($1.6$) & $11.3$ ($1.3$) & $\mathbf{8.3}$ ($\mathbf{1.2}$) & $\mathbf{8.2}$ ($\mathbf{1.2}$) & $\mathbf{8.2}$ ($\mathbf{1.2}$) \\
  \hline
  \end{tabular}
  }
  \caption{Average test errors and standard deviations
  (in $\%$) of the proposed algorithms on two-class UCI data sets.
  All experiments are repeated $10$ times.
  Test errors at $50$, $100$ and $1000$ boosting iterations are reported
  }
  \label{tab:UCIbinary}
\end{table*}

\begin{table*}[t]
  \centering
  {
  \begin{tabular}{l|cccccc}
  \hline
  Data set &  AdaBoost.ECC & AdaBoost.MH & AdaBoost.MO & MultiBoost & \randomRank & \randomProj  \\
  \hline
  \hline
dna ($3$ classes) & $6.8$ ($0.9$) & $\mathbf{5.6}$ ($\mathbf{1.2}$) & $6.9$ ($1.2$) & $7.0$ ($0.9$) & $6.7$ ($0.9$) & $6.7$ ($0.9$)  \\
svmguide2 ($3$ classes) & $23.2$ ($3.7$) & $21.7$ ($3.3$) & $22.9$ ($4.3$) & $22.1$ ($4.2$) & $\mathbf{19.8}$ ($\mathbf{3.0}$) & $21.1$ ($3.6$) \\
wine ($3$ classes) & $3.9$ ($3.0$) & $4.3$ ($3.8$) & $3.6$ ($3.7$) & $4.3$ ($3.5$) & $3.2$ ($2.9$) & $\mathbf{3.0}$ ($\mathbf{3.0}$)  \\
vehicle ($4$ classes) & $21.0$ ($3.6$) & $21.6$ ($3.4$) & $21.3$ ($3.0$) & $21.8$ ($3.0$) & $\mathbf{20.0}$ ($\mathbf{2.3}$) & $22.1$ ($2.3$) \\
glass ($6$ classes) & $23.0$ ($3.8$) & $27.0$ ($3.6$) & $26.2$ ($6.8$) & $26.2$ ($5.5$) & $26.8$ ($4.5$) & $\mathbf{22.5}$ ($\mathbf{4.2}$)  \\
satimage ($6$ classes) & $11.5$ ($0.7$) & $11.1$ ($1.1$) & $10.7$ ($1.0$) & $11.6$ ($0.9$) & $\mathbf{10.2}$ ($\mathbf{0.5}$) & $13.1$ ($0.8$) \\
svmguide4 ($6$ classes) & $\mathbf{15.9}$ ($\mathbf{2.7}$) & $17.5$ ($2.5$) & $17.9$ ($2.3$) & $19.0$ ($3.5$) & $17.6$ ($2.8$) & $17.4$ ($2.1$)  \\
segment ($7$ classes) & $2.1$ ($0.5$) & $3.0$ ($0.5$) & $2.3$ ($0.5$) & $2.4$ ($0.7$) & $3.2$ ($0.8$) & $\mathbf{2.1}$ ($\mathbf{0.3}$)  \\
usps ($10$ classes) & $9.2$ ($2.1$) & $9.2$ ($1.7$) & $\mathbf{8.8}$ ($\mathbf{2.5}$) & $10.0$ ($1.8$) & $8.8$ ($2.6$) & $9.1$ ($2.7$)  \\
pendigits ($10$ classes) & $5.2$ ($0.8$) & $5.8$ ($0.9$) & $6.3$ ($1.4$) & $7.0$ ($1.4$) & $\mathbf{2.8}$ ($\mathbf{0.9}$) & $5.2$ ($0.9$)  \\
vowel ($11$ classes) & $8.7$ ($2.5$) & $11.2$ ($2.3$) & $12.1$ ($3.0$) & $9.3$ ($2.8$) & $\mathbf{3.1}$ ($\mathbf{1.3}$) & $8.1$ ($2.2$) \\
  \hline
  \end{tabular}
  }
  \caption{Average test errors (in $\%$) of different algorithms on multi-class UCI data sets.
  All experiments are repeated $10$ times and the number of boosting iterations
  is set to $1000$
  }
  \label{tab:UCImulti}
\end{table*}

\begin{figure*}[t]
    \begin{center}
        \includegraphics[width=0.19\textwidth,clip]{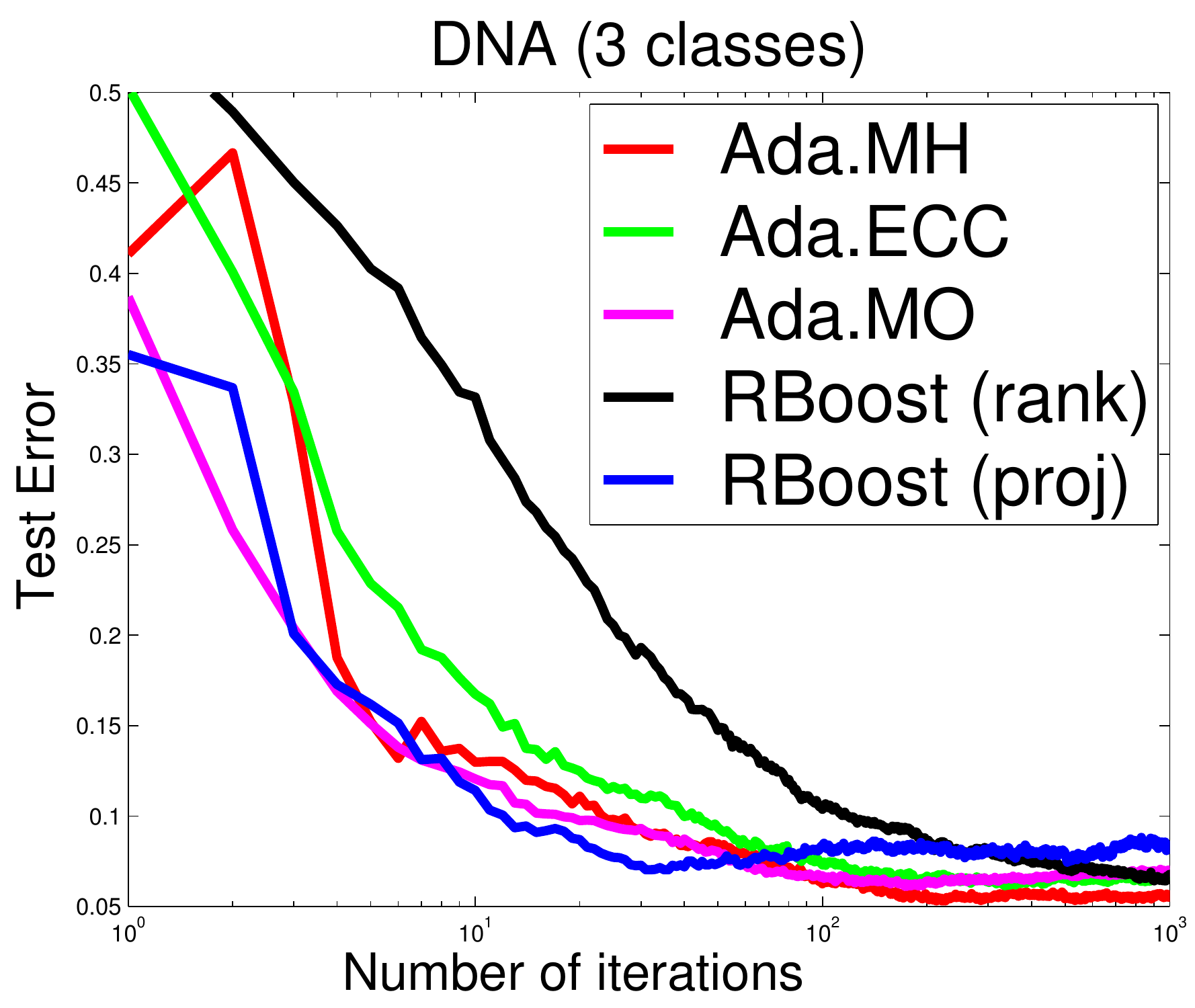}
        \includegraphics[width=0.19\textwidth,clip]{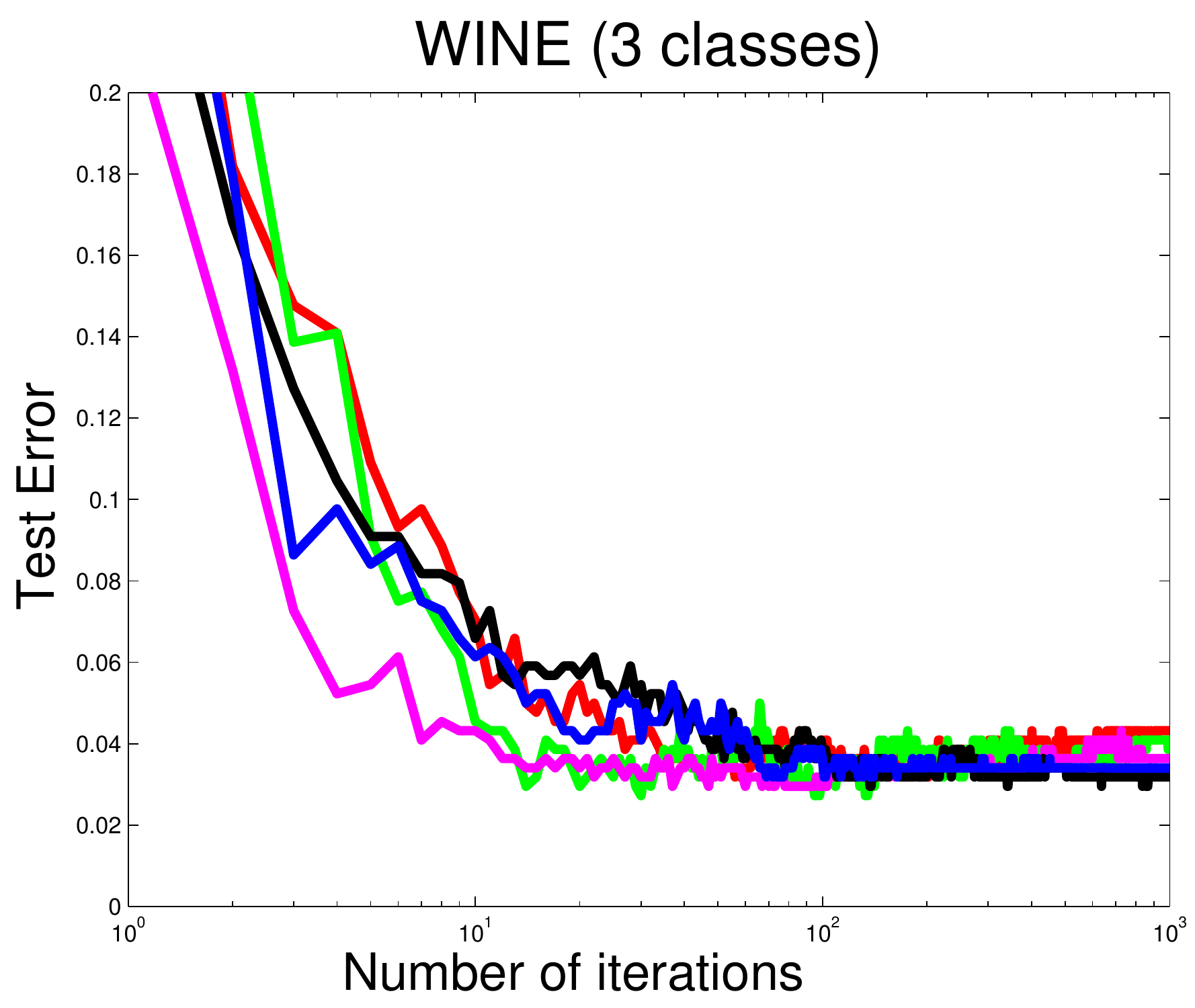}
        \includegraphics[width=0.19\textwidth,clip]{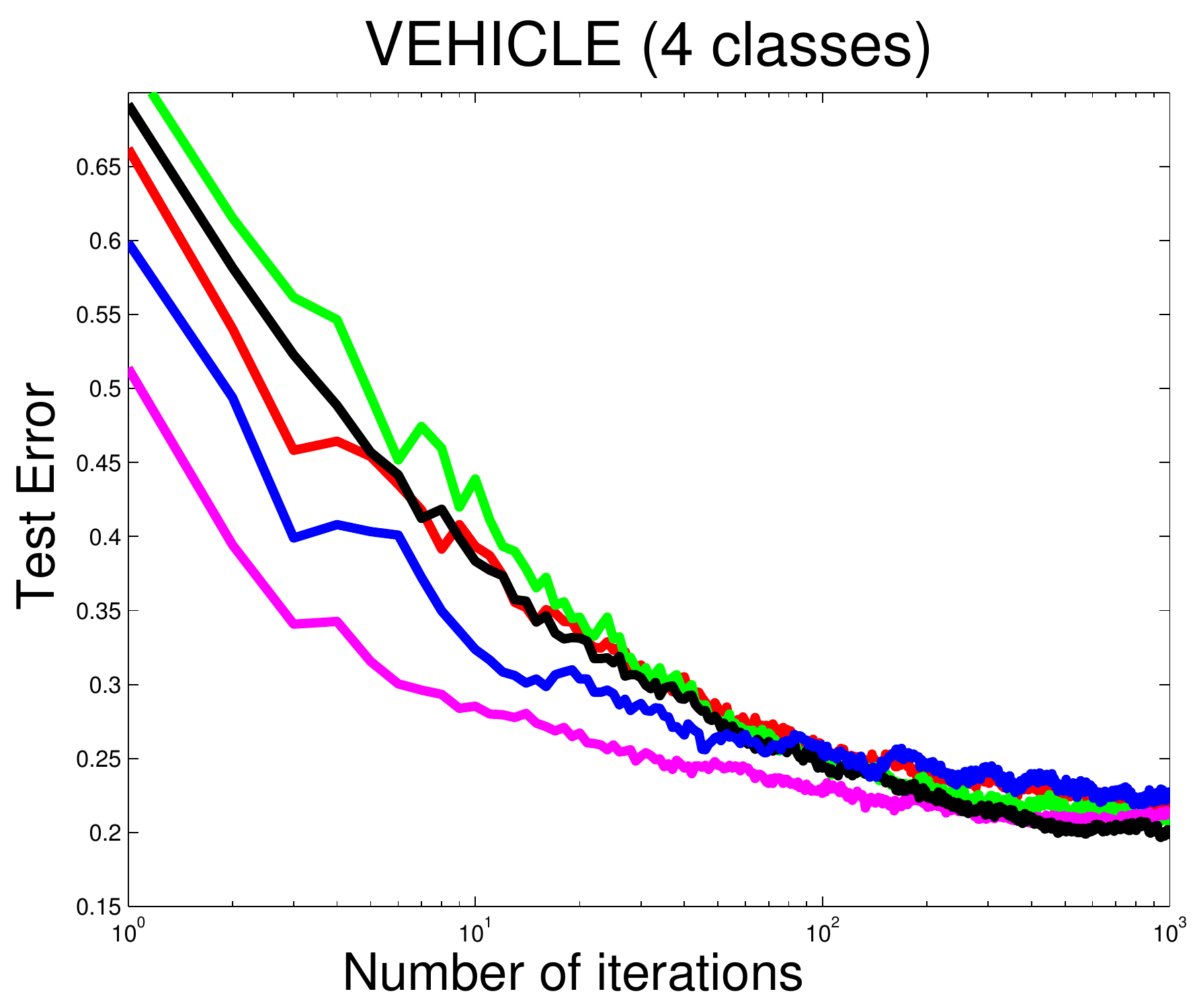}
        \includegraphics[width=0.19\textwidth,clip]{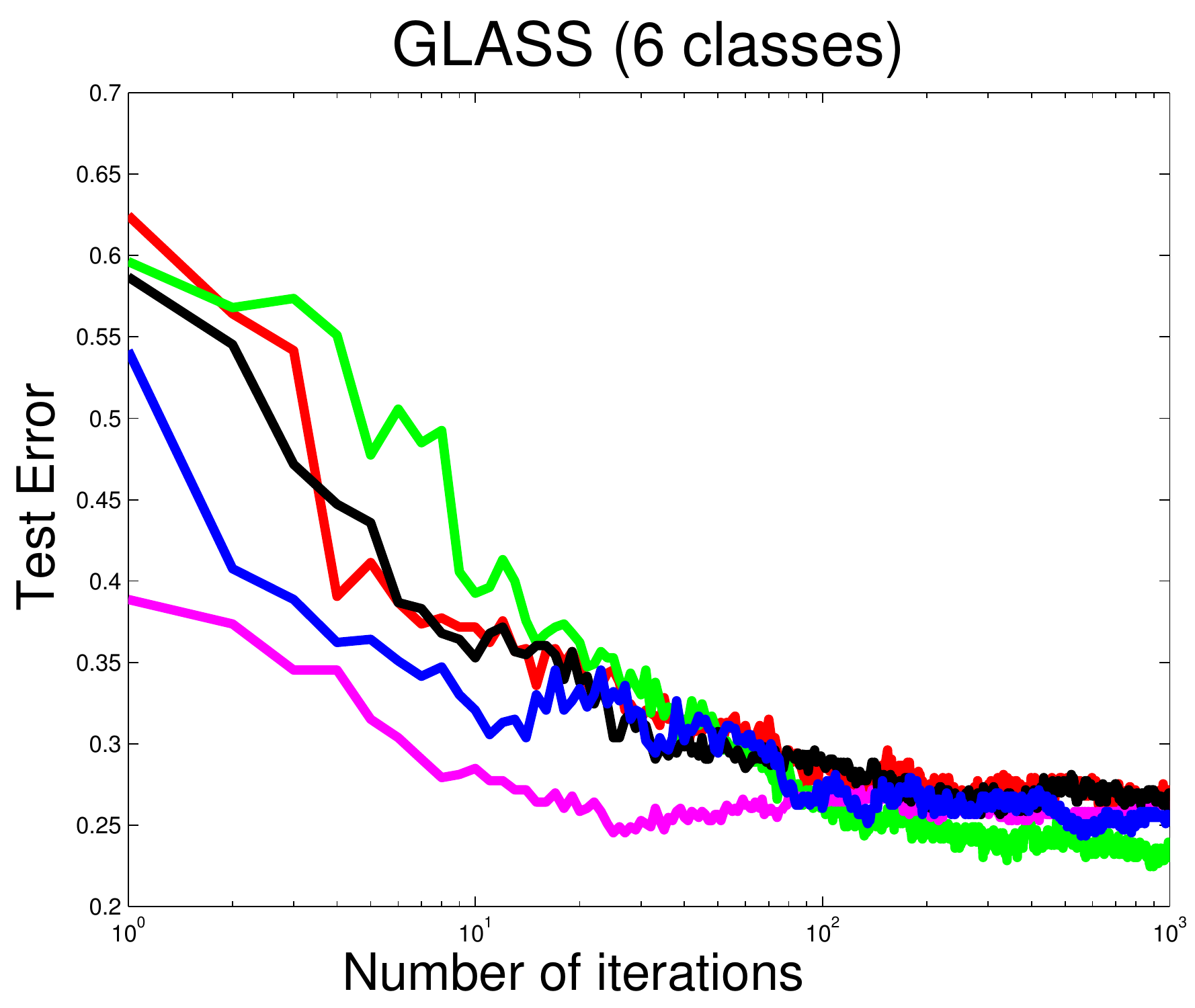}
        \includegraphics[width=0.19\textwidth,clip]{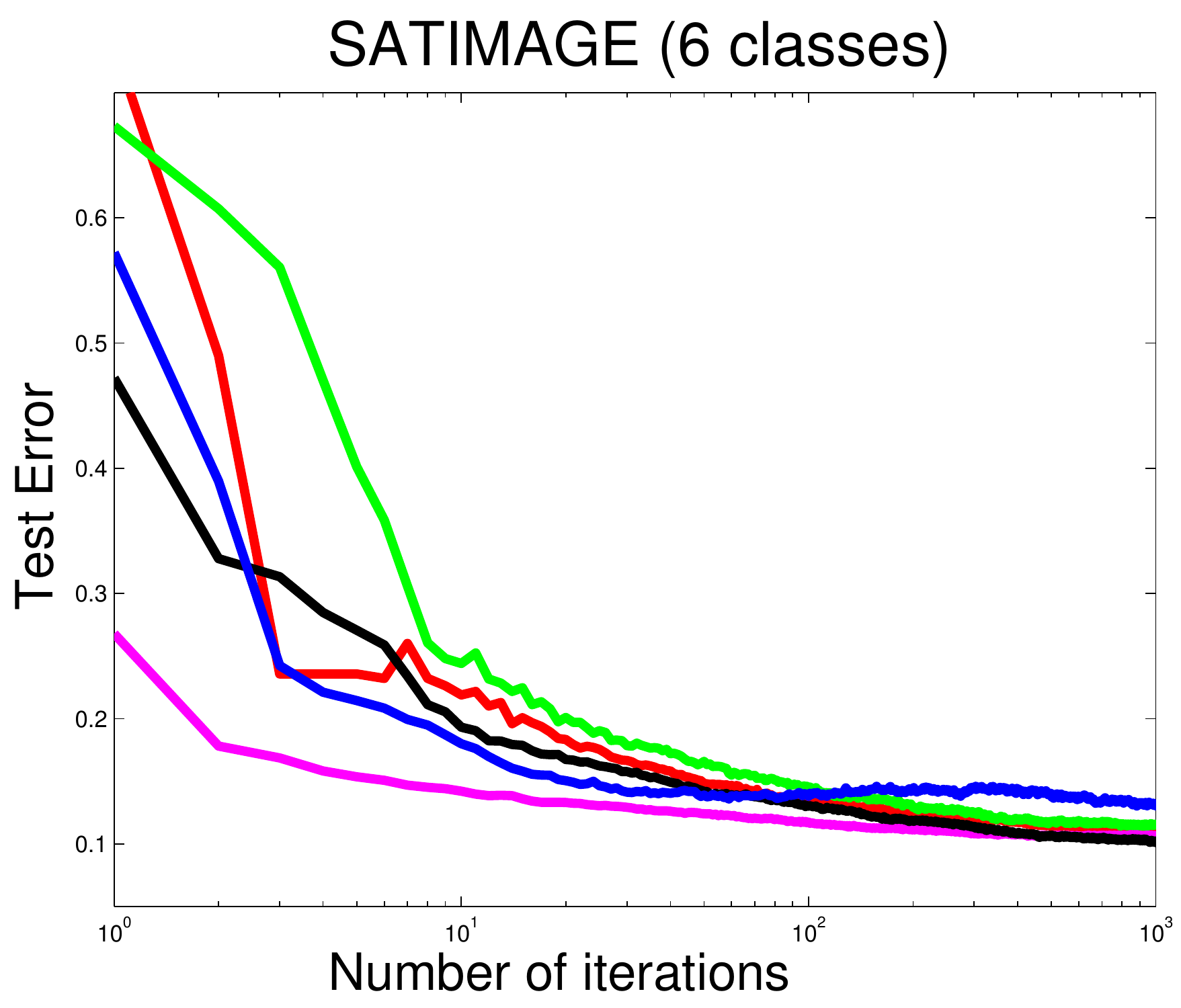}
        \includegraphics[width=0.19\textwidth,clip]{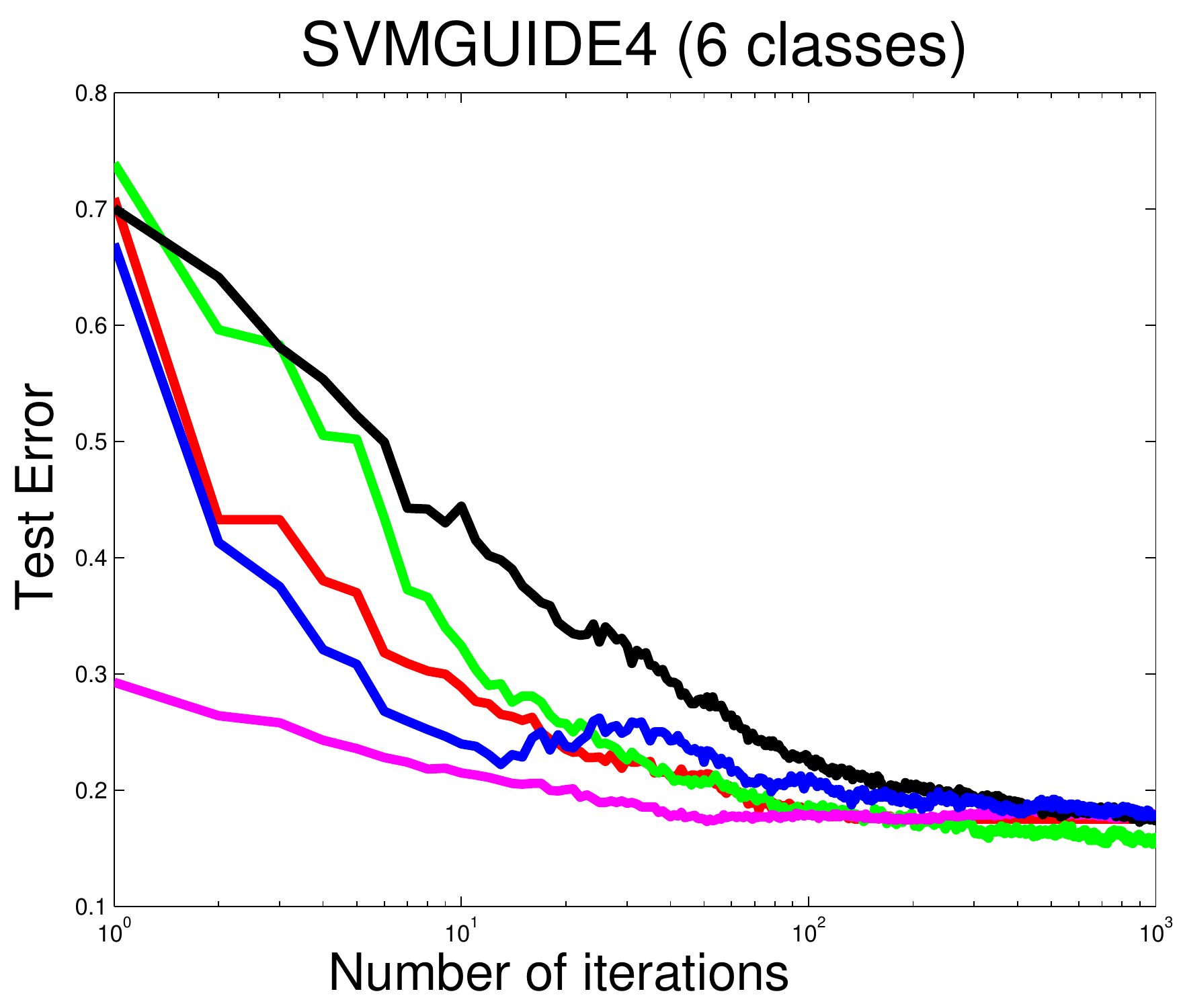}
        \includegraphics[width=0.19\textwidth,clip]{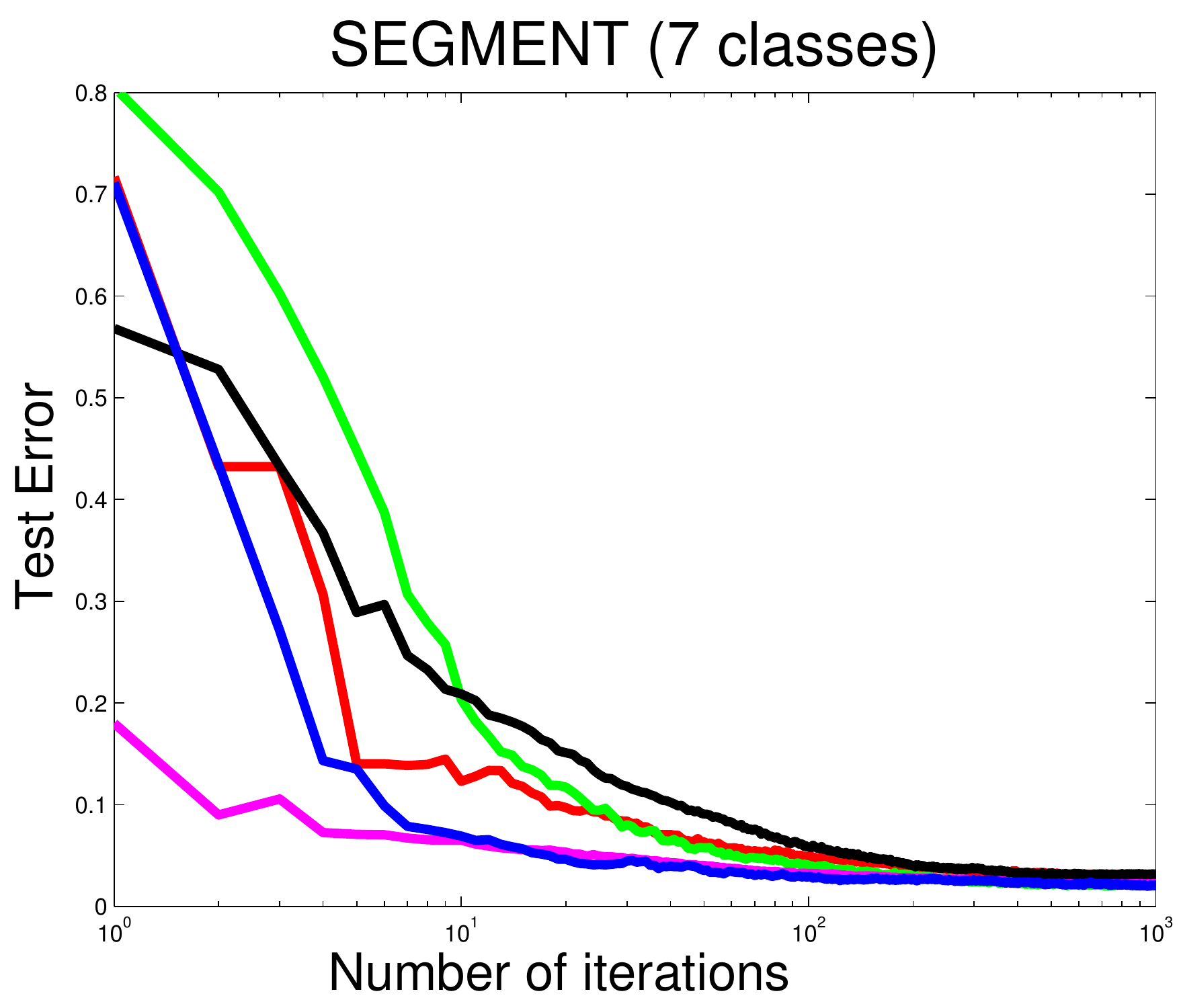}
        \includegraphics[width=0.19\textwidth,clip]{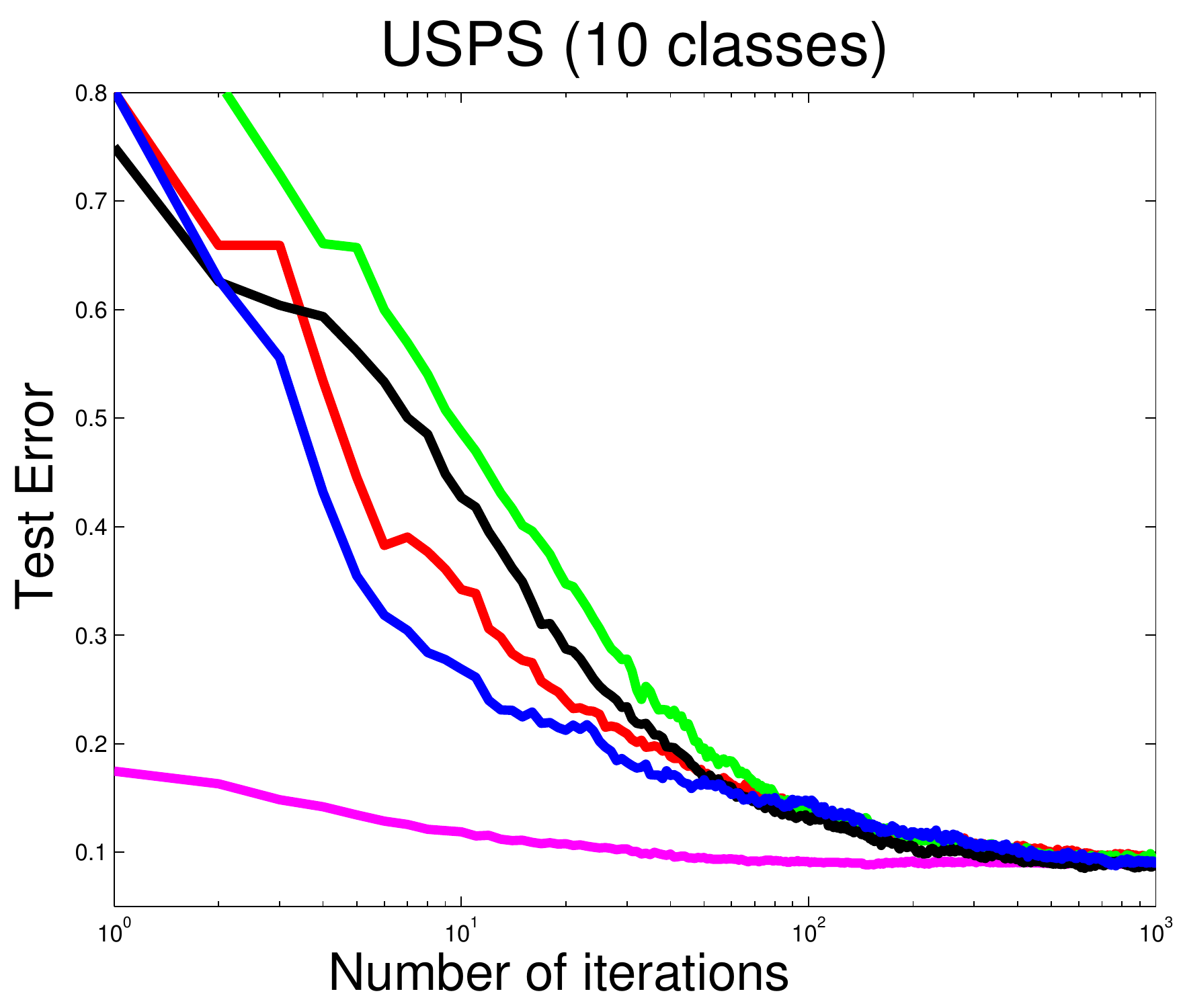}
        \includegraphics[width=0.19\textwidth,clip]{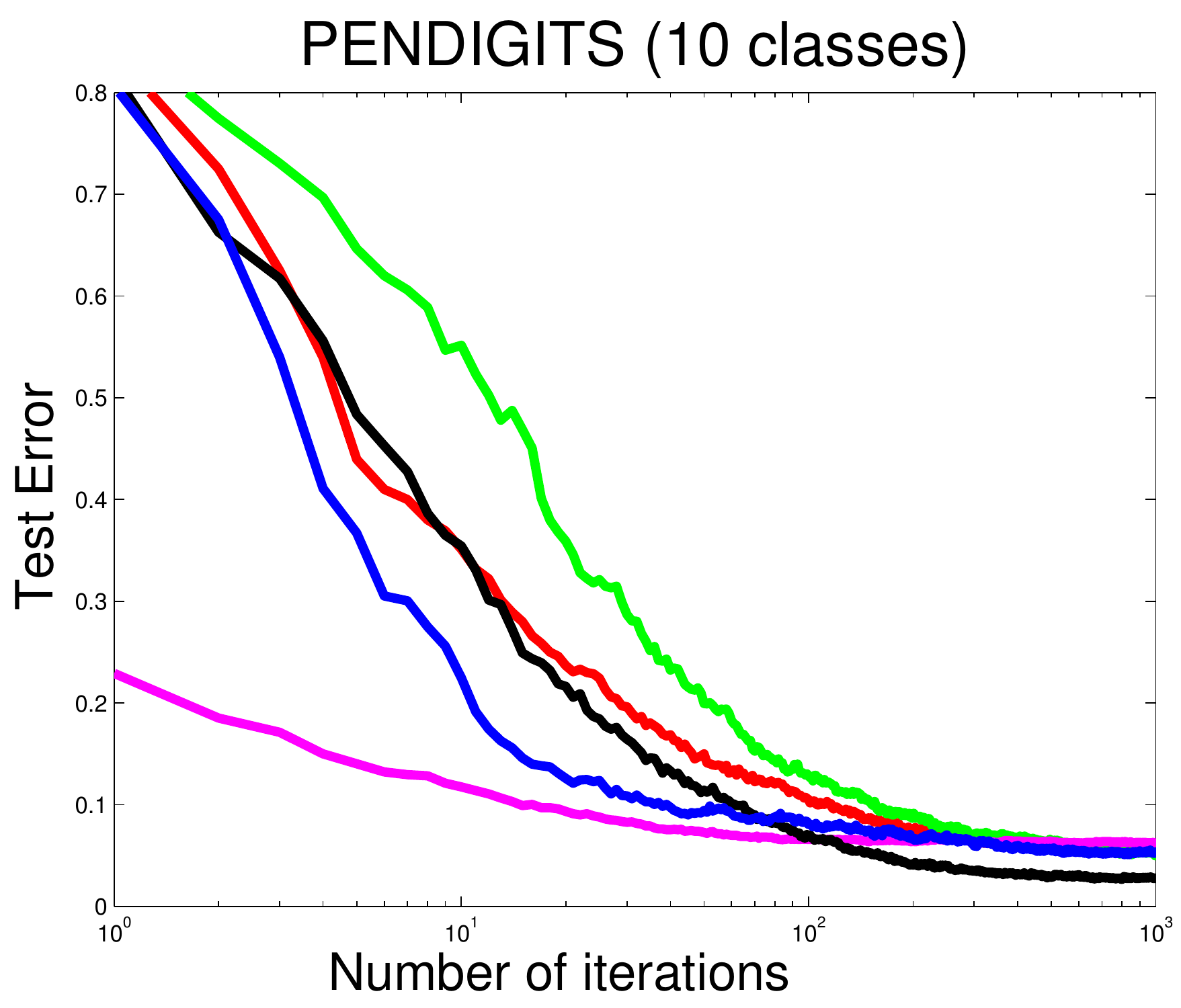}
        \includegraphics[width=0.19\textwidth,clip]{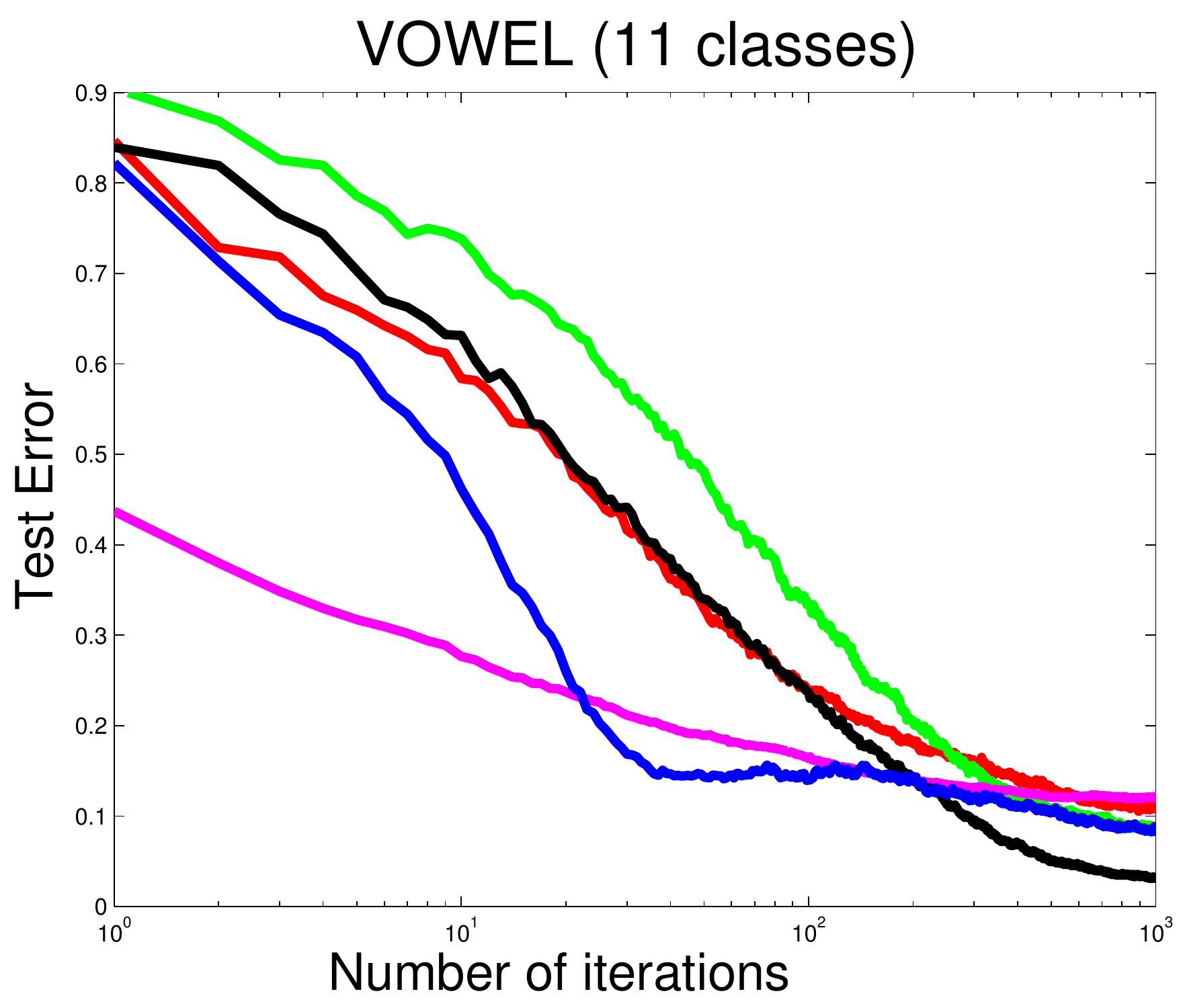}
    \end{center}
    \caption{
    Average test error curves on multi-class UCI data sets.
    The vertical axis denotes the averaged test error rate and
    the horizontal axis denotes the number of boosting iterations.
    Best viewed in color.
    }
    \label{fig:UCImulti}
  \end{figure*}

The next experiment is conducted on both binary and multi-class UCI machine learning repository
and Statlog data sets\footnote{For USPS and pendigits, we use $100$ samples from each class.}.
For binary classification problems, we compare our approaches with AdaBoost \cite{Freund1997Decision} while for multi-class problems, we compare our approaches with AdaBoost.MH \cite{Schapire1999Improved}, AdaBoost.MO \cite{Schapire1999Improved}, AdaBoost.ECC \cite{Guruswami1999Multiclass} and MultiBoost \cite{Shen2011Direct}.
Each data set is then randomly split into two groups: $75\%$ of samples for training and the rest for evaluation.
We set the maximum number of boosting iterations to $1000$.
For AdaBoost.MH, AdaBoost.MO and AdaBoost.ECC, the training stops when the algorithm converges, \eg, when the weighted error of weak classifiers is greater than $0.5$.
For MultiBoost, we use the logistic loss and choose the regularization parameter from $\{ 10^{-8}, 10^{-7}, 10^{-6}, 10^{-5}, 10^{-4} \}$.
For \randomRank, we set $n$ to be $2\cdot 10^4$.
For \randomProj, we set $n$ to be equal to the number of boosting iterations, \ie, $1000$.
Note that we have not carefully tuned $n$ in this experiment.
The regularization parameter, $\nu$, is determined by $5$-fold cross validation.
We choose the best $\nu$ from $\{ 10^{-5}, 10^{-4}, 10^{-3}, 10^{-2} \}$ for binary problems and from 
$\{ 10^{-8},$ $2.5 \times 10^{-8},$ $5 \times 10^{-8},$
$7.5 \times 10^{-8},$ $10^{-7}, \cdots, 10^{-2} \}$
for multi-class problems.
The training stops when adding more weak classifiers does not further decrease the objective function of \eqref{EQ:logPrimal}.
All experiments are repeated $10$ times and the mean and standard
deviation of test errors are reported in Tables~\ref{tab:UCIbinary} and \ref{tab:UCImulti}.
For binary classification problems, we observe that all methods perform similarly.
This indicates that random projection based classifiers work well in practice.
This is not surprising since it can be shown easily that, for two-class problems, \randomRank simply performs AdaBoost on the randomly projected data \cite{Rudin2009Margin}.
By the theory of random projections  one would expect the performance of AdaBoost trained using the data in the original space to be similar to that of AdaBoost trained using the randomly projected data.
For multi-class problems, we observe that most methods perform very similarly.
However, \randomRank has a slightly better generalization performance 
than other multi-class boosting algorithms on $5$ out of $11$ data sets while
\randomProj performs slightly better than other algorithms on 
$3$ out of $11$ data sets.

\revised{
We then statistically compare both proposed approaches using 
the nonparametric Wilcoxon signed-rank test (WSRT) \cite{Demvsar2006Statistical}.
WSRT tests the median performance difference between 
\randomProj and \randomRank.
In this test, we set the significance level to be $5\%$.
The null-hypothesis declares that there is no difference between
the median performance of both algorithms at the $5\%$ significance level,
\ie, both algorithms perform equally well in a statistical sense.
According to the table of exact critical values for the Wilcoxon's test,
for a confidence level of $0.05$ and $11$ data sets,
the difference between the classifiers is significant if the smaller
of the rank sums is equal or less than $10$.
Since the signed rank statistic result ($16$) is not less than 
the critical value ($10$), 
WSRT indicates a failure to reject the null hypothesis 
at the $5\%$ significance level. 
In other words, the test statistics suggest that both 
\randomProj and \randomRank perform equally well.

We further conduct an additional experiment on \randomRank and \randomProj 
using a different weak classifier.
An alternative choice of weak classifiers for training boosting classifiers
is weighted Fisher linear discriminant analysis (WLDA) \cite{Shen08ICIP}.
WLDA learns a linear projection function which 
ensures good class separation between
normally distributed samples of two classes.
The linear projection function is
defined as $ (\Sigma_1 + \Sigma_2)^{-1} (\mu_1 - \mu_2) $
where $\mu_1$ and $\mu_2$ are weighted class mean, and 
$\Sigma_1$ and $\Sigma_2$ are weighted class covariance 
matrices of the first and second class, respectively.
In our experiment, we project the weighted input data 
to a line using WLDA and
train the decision stump on the new $1$D data \cite{Shen08ICIP}.
Although WLDA has a closed-form solution, computing the inverse
of the covariance matrix can be computationally expensive when the size 
of covariance matrices is large,
\ie, the time complexity is cubic in 
the size of covariance matrices which is $O ( [\min (n, (m-1 ) k ) ]^3 )$.
For \randomRank, it is computationally infeasible to
find the inverse of the covariance matrices when 
$n$ ($n = 20,000$) and $  (m-1) k  $ is large.
The is one of the advantages for 
\randomRank, compared with \randomProj. 

So instead we randomly select $1000$ dimensions from $n$
at each boosting iteration and then apply WLDA.
We concatenate the new WLDA feature to $n$ randomly
projected features and train \randomRank.
The average classification error of both approaches
is shown in Table~\ref{tab:exp_proj_wlda}.
1) We observe that the performance of both approaches
often improves when we apply a more discriminative WLDA
as the weak learner, compared with decision stumps.
2) Again, we statistically compare the performance of both 
proposed approaches (with WLDA as the weak learner).
Since the signed rank statistic result ($10.5$) is not less than 
the critical value ($0$), 
WSRT indicates a failure to reject the null hypothesis 
at the $5\%$ significance level. 
In summary, both algorithms also perform equally well when
WLDA is used as the weak learner.
Note that other weak learners, \eg, LIBLINEAR and
radial basis function (RBF),
may also be applied here.

We plot average test error curves of multi-class UCI data sets in
Fig.~\ref{fig:UCImulti}.
Again, we can see that both of the proposed methods perform similarly. 

}

\begin{table}[t]
  \centering
  {
  \begin{tabular}{l|cc}
  \hline
  Data set &  \randomRank (WLDA) & \randomProj (WLDA) \\
  \hline
  \hline
svmguide2 ($3$ classes) & $\mathbf{18.9}$ ($\mathbf{3.1}$) & $19.7$ ($2.6$) \\
wine ($3$ classes) & $3.2$ ($2.9$) & $\mathbf{2.5}$ ($\mathbf{3.1}$) \\
vehicle ($4$ classes) & $19.6$ ($2.3$) & $\mathbf{18.4}$ ($\mathbf{2.6}$) \\
glass ($6$ classes) & $25.9$ ($4.0$) & $\mathbf{22.5}$ ($\mathbf{4.2}$) \\
pendigits ($10$ classes) & $\mathbf{2.8}$ ($\mathbf{0.9}$) & $4.0$ ($1.2$) \\
vowel ($11$ classes) & $\mathbf{2.6}$ ($\mathbf{1.1}$) & $4.5$ ($1.8$) \\
  \hline
  \end{tabular}
  }
  \caption{Average test errors (shown in $\%$) with linear perceptron
  classifiers trained by weighted linear
  discriminant analysis (WLDA) as the weak learner. 
  }
  \label{tab:exp_proj_wlda}
\end{table}

  \begin{figure}[t]
    \begin{center}
        \includegraphics[width=0.24\textwidth,clip]{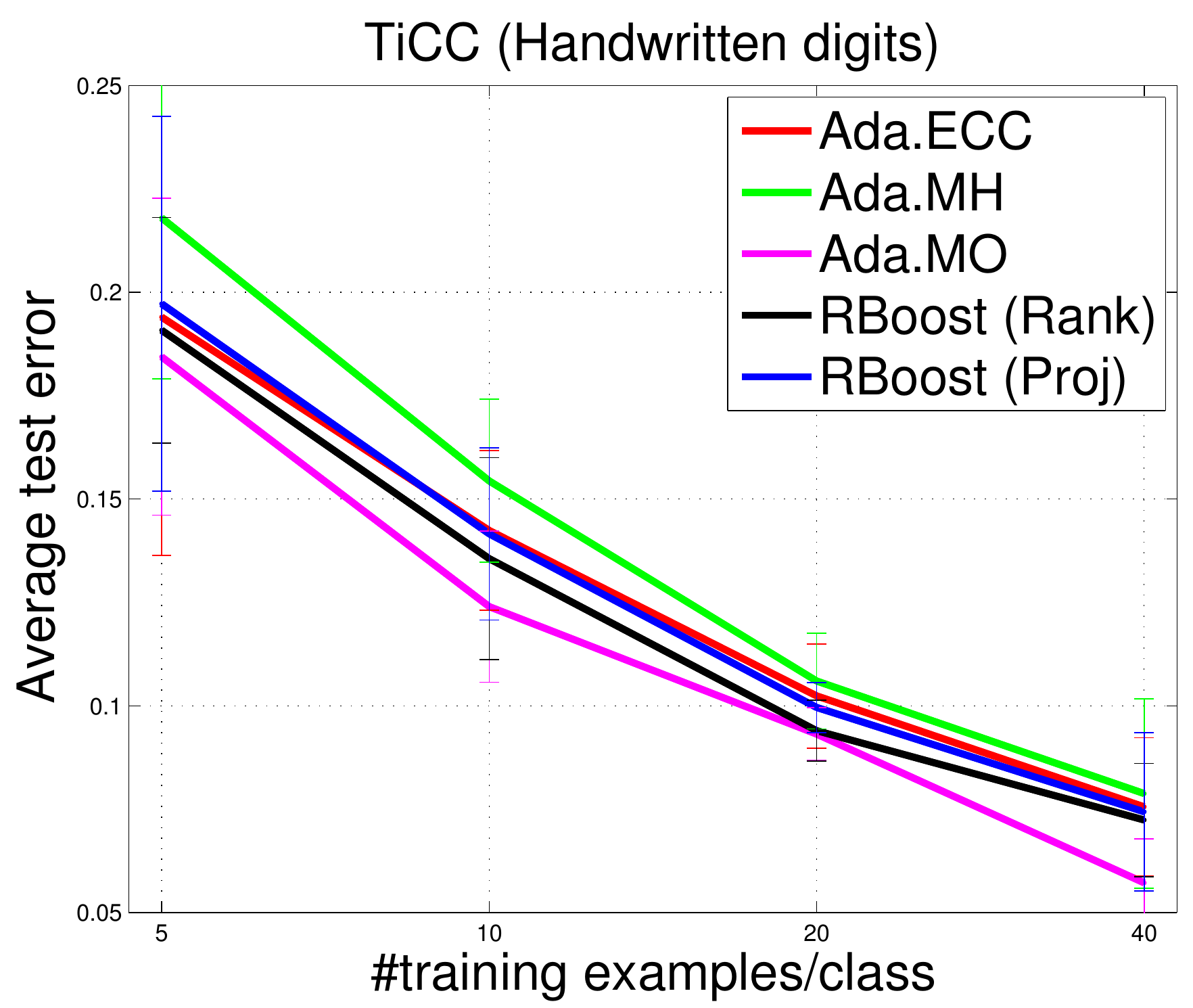}
        \includegraphics[width=0.24\textwidth,clip]{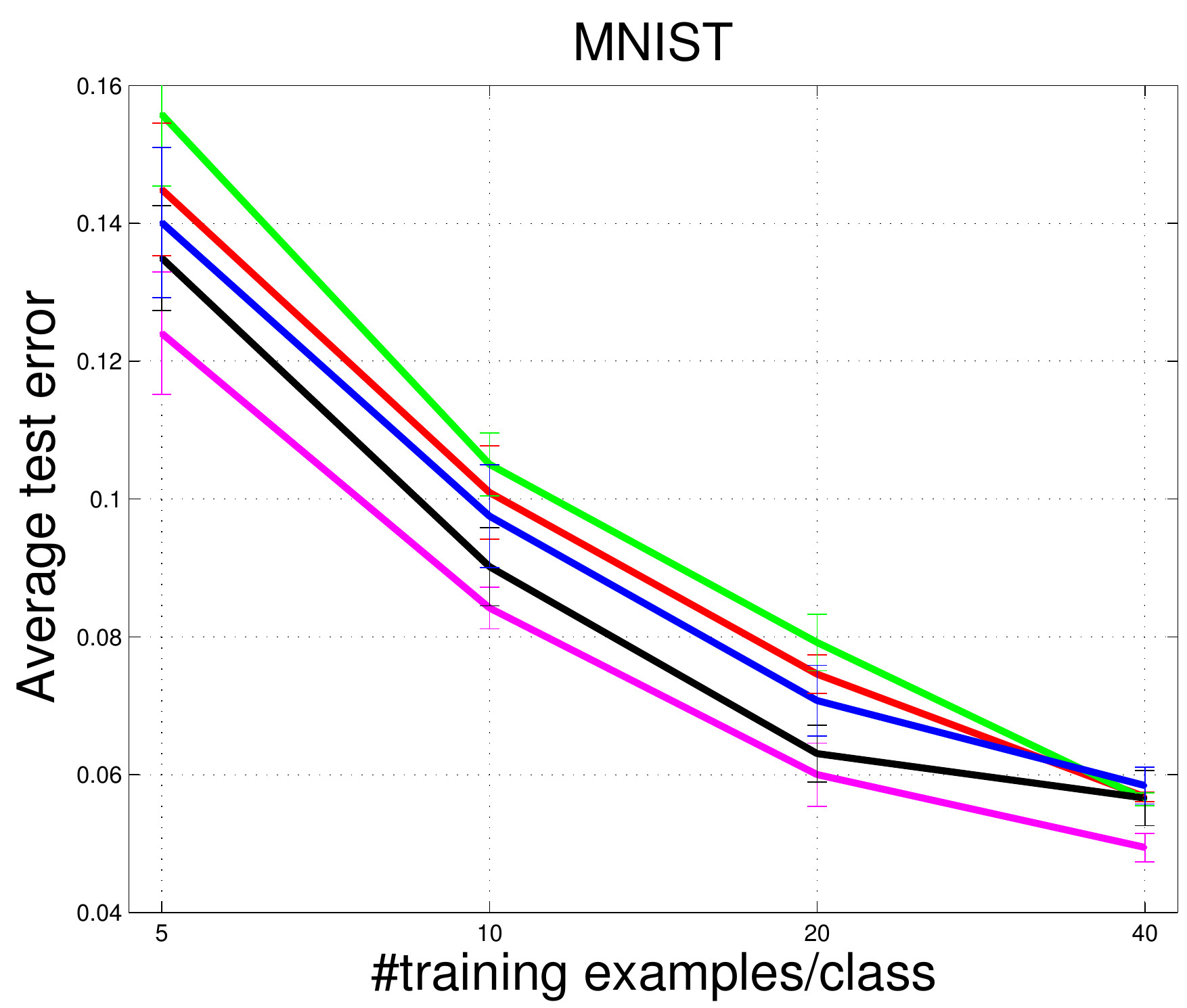}
    \end{center}
    \caption{
    Average test errors on handwritten digits data sets by varying the amount of training samples per class.
    \textbf{left:} TiCC. \textbf{right:} MNIST.
    Best viewed in color.
    }
    \label{fig:HandWritten}
  \end{figure}

\subsection{Handwritten digits data sets}
        In this experiment, we vary the number of training samples and
        compare the performance of different boosting algorithms.  We
        evaluate our algorithms on popular handwritten digits data
        sets (MNIST) and a more difficult handwritten character data
        sets (TiCC) \cite{Maaten2009New}.  We first resize the
        original image to a resolution of $28 \times 28$ pixels and
        apply a deslant technique, similar to the one applied in
        \cite{Meier2011Better}.  We then extract $3$ levels of HOG
        features with $50\%$ block overlapping (spatial pyramid
        scheme) \cite{Lazebnik2006Beyond}.  The block size in each
        level is $4 \times 4$, $7 \times 7$ and $14 \times 14$ pixels,
        respectively.  Extracted HOG features from all levels are
        concatenated.  In total, there are $2,172$ HOG features.  We
        perform dimensionality reduction using Principal Component
        Analysis (PCA) on training samples (similar to PCA-SIFT
        \cite{Ke2004PCASIFT}).  Our PCA projected data captures $90\%$
        of the original data variance.  For \randomRank, we set $n$ to
        be $100K$.  For \randomProj, we choose the best parameter from
        $\{ 5 \times 10^{-8}, 10^{-7}, 5 \times 10^{-7}, 10^{-6}, 5
        \times 10^{-6}, 10^{-5} \}$.  For MNIST, we randomly select
        $5$, $10$, $20$ and $40$ samples as training sets and use the
        original test sets of $10,000$ samples.  For TiCC, we randomly
        select $5$, $10$, $20$ and $40$ samples from each class as
        training sets and use $50$ unseen samples from each class as
        test sets.  All experiments are repeated $5$ times ($1000$
        boosting iterations) and the results are summarized in
        Fig.~\ref{fig:HandWritten}.  For handwritten digits,
        we observe that our algorithms and AdaBoost.MO
        perform slightly better than
        AdaBoost.ECC and AdaBoost.MH.

\revised{\emph{Note that AdaBoost.MO trains $2^{k-1} - 1$ weak classifiers
at each iteration, while both of our algorithms train $1$ weak classifier
at each iteration}.
For example, on MNIST
digit data sets, the AdaBoost.MO model would have a total of $511,000$
weak classifiers ($1000$ boosting iteration)
while our multi-class classifier would only consist of $1000$
weak classifiers.
In other words, AdaBoost.MO is
$511$ times slower during performance evaluation.
We suspect that these additional weak classifiers improve the
generalization performance of AdaBoost.MO for handwritten digits
data sets, where there is a large variation within the same class label.
}

\begin{figure}[t]
    \centering
        \includegraphics[width=0.24\textwidth,clip]{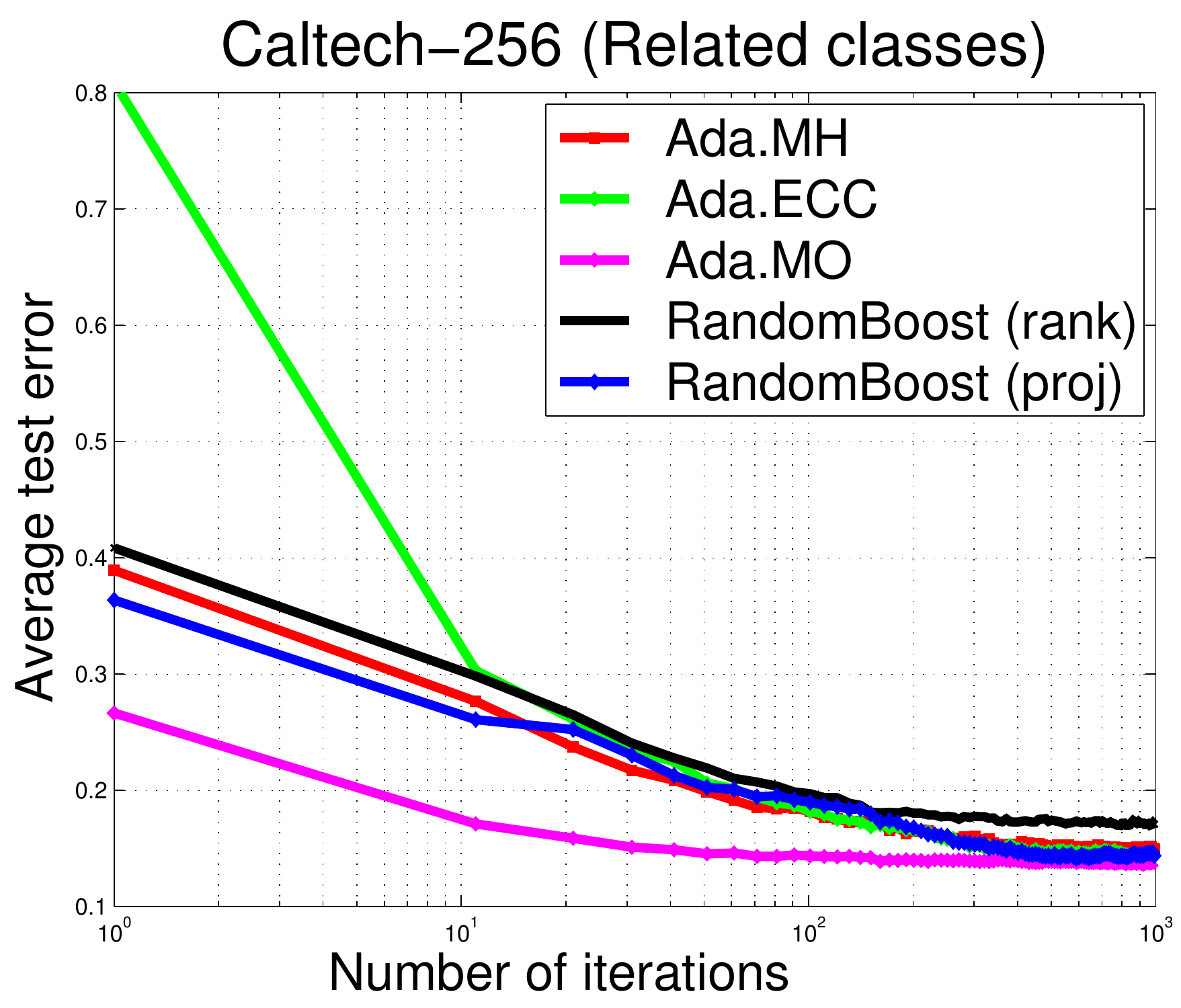}
        \includegraphics[width=0.24\textwidth,clip]{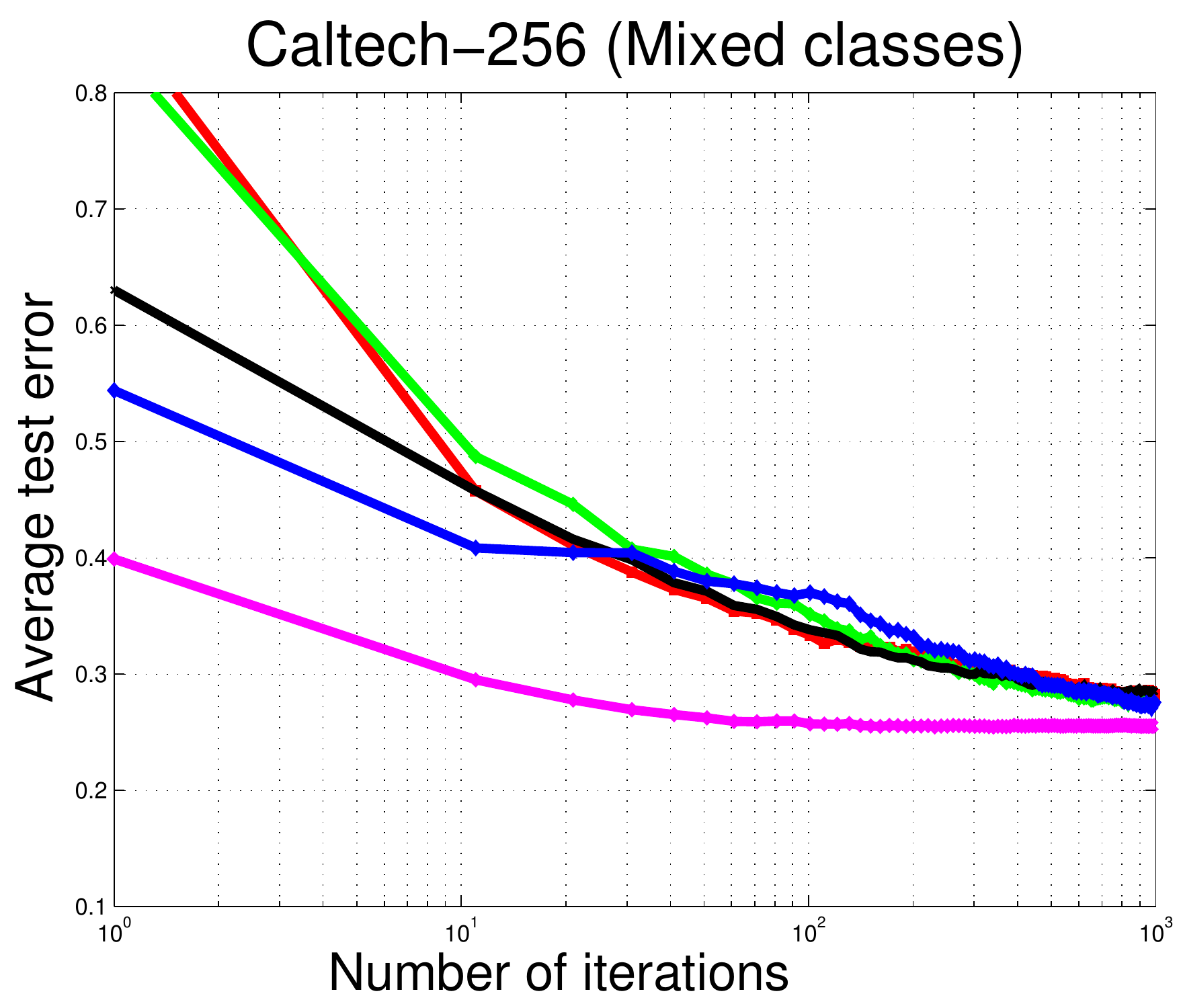}
    \caption{
    Average test error curves on Caltech-$256$ data sets.
    \textbf{left:} Related classes. \textbf{right:} Mixed classes.
    Best viewed in color.
    }
    \label{fig:Caltech}
  \end{figure}

\subsection{Caltech-$256$ data sets}
We also evaluate our algorithms on a subset of Caltech-$256$.
We consider two types of classes as experimented in
\cite{Tommasi2010Safety}: related classes\footnote{bulldozer,
firetruck, motorbikes, schoolbus, snowmobile and car-side.}
and mixed classes\footnote{dog,
horse, zebra, helicopter, fighter-jet, motorbikes, car-side, dolphin,
goose and cactus.}.
We use the same pre-computed features used in
\cite{Gehler2009Feature}, \ie, PHOG, appearance, region covariance and LBP.
The data set is randomly split into two groups: $25\%$ for training and the rest for evaluation.
On average, there are $56$ training samples per class for related classes and $48$ training samples per class for mixed classes.
We use the same setting as in the handwritten digits experiment.
        The average test accuracies of $5$ runs are reported in
        Fig.~\ref{fig:Caltech}.
\revised{
        We see again that AdaBoost.MO converges faster.
        This is not surprising as we previously mentioned that
        AdaBoost.MO trains $2^{k-1} - 1$ weak classifiers
        at each iteration.
        As AdaBoost.MO is not scalable on a large
        number of classes, it is extremely slow during 
        performance evaluation.
        Based on our experiments, 
        AdaBoost.MO requires approximately $2^{k-1}$ times  
        as much execution time as other algorithms
        during test time.
        The second observation is that 
        our proposed methods usually converge slightly faster than
        AdaBoost.MH and AdaBoost.MH. 
     
}

\section{Conclusion}
\label{sec:con}

We have shown that, by exploiting random projections, it is possible
to devise a single-vector parameterized boosting-based classifier,  which is
capable of performing multi-class classification.  This approach
represents a significant divergence from existing multi-class
classification approaches, as neither the number of classifiers, nor
the number of parameters will grow as the number of classes increases.
We have demonstrated two examples of the proposed approach, in the
form of multi-class boosting algorithms,  which solve the pairwise
ranking problem and pairwise loss in the large margin framework.
These algorithms are effective and can cope with both binary and
multi-class classification problems as demonstrated on both synthetic
and real world data sets.

Our goal thus far has been to formulate a single-vector  multi-class
boosting classifier, which demonstrates promising results and
alleviate the proliferation of parameters typically faced in
large-scale problems.
Reducing the training time required by both methods for large-scale
problems is yet another challenging issue.
Techniques, such as
approximating the weak classifiers' threshold \cite{Pham2007Fast},
approximating the weak classifiers using FilterBoost \cite{Bradley2008FilterBoost}
or incremental weak classifier learning \cite{Pang2012Incremental}, offer
an interesting approach towards this goal.
An exploration on the effect of the size of random projection matrices
on convergence and scaling could also be carried out.

\appendices
\label{app:theorem}
\def\lVert{\left\Vert}
\def\rVert{\right\Vert}

\def\xb{ {\boldsymbol x} }

\def\ub{ {\boldsymbol w} }

\def\wb{ {\boldsymbol w} }

\def\eb{ {\bf e} }

\def\zb{ {\boldsymbol z} }

\def\vb{ {\boldsymbol q} }

\def\Hb{ {\bf H} }

\def\Rb{ {\bf P} }

\section{Proof of Theorem \ref{thm:multiclass} }

\paragraph{Margin Preservation} 
If the boosting has margin $\gamma$,
 then for any $\delta,\epsilon \in(0,1) $ and any \[n >
 \frac{12}{3\epsilon^2-2\epsilon^3} \ln\frac{6km}{\delta},\] with
 probability at least $1-\delta$, the boosting associated with
 projected weak learners' coefficients, $\bP\bw_r, \forany r$, and the
 projected weak learners' response, $\bP \bH$, has margin no less than
 \[ -\frac{1+3\epsilon}{1-\epsilon^2}+ \frac{\sqrt{1 - \epsilon^2} }{1 + \epsilon}+\frac{1 + \epsilon }{1 - \epsilon } \gamma.\]

\begin{proof}
By margin definition, for all $(\xb,y)\in S$
\begin{align*}
\frac{\inner{\ub_y}{\Hb(\xb)}}{{\|\ub_y\|}{\|\Hb(\xb)\|} } - \max_{y' \neq y} \frac{\inner{\ub_{y'}}{\Hb(\xb)}}{{\|\ub_{y'}\|}{\|\Hb(\xb)\|}}\ge \gamma.
\end{align*}
Take any single $(\xb,y) \in S$, 
and let $\hat{y} = \argmax_{y'}
\frac{\inner{\ub_{y'}}{\Hb(\xb)}}{{\|\ub_{y'}\|}{\|\Hb(\xb)\|}}$ and $\hat{\hat{y}} =
\argmax_{y''}
\frac{\inner{\Rb\ub_{y''}}{\Rb\Hb(\xb)}}{{\|\Rb\ub_{y''}\|}{\|\Rb\Hb(\xb)\|}}$, we have by
Lemma~\ref{lem:inner} (substituting $\xb$ in Lemma~\ref{lem:inner} by $\Hb(\xb)$ ) and union bound over all $y' \neq y$
 \begin{align*}
\Pr\Big(\frac{ \inner{\Rb\ub_y}{\Rb\Hb(\xb)}}{ \|{\Rb\ub_y}\|\|{\Rb\Hb(\xb)}\|} &
\ge 1-\frac{1 + \epsilon }{ 1 - \epsilon    } \big(1-\frac{ \inner{\ub_y}{\Hb(\xb)}}{
\|{\ub_y}\|\|{\Hb(\xb)}\|}\big)\Big) 
\\ &\ge
1- 6\exp{(-\frac{n}{2} (\frac{\epsilon^2}{2} -\frac{\epsilon^3}{3} ))},
\\
\Pr\Big( \forall y' \neq y, \\
\frac{\inner{\Rb\ub_{y' }}{\Rb\Hb(\xb)}}{\|{\Rb\ub_{y'
}}\|\|{\Rb\Hb(\xb)}\|} & \leq 1-\frac{\sqrt{1 - \epsilon^2} }{1 + \epsilon }
+\frac{\epsilon}{1+\epsilon}\\ &+\frac{1-\epsilon}{1+\epsilon}  \cdot   \frac{
\inner{\ub_{y' }}{\Hb(\xb)}}{ \|{\ub_{y' }}\|\|{\Hb(\xb)}\|}  \Big)
\\&\geq 1- 6(k-1)\exp{ \left(-\frac{n}{2} \cdot (\frac{\epsilon^2}{2} -\frac{\epsilon^3}{3} )
\right)}.
\end{align*} 
By the union bound again, 
with probability at least 
\[
 1 - 6km \cdot \exp{   \left( -\frac{n}{2} (\frac{\epsilon^2}{2} -\frac{\epsilon^3}{3} )
 \right)}
 \] for all $( \xb,y) \in S$, we have
\begin{align*} 
&\frac{ \inner{\Rb\ub_y}{\Rb\Hb(\xb)}}{ \|{\Rb\ub_y}\|\|{\Rb\Hb(\xb)}\|} - \max_{y' \neq y}\frac{ \inner{\Rb\ub_{y'}}{\Rb\Hb(\xb)}}{ \|{\Rb\ub_{y'}}\|\|{\Rb\Hb(\xb)}\|}\\
 \ge& -\frac{ 1+3\epsilon }{ 1-\epsilon^2 }+ \frac{\sqrt{ 1 - \epsilon^2 } }{ 1 + \epsilon }+\frac{ 1 + \epsilon }{ 1 - \epsilon } \gamma
\end{align*} 
Let $\delta =  6km\exp{(-\frac{n}{2} (\frac{\epsilon^2}{2} -\frac{\epsilon^3}{3} ))}$, we
have the desirable lower bound on $n$.
\end{proof}

\def\ub{ {\bf u} }

\section{Proof of Theorem \ref{thm:single-param-trick} }

\paragraph{Single-vector Multi-class Boosting}
Given any random Gaussian matrix $\bR \in \RR^{n \times kT}$, whose
entry $\bR(i,j) = $ $ \frac{1}{\sqrt{n}}{a_{ij}}$ where $a_{ij}$ is i.i.d.
random variables from $\Ncal(0,1)$.  Denote $\bP_y \in \RR^{n,T}$ as
the $y$-th submatrix of $\bR$, that is $\bR =
        [\bP_1,$ $ \cdots,\bP_r,$ $\cdots, $ $ \bP_k]$.
If the boosting has margin
$\gamma$, then for any $\delta,\epsilon \in (0,1] $ and any \[n
>\frac{12}{ 3\epsilon^2-2\epsilon^3  } \ln{\frac{6m(k-1)}{\delta}} ,\]
there exists a single-vector $\bv \in \RR^n $, such that
\begin{align}
\Pr\Big( %
&\frac{\inner {\bv}{\bP_y  \bH(\bx) } -\inner {\bv}{\bP_{y'}\bH(\bx)}}{\|\bv\|\sqrt{\|\bP_y\bH(\bx) \|^2+\|\bP_{y'}\bH(\bx) \|^2}} \ge
\nonumber\\
&
\frac{-2\epsilon}{1-\epsilon}+\frac{1+\epsilon}{{\sqrt{2k}}(1-\epsilon)}\gamma \Big)\ge 1-\delta,
\quad \quad
            \forall y'\neq y.
\end{align}

\begin{proof}
By the margin definition, there exists $\wb$, such that for all $(\Hb(\xb),y)\in S$,
\begin{align*}
\frac{\inner{\wb_y}{\Hb(\xb)}}{\|{\wb_y}\|\|{\Hb(\xb)}\|} - \frac{\inner{\wb_{y'}}{\Hb(\xb)}}{\|{\wb_{y'}}\|\|{\Hb(\xb)}\|}\ge \gamma, \forall y' \neq y.
\end{align*}
Without losing generality, we assume $\wb_y$ has unit length, which can always be achieved
by normalization,  for all $y$. So now 
\begin{align*}
{\inner{\wb_y}{\Hb(\xb)}} - {\inner{\wb_{y'}}{\Hb(\xb)}}\ge \gamma{\|{\Hb(\xb)}\|}, \forall y' \neq y.
\end{align*}

 This can be rewritten as 
\begin{align*}
\inner{\ub}{\Hb(\xb) \otimes \eb_ y } - \inner{\ub}{\Hb(\xb)  \otimes {\eb_{y'}}}
=
\\
\inner{\Hb(\xb) \otimes \eb_ y -\Hb(\xb) \otimes \eb_{y'} }{\ub} \ge \gamma{\|{\Hb(\xb)}\|},
\end{align*} 
where $\ub $ is the concatenation of all $\wb_y$, \ie $\ub = [\wb_1^\T, \cdots,
\wb_y^\T, \cdots \wb_k ^\T]^\T$; the vector $\eb_y \in \RR^k $
with 1 at the $y$-th dimension and zeros in others, and $\otimes $ is the tensor product.
Define $\zb_{\xb,y'} = \Hb(\xb) \otimes \eb_ y -\Hb(\xb) \otimes \eb_{y'} $.

Applying  Lemma~\ref{lem:inner} to $\ub$ and $\zb_{\xb,y'} $, 
we have for a given $(\xb,y)$ and a fixed $y' \neq y$, 
with probability at least $1-6\exp{(-\frac{n}{2} (\frac{\epsilon^2}{2}
-\frac{\epsilon^3}{3} ))}$,
the following holds,
\begin{align*}
 &\frac{\inner{\Rb\ub}{\Rb\zb_{\xb,y'}}}{\|\Rb\ub\|\|\Rb\zb_{\xb,y'}\|} \\
 &\ge 1-\frac{1+\epsilon}{1-\epsilon} \left(
 1-\frac{\inner{\Hb(\xb) \otimes \eb_ y -\Hb(\xb) \otimes \eb_{y'}
 }{\ub}}{\sqrt{2}\|\ub\|\|\Hb(\xb)\|}\right)
 \\
 &=1-\frac{1+\epsilon}{1-\epsilon}+\\&\frac{1+\epsilon}{{\sqrt{2}}
 ( 1 - \epsilon)}
 \left(\frac{\inner{\wb_y}{\Hb(\xb)}}{\|\ub\|\|\Hb(\xb)\|}-\frac{\inner{\wb_{y'}}{\Hb(\xb)}}{\|\ub\|\|\Hb(\xb)\|}
 \right)\\
 & \ge 1-\frac{1+\epsilon}{1-\epsilon}+\frac{1+\epsilon}{{\sqrt{2 k }}(1-\epsilon)}\gamma\\
 & = \frac{-2\epsilon}{1-\epsilon}+\frac{1+\epsilon}{{\sqrt{2 k}}(1-\epsilon)}\gamma.
  \end{align*}
By the union bound over $m$ samples and $k-1$ many $y'$,
\begin{align*}
 \Pr \Big( &\exists (\xb,y) \in S,\exists y' \neq y,\\& \frac{\inner{\Rb\ub}{\Rb\zb_{\xb,y'}}}{\|\Rb\ub\|\|\Rb\zb_{\xb,y'}\|} <  \frac{-2\epsilon}{1-\epsilon}+\frac{1+\epsilon}{{\sqrt{2L}}(1-\epsilon)}\gamma\Big) \\
& \leq 6m (  k-1  )\exp{  \left( -\frac{n}{2}  \cdot (\frac{\epsilon^2}{2} -\frac{\epsilon^3}{3}
)\right)}. 
 \end{align*}
Let $\vb = \Rb\ub$, we have
 \begin{align*} 
     \inner{\Rb\ub}{\Rb\zb_{\xb,y'}} = \inner{\vb}{\Rb_y\xb - \Rb_{y'}\xb}.
 \end{align*}
 Setting $\delta = 6m (k-1) \exp{(-\frac{n}{2} (\frac{\epsilon^2}{2} -\frac{\epsilon^3}{3} ))}$ gives the bound on $n$. Thus 
 \begin{align*}
\Pr\Big( &\forall (\xb,y) \in S, \forall y'\neq y, 
 \\ &
\frac{\inner {\vb}{\Rb_y 
\Hb(\xb)} -\inner
{\vb}{\Rb_{y'}\Hb(\xb)}}{\|\vb\|\sqrt{\|\Rb_y\Hb(\xb)\|^2+\|\Rb_{y'}\Hb(\xb)\|^2}} \ge\\
&\frac{-2\epsilon}{1-\epsilon}+\frac{1+\epsilon}{{\sqrt{2k}}(1-\epsilon)}\gamma \Big)\ge 1-\delta,
\end{align*} which concludes the proof.
\end{proof}

We have used the following two lemmas for proving the above two theorems. 

\begin{lem}
\label{lem:inner}
For any $\wb,\xb \in \RR^d$, any random Gaussian matrix $\Rb \in \RR^{n,d}$ whose entry $\Rb(i,j) = \frac{1}{\sqrt{n}}{a_{ij}}$ where $a_{ij}$s are i.i.d. random variables from $\Ncal(0,1)$, \[ \gamma = \frac{\inner{\wb}{\xb}}{\|\wb\|\|\xb\|}, \] for any $\epsilon \in (0,1) $, if  $\gamma \in (0,1]$,
 then with probability at least \[1-6\exp{(-\frac{n}{2} (\frac{\epsilon^2}{2} -\frac{\epsilon^3}{3} ))},\] the following holds
 \begin{align}
  1-\frac{(1 + \epsilon) }{(1 - \epsilon )}{(1-\gamma)} \leq \frac{\inner{\Rb\wb}{\Rb\xb}}{\|{\Rb\wb}\|\|{\Rb\xb}\|}  \nonumber\\ \leq 1-\frac{\sqrt{(1 - \epsilon^2)} }{(1 + \epsilon )} +\frac{\epsilon}{(1+\epsilon)} +\frac{(1-\epsilon)}{(1+\epsilon)} {\gamma}.
  \label{eq:inner}
  \end{align}
\end{lem}
\begin{proof}
From Lemma \ref{lem:JL} and union bound, we know 
  \begin{align}
 (1 - \epsilon) \leq \frac{\|{\Rb\xb}\|^2}{\|\xb\|^2} \leq (1 + \epsilon), (1 - \epsilon) \leq \frac{\|{\Rb\wb}\|^2}{\|\wb\|^2} \leq (1 + \epsilon)
 \label{eq:JL1}
  \end{align} holds with probability at least $1-4\exp{(-\frac{n}{2} (\frac{\epsilon^2}{2} -\frac{\epsilon^3}{3} ))}$.
  When \eqref{eq:JL1} holds, due to the fact that increasing the length of two unit length
  vectors (\ie from $\frac{\Rb\xb}{\|\Rb\xb\|}$ and $\frac{\Rb\wb}{\|\Rb\wb\|}$ to
  $\frac{\Rb\xb}{\sqrt{(1-\epsilon)}\|\xb\|}$ and
  $\frac{\Rb\wb}{\sqrt{(1-\epsilon)}\|\wb\|}$) increases the norm of their
  difference\footnote{Note that the opposite does not hold in general.}, 
  we have
 \begin{align}
 \lVert \frac{\Rb\xb}{\|\Rb\xb\|} - \frac{\Rb\wb}{\|\Rb\wb\|} \rVert^2 \leq  \lVert \frac{\Rb\xb}{\sqrt{(1-\epsilon)}\|\xb\|} - \frac{\Rb\wb}{\sqrt{(1-\epsilon)}\|\wb\|} \rVert^2.
   \label{eq:r1}
 \end{align}
 It is easy to prove that
  \begin{align}
  \lVert \frac{\Rb\xb}{\|\xb\|} -\frac{\Rb\wb}{\|\wb\|} \rVert^2  \leq& \lVert \sqrt{(1-\epsilon)}\frac{\Rb\xb}{\|\Rb\xb\|} - \sqrt{(1+\epsilon)}\frac{\Rb\wb}{\|\Rb\wb\|} \rVert ^2 \nonumber\\ 
  \leq& \lVert \sqrt{(1+\epsilon)}(\frac{\Rb\xb}{\|\Rb\xb\|} -\frac{\Rb\wb}{\|\Rb\wb\|}) \rVert^2 \nonumber\\ &+ (\sqrt{(1+\epsilon)}-\sqrt{(1-\epsilon)})^2.
     \label{eq:r2}
  \end{align}
  The first inequality is due to \eqref{eq:JL1}, the second inequality is due to the property of an acute angle.

 Applying Lemma \ref{lem:JL} to the vector $\left( \frac{\xb}{\|\xb\|} -\frac{\wb}{\|\wb\|} \right)$, we have
 \begin{align}
  & {(1 - \epsilon) } \lVert \frac{\xb}{\|\xb\|} -\frac{\wb}{\|\wb\|} \rVert^2 \leq \lVert \frac{\Rb\xb}{\|\xb\|} - \frac{\Rb\wb}{\|\wb\|} \rVert^2 \nonumber \\
   &\leq {(1 + \epsilon) } \lVert \frac{\xb}{\|\xb\|} -\frac{\wb}{\|\wb\|} \rVert^2 %
   \label{eq:xw}
    \end{align}
    holds for certain probability.
    
    Let $\beta$ be the angle of $\wb$ and $ \xb$, and $\alpha$ be the angle of $\Rb\wb$ and $ \Rb\xb$, we have
    \begin{align}
     \gamma = \frac{\inner{\wb}{\xb}}{\|{\wb}\|\|{\xb}\|} = \cos(\beta) = 1-2\sin^2(\frac{\beta}{2}) \nonumber \\ = 1-\frac{1}{2} \lVert \frac{\xb}{\|\xb\|} - \frac{\wb}{\|\wb\|} \rVert^2.
     \label{eq:cos1}
     \end{align}
     Similarly
     \begin{align}
     \frac{\inner{\Rb\wb}{\Rb\xb}}{\|{\Rb\wb}\|\|{\Rb\xb}\|} = 1-\frac{1}{2} \lVert \frac{\Rb\xb}{\|\Rb\xb\|} - \frac{\Rb\wb}{\|\Rb\wb\|} \rVert^2.
     \label{eq:cos2}
     \end{align}
     Using \eqref{eq:xw}, \eqref{eq:r1} and \eqref{eq:r2} we get $\lVert
     \frac{\Rb\xb}{\|\Rb\xb\|} - \frac{\Rb\wb}{\|\Rb\wb\|} \rVert^2$ is bounded below and
     above by two terms involving $\lVert \frac{\xb}{\|\xb\|} - \frac{\wb}{\|\wb\|} \rVert^2$.
     Plugging \eqref{eq:cos1} and \eqref{eq:cos2} into the two side bounds, we get
     \eqref{eq:inner}.
  Here we applied Lemma \ref{lem:JL} to 3 vectors, namely $\xb$, $\wb$, and  $(\frac{\xb}{\|\xb\|} -\frac{\wb}{\|\wb\|})$, thus by union bound, the probability of the above holds is at least $1-6\exp{(-\frac{n}{2} (\frac{\epsilon^2}{2} -\frac{\epsilon^3}{3} ))}$.
 \end{proof}
 
\begin{lem} 
\label{lem:JL}
For any $\xb \in \RR^T$, any random Gaussian matrix $\Rb \in \RR^{n \times T}$ whose entry $\Rb(i,j) = \frac{1}{\sqrt{n}}{a_{ij}}$ where $a_{ij}$s are i.i.d. random variables from $\Ncal(0,1)$, for any $\epsilon \in (0,1) $,
   \begin{align*}
 &\Pr\Big((1 - \epsilon) \leq \frac{\|{\Rb\xb}\|^2}{\|\xb\|^2} \leq (1 + \epsilon)\Big)\\
 & \ge 1-2\exp{(-\frac{n}{2} (\frac{\epsilon^2}{2} -\frac{\epsilon^3}{3} ))}.
 \end{align*}
\end{lem}
\begin{proof}
Obviously, for any $\wb, \xb \in \RR^T$, the following holds:
\begin{align*}
    &{\mathbb E}( \inner{\Rb\wb}{\Rb\xb} ) \nonumber\\
    =&\frac{1}{n}  {\mathbb E}\Big[\sum_{\ell=1}^n\Big( \sum_{j=1}^d a_{\ell j}w_j \sum_{i=1}^d a_{\ell i}x_i \Big)\Big] \nonumber\\
=&\frac{1}{n} \sum_{\ell=1}^n\Big( \sum_{j=1}^d {\mathbb E} (a_{\ell j}^2)w_jx_j \nonumber\\&+ \sum_{j=1}^d 
{\mathbb E} (a_{\ell j})w_j \sum_{i\neq j: i=1}^d
{\mathbb E} (a_{\ell i})x_i \Big). \nonumber\\
= & \inner{\wb}{\xb}. 
\end{align*}
To obtain above, we only used the fact that $\{a_{ij}\}$ are independent with zero mean and unit variance. 

Due to 2-stability of Gaussian distribution, we know $\sum_{j=1}^d a_{\ell j}w_j  =  \|\wb\| z_\ell$ and $ \sum_{j=1}^d a_{\ell j}x_j =  \|\xb\| z'_\ell$, where ${z_\ell}$ and ${z'_\ell} \sim \Ncal(0,1)$.  we have $\inner{\Rb\wb}{\Rb\xb} = \frac{1}{n}\|\wb\| \|\xb\|\sum_{\ell=1}^n z_{\ell}z'_{\ell} $. If $\wb = \xb$, $\sum_{\ell=1}^n z^2_{\ell} $ is chi-square distributed with $n$-degree freedom. Applying the standard tail bound of  chi-square distribution, we have
\begin{align*}
 &\Pr\Big( \inner{\Rb\wb}{\Rb\xb} \leq (1 - \epsilon) \inner{\wb}{\xb} \Big)\\
 & \leq \exp{\Big(\frac{n}{2}(1-(1-\epsilon) + \ln(1-\epsilon) ) \Big)} \leq  \exp{(-\frac{n}{4} \epsilon^2)}.
 \end{align*}
 Here we used the inequality $\ln(1-\epsilon) \leq -\epsilon -\epsilon^2/2$. Similarly, we have
 \begin{align*}
 &\Pr\Big( \inner{\Rb\wb}{\Rb\xb} \leq (1 + \epsilon)\inner{\wb}{\xb} \Big)\\
 & \leq \exp{\Big(\frac{n}{2}(1-(1+\epsilon) + \ln(1+\epsilon) ) \Big)} \leq \exp{(-\frac{n}{2} (\frac{\epsilon^2}{2} -\frac{\epsilon^3}{3} ))}.
 \end{align*}
Here we used the inequality $\ln(1+\epsilon) \leq \epsilon -\epsilon^2/2+\epsilon^3/3$. 
\end{proof}

\bibliographystyle{ieee}

\end{document}